%% file: Kamalaruban15.tex
\documentclass[12pt]{colt2015} 

\title[Exp-Concavity of Proper Composite Losses]{Exp-Concavity of Proper Composite Losses}
\usepackage{times}
\usepackage{stmaryrd}
\usepackage{url}
\usepackage{algorithmic}
\usepackage{pdfsync}
\usepackage{enumerate}
\input{./Definitions}

\definecolor{islamicgreen}{rgb}{0.0, 0.56, 0.0}

\DeclareMathOperator*{\argmin}{argmin}

\coltauthor{\Name{Parameswaran Kamalaruban} \Email{Kamalaruban.Parameswaran@nicta.com.au}\\
\Name{Robert C. Williamson} \Email{Bob.Williamson@nicta.com.au}\\
\Name{Xinhua Zhang} \Email{Xinhua.Zhang@nicta.com.au}\\
\addr Australian National University and NICTA, Canberra ACT 0200, Australia}

\begin{document}

\maketitle

\begin{abstract}
The goal of online prediction with expert advice is to find a decision strategy which will perform almost as well as the best expert in a given pool of experts, on any sequence of outcomes. This problem has been widely studied and $O(\sqrt{T})$ and $O(\log{T})$ regret bounds can be achieved for convex losses (\cite{zinkevich2003online}) and strictly convex losses with bounded first and second derivatives (\cite{hazan2007logarithmic}) respectively. In special cases like the Aggregating Algorithm (\cite{vovk1995game}) with mixable losses and the Weighted Average Algorithm (\cite{kivinen1999averaging}) with exp-concave losses, it is possible to achieve $O(1)$ regret bounds. \cite{van2012exp} has argued that mixability and exp-concavity are roughly equivalent under certain conditions. Thus by understanding the underlying relationship between these two notions we can gain the best of both algorithms (strong theoretical performance guarantees of the Aggregating Algorithm and the computational efficiency of the Weighted Average Algorithm). In this paper we provide a complete characterization of the exp-concavity of any proper composite loss. Using this characterization and the mixability condition of proper losses (\cite{van2012mixability}), we show that it is possible to transform (re-parameterize) any $\beta$-mixable binary proper loss into a $\beta$-exp-concave composite loss with the same $\beta$. In the multi-class case, we propose an approximation approach for this transformation.
\end{abstract}


\section{Introduction}
\label{sec:intro}
Loss functions are the means by which the quality of a prediction in learning problem is evaluated. A composite loss (the composition of a class probability estimation (CPE) loss with an invertible link function which is essentially just a re-parameterization) is proper if its risk is minimized when predicting the true underlying class probability (a formal definition is given later). In \cite{vernet2011composite}, there is an argument that shows that there is no point in using losses that are neither proper nor proper composite as they are inadmissible. Flexibility in the choice of loss function is important to tailor the solution to a learning problem (\cite{buja2005loss}, \cite{hand1994deconstructing}, \cite{hand2003local}), and it could be attained by characterizing the set of loss functions using natural parameterizations.

The goal of the learner in a \textit{game of prediction with expert advice} (which is formally described in section \ref{sec:games}) is to predict as well as the best expert in the given pool of experts. The regret bound of the learner depends on the merging scheme used to merge the experts' predictions and the type of loss function used to measure the performance. It has already been shown that constant regret bounds are achievable for mixable losses when the Aggregating Algorithm is the merging scheme (\cite{vovk1995game}), and for exp-concave losses when the Weighted Average Algorithm is the merging scheme (\cite{kivinen1999averaging}). We can see that the exp-concavity trivially implies mixability. Even though the converse implication is not true in general, under some re-parameterization we can make it possible. This paper discusses general conditions on proper losses under which they can be transformed to an exp-concave loss through a suitable link function. In the binary case, these conditions give two concrete formulas (Proposition \ref{geoprop} and Corollary \ref{specialcoro}) for link functions that can transform $\beta$-mixable proper losses into $\beta$-exp-concave, proper, composite losses. The explicit form of the link function given in Proposition \ref{geoprop} is derived using the same geometric construction used in \cite{van2012exp}. 

Further we extend the work by \cite{vernet2011composite}, to provide a complete characterization of the exp-concavity of the proper composite multi-class losses in terms of the Bayes risk associated with the underlying proper loss, and the link function. The mixability of proper losses (mixability of a proper composite loss is equivalent to the mixability of its generating proper loss) is studied in \cite{van2012mixability}. Using these characterizations (for the binary case), in Corollary \ref{specialcoro} we derive an \textit{exp-concavifying link} function that can also transform any $\beta$-mixable proper loss into a $\beta$-exp-concave composite loss. Since for the multi-class losses these conditions do not hold in general, we propose a geometric approximation approach (Proposition \ref{exp_concave_approx}) which takes a parameter $\epsilon$ and transforms the mixable loss function appropriately on a subset $S_\epsilon$ of the prediction space. When the prediction space is $\Delta^n$, any prediction belongs to the subset $S_\epsilon$ for sufficiently small $\epsilon$. In the conclusion we provide a way to use the Weighted Average Algorithm with learning rate $\beta$ for proper $\beta$-mixable but non-exp-concave loss functions to achieve $O(1)$ regret bound.

The exp-concave losses achieve $O(\log{T})$ regret bound in online convex optimization algorithms, which is a more general setting of online learning problems. Thus the exp-concavity characterization of composite losses could be helpful in constructing exp-concave losses for online learning problems.

The paper is organized as follows. In Section \ref{sec:preli} we formally introduce the loss function, several loss types, conditional risk, proper composite losses and a game of prediction with expert advice. In Section \ref{sec:expsec} we consider our main problem --- whether one can always find a link function to transform $\beta$-mixable losses into $\beta$-exp-concave losses. Section \ref{sec:conc} concludes with a brief discussion. The impact of the choice of substitution function on the regret of the learner is explored via experiments in Appendix \ref{sec:subfunc}. In Appendix \ref{sec:probability}, we discuss the mixability conditions of \textit{probability games} with continuous outcome space. Detailed proofs are in Appendix \ref{sec:proof}.

\section{Preliminaries and Background}
\label{sec:preli}

This section provides the necessary background on loss functions, conditional risks, and the sequential prediction problem.

\subsection{Notation}

We use the following notation throughout. A superscript prime, $A'$ denotes transpose of the matrix or vector $A$, except when applied to a real-valued function where it denotes derivative ($f'$). We denote the matrix multiplication of compatible matrices $A$ and $B$ by $A \cdot B$, so the inner product of two vectors $x,y \in \mathbb{R}^n$ is $x' \cdot y$. Let $[n]:=\{1,...,n\}$, $\mathbb{R}_{+} := [0,\infty)$ and the $n$-simplex $\Delta^n:=\{(p_1,...,p_n)': 0 \leq p_i \leq 1, \forall{i \in [n]}, \mathrm{and} \sum_{i \in [n]} p_i = 1\}$. If $x$ is a $n$-vector, $A=\mathrm{diag}(x)$ is the $n \times n$ matrix with entries $A_{i,i}=x_i$ , $i \in [n]$ and $A_{i,j}=0$ for $i \neq j$. If $A-B$ is positive definite (resp. semi-definite), then we write $A \succ B$ (resp. $A \succcurlyeq B$). We use $e_i^n$ to denote the $i$th $n$-dimensional unit vector, $e_i^n=(0,...,0,1,0,...0)'$ when $i \in [n]$, and define $e_i^n=0_n$ when $i > n$. The $n$-vector $\vone_n:=(1,...,1)'$. We write $\llbracket P \rrbracket=1$ if $P$ is true and $\llbracket P \rrbracket=0$ otherwise. Given a set $S$ and a weight vector $w$, the \textit{convex combination} of the elements of the set w.r.t the weight vector is denoted by $\mathrm{co}_w S$, and the \textit{convex hull} of the set which is the set of all possible convex combinations of the elements of the set is denoted by $\mathrm{co} S$ (\cite{rockafellar1970convex}). If $S,T \subset \RR^n$, then the \textit{Minkowski sum} $S \varoplus T := \{ s+t: s  \in S, t \in T \}$. $\mathcal{Y}^{\mathcal{X}}$ represents the set of all functions $f:\mathcal{X}\rightarrow\mathcal{Y}$. Other notation (the Kronecker product $\otimes$, the Jacobian $\textsf{D}$, and the Hessian $\textsf{H}$) is defined in Appendix A of \cite{van2012mixability}. $\textsf{D}f(v)$ and $\textsf{H}f(v)$ denote the Jacobian and Hessian of $f(v)$ w.r.t. $v$ respectively. When it is not clear from the context, we will explicitly mention the variable; for example $\textsf{D}_{\tilde{v}}f(v)$ where $v=h(\tilde{v})$.

\subsection{Loss Functions}

For a prediction problem with an instance space $\mathcal{X}$, outcome space $\mathcal{Y}$ and prediction space $\mathcal{V}$, a loss function $\ell:\mathcal{Y} \times \mathcal{V} \rightarrow \mathbb{R}_{+}$ (bivariate function representation) can be defined to assign a penalty $\ell(y,v)$ for predicting $v \in \mathcal{V}$ when the actual outcome is $y \in \mathcal{Y}$. When the outcome space $\mathcal{Y}=[n], \, n \geq 2$, the loss function $\ell$ is called a \textit{multi-class loss} and it can be expressed in terms of its partial losses $\ell_{i}:=\ell(i,\cdot)$ for any outcome $i \in [n]$, as $\ell(y,v) = \sum_{i \in [n]} \llbracket y=i \rrbracket \ell_{i}(v)$. The vector representation of the multi-class loss is given by $\ell:\mathcal{V}\rightarrow \mathbb{R}_+^n$, which assigns a vector $\ell(v)=(\ell_1(v),...,\ell_n(v))'$ to each prediction $v \in \mathcal{V}$. A loss is differentiable if all of its partial losses are differentiable. In this paper, we will use the bivariate function representation ($\ell(y,v)$) to denote a general loss function and the vector representation for multi-class loss functions.

The \textit{super-prediction set} of a binary loss $\ell$ is defined as
\begin{equation*}
S_{\ell} := \{ x \in \mathbb{R}^n : \exists v \in \mathcal{V}, x \geq \ell(v) \},
\end{equation*}
where inequality is component-wise. For any dimension $n$ and $\beta \geq 0$, the $\beta$-exponential operator $E_{\beta}:[0,\infty]^n \rightarrow [0,1]^n$ is defined by $E_{\beta}(x):=(e^{-\beta x_1},...,e^{-\beta x_n})'$. For $\beta > 0$ it is clearly invertible with inverse $E_{\beta}^{-1}(z)=-\beta^{-1}(\ln{z_1},...,\ln{z_n})'$. The $\beta$-exponential transformation of the super-prediction set is given by
\begin{equation*}
E_{\beta}(S_{\ell}) := \{ (e^{-\beta x_1},...,e^{-\beta x_n})' \in \mathbb{R}^n : (x_1,...,x_n)' \in S_{\ell} \}, \quad \beta > 0.
\end{equation*}

A multi-class loss $\ell$ is \textit{convex} if $f(v)=\ell_{y}(v)$ is convex in $v$ for all $y \in [n]$, $\alpha$\textit{-exp-concave} (for $\alpha > 0$) if $f(v)=e^{-\alpha \ell_{y}(v)}$ is concave in $v$ for all $y \in [n]$, \textit{weakly mixable} if the super-prediction set $S_{\ell}$ is convex (\cite{kalnishkan2005weak}), and $\beta$\textit{-mixable} (for $\beta > 0$) if the set $E_{\beta}(S_{\ell})$ is convex (\cite{vovk2009prediction,vovk1995game}). The \textit{mixability constant} $\beta_{\ell}$ of a loss $\ell$ is the largest $\beta$ such that $\ell$ is $\beta$-mixable; i.e. $\beta_{\ell} := \sup{\{ \beta > 0 : \ell \, \mbox{is} \, \beta\mbox{-mixable}\}}$. If the loss function $\ell$ is $\alpha$-exp-concave (resp. $\beta$-mixable) then it is $\alpha'$-exp-concave for any $0 < \alpha' \leq \alpha$ (resp. $\beta'$-mixable for any $0 < \beta' \leq \beta$), and its $\lambda$-scaled version ($\lambda \ell$) for some $\lambda > 0$ is $\frac{\alpha}{\lambda}$-exp-concave (resp. $\frac{\beta}{\lambda}$-mixable). If the loss $\ell$ is $\alpha$-exp-concave, then it is convex and $\alpha$-mixable.

For a multi-class loss $\ell$, if the prediction space $\mathcal{V}=\Delta^n$ then it is said to be \textit{multi-class probability estimation (CPE) loss}, where the predicted values are directly interpreted as probability estimates: $\ell:\Delta^n \rightarrow \mathbb{R}_+^n$. We will say a multi-CPE loss is \textit{fair} whenever $\ell_{i}(e_i^n)=0$, for all $i \in [n]$. That is, there is no loss incurred for perfect prediction. Examples of multi-CPE losses include the \textit{square loss} $\ell_{i}^{\mathrm{sq}}(q):=\sum_{j \in [n]}(\llbracket i=j \rrbracket - q_j)^2$, the \textit{log loss} $\ell_{i}^{\mathrm{log}}(q):=-\log{q_i}$, the \textit{absolute loss} $\ell_{i}^{\mathrm{abs}}(q):=\sum_{j \in [n]}|\llbracket i=j \rrbracket - q_j|$, and the \textit{0-1 loss} $\ell_{i}^{\mathrm{01}}(q):=\llbracket i \in \arg \max_{j \in [n]} q_j \rrbracket$.

\subsection{Conditional and Full Risks}

Let $\textsf{X}$ and $\textsf{Y}$ be random variables defined on the instance space $\mathcal{X}$ and the outcome space $\mathcal{Y}=[n]$ respectively. Let $D$ be the joint distribution of $(\textsf{X},\textsf{Y})$ and for $x \in \mathcal{X}$, denote the conditional distribution by $p(x)=(p_1(x),...,p_n(x))'$ where $p_i(x):=P(\textsf{Y}=i|\textsf{X}=x), \, \forall{i \in [n]}$, and the marginal distribution by $M(x):=P(\textsf{X}=x)$. For any multi-CPE loss $\ell$, the \textit{conditional risk} is defined as
\begin{equation}
\label{condrisk}
L_{\ell} : \Delta^n \times \Delta^n \ni (p,q) \mapsto L_{\ell}(p,q)=\mathbb{E}_{\textsf{Y}\sim p}[\ell_{\textsf{Y}}(q)]=p' \cdot \ell(q)=\sum_{i \in [n]}p_i \ell_i(q) \in \mathbb{R}_+ \, ,
\end{equation}
where $\textsf{Y}\sim p$ represents a Multinomial distribution with parameter $p \in \Delta^n$. The \textit{full risk} of the estimator function $q:\mathcal{X}\rightarrow\Delta^n$ is defined as
\begin{equation*}
\mathbb{L}_{\ell}(p,q,M):=\mathbb{E}_{(\textsf{X},\textsf{Y})\sim D}[\ell_{\textsf{Y}}(q(\textsf{X}))]=\mathbb{E}_{\textsf{X}\sim M}[L_{\ell}(p(\textsf{X}),q(\textsf{X}))].
\end{equation*}
Furthermore the \textit{Bayes risk} is defined as
\begin{equation*}
\underline{\mathbb{L}}_{\ell}(p,M):=\inf_{q \in (\Delta^n)^{\mathcal{X}}}{\mathbb{L}_{\ell}(p,q,M)}=\mathbb{E}_{\textsf{X}\sim M}[\Lubar_{\ell}(p(\textsf{X}))],
\end{equation*}
where $\Lubar_{\ell}(p)=\inf_{q \in \Delta^n}{L_{\ell}(p,q)}$ is the \textit{conditional Bayes risk} and is always concave (\cite{gneiting2007strictly}). If $\ell$ is fair, $\Lubar_{\ell}(e_i^n)=\ell_{i}(e_i^n)=0$. One can understand the effect of choice of loss in terms of the conditional perspective (\cite{reid2011information}), which allows one to ignore the marginal distribution $M$ of $\textsf{X}$ which is typically unknown.

\subsection{Proper and Composite Losses}
\label{sec:proper}

A multi-CPE loss $\ell:\Delta^n \rightarrow \mathbb{R}_+^n$ is said to be \textit{proper} if for all $p \in \Delta^n$, $\Lubar_{\ell}(p) = L_{\ell}(p,p)=p' \cdot \ell(p)$ (\cite{buja2005loss}, \cite{gneiting2007strictly}), and \textit{strictly proper} if $\Lubar_{\ell}(p) < L_{\ell}(p,q)$ for all $p,q \in \Delta^n$ and $p \neq q$. It is easy to see that the log loss, square loss, and 0-1 loss are proper while absolute loss is not. Furthermore, both log loss and square loss are strictly proper while 0-1 loss is proper but not strictly proper.

Given a proper loss $\ell: \Delta^n \to \RR_+^n$ with differentiable Bayes conditional risk $\Lubar_\ell: \Delta^n \mapsto \RR_+$, in order to be able to calculate derivatives easily, following \cite{van2012mixability} we define
\begin{align}
  \Deltatil^n &:= \cbr{(p_1, \ldots, p_{n-1})' : p_i \ge 0, \ \sum_{i=1}^{n-1} p_i \le 1} \\
  \Pi_\Delta &: \RR_+^n \ni p = (p_1, \ldots, p_n)' \mapsto \ptil = (p_1, \ldots, p_{n-1})' \in
  \RR_+^{n-1} \\
  \Pi_\Delta^{-1} &: \Deltatil^n \ni \ptil = (\ptil_1, \ldots, \ptil_{n-1})' \mapsto p = (\ptil_1, \ldots, \ptil_{n-1}, 1 - \sum\nolimits_{i=1}^{n-1} \ptil_i)' \in \Delta^n \\
  \Lubartil_\ell &: \Deltatil^n \ni \ptil \mapsto \Lubar_\ell(\Pi_\Delta^{-1} (\ptil)) \in \RR_+ \\
  \elltil &: \Deltatil^n \ni \ptil \mapsto \Pi_\Delta (\ell (\Pi_\Delta^{-1}(\ptil))) \in \RR_+^{n-1}.
\end{align}
Let $\psitil: \Deltatil^n \to \mathcal{V} \subseteq \RR_+^{n-1}$ be continuous and strictly monotone (hence invertible) for some convex set $\mathcal{V}$.
It induces $\psi: \Delta^n \to \mathcal{V}$ via
\begin{align}
\label{eq:link_def}
  \psi := \psitil \circ \Pi_\Delta.
\end{align}
Clearly $\psi$ is continuous and invertible with $\psi^{-1} = \Pi_\Delta^{-1} \circ \psitil^{-1}$. We can now extend the notion of properness to the prediction space $\mathcal{V}$ from $\Delta^n$ using this link function. Given a proper loss $\ell: \Delta^n \to \RR_+^n$, a \textit{proper composite loss} $\ell^{\psi}:\mathcal{V} \rightarrow \mathbb{R}_{+}^n$ for multi-class probability estimation is defined as $\ell^{\psi}:=\ell \circ \psi^{-1} = \ell \circ \Pi_\Delta^{-1} \circ \psitil^{-1}$. We can easily see that the conditional Bayes risks of the composite loss $\ell^{\psi}$ and the underlying proper loss $\ell$ are equal ($\Lubar_{\ell}=\Lubar_{\ell^{\psi}}$). Every continuous proper loss has a convex super-prediction set (\cite{vernet2011composite}). Thus they are weakly mixable. Since by applying a link function the super-prediction set won't change (as it is just a re-parameterization), all proper composite losses are also weakly mixable.

\subsection{Game of Prediction with Expert Advice}
\label{sec:games}
Let $\mathcal{Y}$ be the outcome space, $\mathcal{V}$ be the prediction space, and $\ell:\mathcal{Y} \times \mathcal{V} \rightarrow \mathbb{R}_{+}$ be the loss function, then a \textit{game of prediction with expert advice} represented by the tuple ($\mathcal{Y}, \mathcal{V}, \ell$) can be described as follows: for each trial $t=1,...,T$,
\begin{itemize}
\item{$N$ experts make their prediction $v_t^1,...,v_t^N \in \mathcal{V}$}
\item{the learner makes his own decision $v_t \in \mathcal{V}$}
\item{the environment reveals the actual outcome $y_t \in \mathcal{Y}$}
\end{itemize}
Let $S=(y_1,...,y_T)$ be the outcome sequence in $T$ trials. Then the \textit{cumulative loss} of the learner over $S$ is given by $L_{S,\ell}:=\sum_{t=1}^T \ell(y_t,v_t)$, of the \textit{i}-th expert is given by $L_{S,\ell}^i:=\sum_{t=1}^T \ell(y_t,v_t^i)$, and the \textit{regret} of the learner is given by $R_{S,\ell}:=L_{S,\ell}-\min_i L_{S,\ell}^i$. The goal of the learner is to predict as well as the best expert; to which end the learner tries to minimize the regret.

When using the exponential weights algorithm (which is an important family of algorithms in game of prediction with expert advice), at the end of each trial, the weight of each expert is updated as $w_{t+1}^i \varpropto w_t^i \cdot e^{-\eta \ell(y_t,v_t^i)}$ for all $i \in [N]$, where $\eta$ is the learning rate and $w_t^i$ is the weight of the $i^{\mathrm{th}}$ expert at time $t$ (the weight vector of experts at time $t$ is denoted by $w_t=(w_t^1,...,w_t^N)'$). Then based on the weights of experts, their predictions are merged using different merging schemes to make the learner's prediction. The Aggregating Algorithm and the Weighted Average Algorithm are two important algorithms in the family of exponential weights algorithm.

Consider multi-class games with outcome space $\mathcal{Y}=[n]$. In the Aggregating Algorithm with learning rate $\beta$, first the loss vectors of the experts and their weights are used to make a \textit{generalized prediction} $g_t=(g_t(1),...,g_t(n))'$ which is given by
\begin{equation*}
g_t := E_{\beta}^{-1}\left( \mathrm{co}_{w_t}\{E_{\beta}\left( (\ell_1(v_t^i),...,\ell_n(v_t^i))' \right) \}_{i \in [N]} \right) = E_{\beta}^{-1}\left(\sum_{i \in [N]}w_t^i (e^{-\beta \ell_1(v_t^i)},...,e^{-\beta \ell_n(v_t^i)})' \right).
\end{equation*}
Then this generalized prediction is mapped into a permitted prediction $v_t$ via a \textit{substitution function} such that $(\ell_1(v_t),...,\ell_n(v_t))' \leq c(\beta) (g_t(1),...,g_t(n))'$, where the inequality is element-wise and the constant $c(\beta)$ depends on the learning rate. If $\ell$ is $\beta$-mixable, then $E_{\beta}(S_{\ell})$ is convex, so $\mathrm{co}\{E_{\beta}\left( \ell(v) \right) : v \in \mathcal{V} \} \subseteq E_{\beta}(S_{\ell})$, and we can always choose a substitution function with $c(\beta)=1$. Consequently for $\beta$-mixable losses, the learner of the Aggregating Algorithm is guaranteed to have regret bounded by $\frac{\log{N}}{\beta}$ (\cite{vovk1995game}).

In the Weighted Average Algorithm with learning rate $\alpha$, the experts' predictions are simply merged according to their weights to make the learner's prediction $v_t=\mathrm{co}_{w_t}\{v_t^i\}_{i \in [N]}$, and this algorithm is guaranteed to have a $\frac{\log{N}}{\alpha}$ regret bound for $\alpha$-exp-concave losses (\cite{kivinen1999averaging}). In either case it is preferred to have bigger values for the constants $\beta$ and $\alpha$ to have better regret bounds. Since an $\alpha$-exp-concave loss is $\beta$-mixable for some $\beta \geq \alpha$, the regret bound of the Weighted Average Algorithm is worse than that of the Aggregating Algorithm by a small constant factor. In \cite{vovk2001competitive}, it is noted that $(\ell_1(\mathrm{co}_{w_t}\{v_t^i\}_i),...,\ell_n(\mathrm{co}_{w_t}\{v_t^i\}_i))' \leq g_t = E_{\beta}^{-1}\left( \mathrm{co}_{w_t} \{ E_{\beta} \left( (\ell_1(v_t^i),...,\ell_n(v_t^i))' \right) \}_i \right)$ is always guaranteed only for $\beta$-exp-concave losses. Thus for $\alpha$-exp-concave losses, the Weighted Average Algorithm is equivalent to the Aggregating Algorithm with the weighted average of the experts' predictions as its substitution function and $\alpha$ as the learning rate for both algorithms.

Even though the choice of substitution function will not have any impact on the regret bound and the weight update mechanism of the Aggregating Algorithm, it will certainly have impact on the actual regret of the learner over a given sequence of outcomes. According to the results given in Appendix \ref{sec:subfunc} (where we have empirically compared some substitution functions), this impact on the actual regret varies depending on the outcome sequence, and in general the regret values for practical substitution functions don't differ much --- thus we can stick with a computationally efficient substitution function.

\section{Exp-Concavity of Proper Composite Losses}
\label{sec:expsec}
Exp-concavity of a loss is desirable for better (logarithmic) regret bounds in online convex optimization algorithms, and for efficient implementation of exponential weights algorithms. In this section we will consider whether one can always find a link function that can transform a $\beta$-mixable proper loss into $\beta$-exp-concave composite loss --- first by using the geometry of the set $E_\beta(S_\ell)$ (Section \ref{sec:geometry}), and then by using the characterization of the composite loss in terms of the associated Bayes risk (Sections \ref{sec:calculus}, and \ref{sec:link}).

\subsection{Geometric approach}
\label{sec:geometry}
In this section we will use the same construction used by \cite{van2012exp} to derive an explicit closed form of a link function that could re-parameterize any $\beta$-mixable loss into a $\beta$-exp-concave loss, under certain conditions which are explained below. Given a multi-class loss $\ell:\mathcal{V}\rightarrow \RR_+^n$, define
\begin{eqnarray}
\ell(\mathcal{V}) &:=& \{ \ell(v) : v \in \mathcal{V} \}, \\
\mathcal{B}_{\beta} &:=& \mathrm{co} E_\beta (\ell(\mathcal{V})).
\end{eqnarray}
For any $g \in \mathcal{B}_{\beta}$ let $c(g):=\sup{\{ c \geq 0 : (g+c\vone_n) \in \mathcal{B}_{\beta}\}}$. Then the ``north-east" boundary of the set $\mathcal{B}_{\beta}$ is given by $\partial_{\vone_n} \mathcal{B}_{\beta} := \{ g + c(g) : g \in \mathcal{B}_{\beta} \}$. The following proposition is the main result of this section.
\begin{proposition}
\label{geoprop}
Assume $\ell$ is strictly proper and it satisfies the condition : $\partial_{\vone_n} \mathcal{B}_{\beta} \subseteq E_\beta (\ell(\mathcal{V}))$ for some $\beta > 0$. Define $\psi(p) := J E_\beta(\ell(p))$ for all $p \in \Delta^n$, where $J=[I_{n-1},-\vone_{n-1}]$. Then $\psi$ is invertible, and $\ell \circ \psi^{-1}$ is $\beta$-exp-concave over $\psi(\Delta^n)$, which is a convex set.
\end{proposition}
The condition stated in the above proposition is satisfied by any $\beta$-mixable proper loss in the binary case ($n=2$), but it is not guaranteed in the multi-class case where $n > 2$. In the binary case the link function can be given as $\psi(p)=e^{-\beta \ell_1(p)} - e^{-\beta \ell_2(p)}$ for all $p \in \Delta^2$.

\input{multiclass}

\subsection{Calculus approach}
\label{sec:calculus}
Proper composite losses are defined by the proper loss $\ell$ and the link $\psi$. In this section we will characterize the exp-concave proper composite losses in terms of $(\textsf{H}\Lubartil_{\ell}(\tilde{p}),\textsf{D}\tilde{\psi}(\tilde{p}))$. The following proposition provides the identities of the first and second derivatives of the proper composite losses (\cite{vernet2011composite}).
\begin{proposition}
\label{multiderivatives}
For all $i \in [n]$, $\tilde{p} \in \mathring{\tilde{\Delta}}^n$ (the interior of ${\tilde{\Delta}}^n$), and $v=\tilde{\psi}(\tilde{p}) \in \mathcal{V} \subseteq \mathbb{R}_+^{n-1}$ (so $\tilde{p}={\tilde{\psi}}^{-1}(v)$),
\begin{eqnarray}
\textnormal{\textsf{D}}\ell_i^{\psi}(v) &=& - (e_i^{n-1}-\tilde{p})' \cdot k(\tilde{p}), \\
\textnormal{\textsf{H}}\ell_i^{\psi}(v) &=& - \left( (e_i^{n-1}-\tilde{p})' \otimes I_{n-1} \right) \cdot \textnormal{\textsf{D}}_v[k(\tilde{p})] + k(\tilde{p})' \cdot [\textnormal{\textsf{D}}\tilde{\psi}(\tilde{p})]^{-1},
\end{eqnarray}
where
\begin{equation}
\label{kpeq}
k(\tilde{p}):=-\textnormal{\textsf{H}}\Lubartil_\ell(\tilde{p}) \cdot [\textnormal{\textsf{D}}\tilde{\psi}(\tilde{p})]^{-1}.
\end{equation}
\end{proposition}
The term $k(\tilde{p})$ can be interpreted as the curvature of the Bayes risk function $\Lubartil_\ell$ relative to the rate of change of the link function $\tilde{\psi}$. In the binary case where $n=2$, above proposition reduces to
\begin{eqnarray}
(\ell_1^{\psi})'(v) &=& -(1-\tilde{p}) k(\tilde{p}) \quad ; \quad (\ell_2^{\psi})'(v) = \tilde{p} k(\tilde{p}) , \label{binaryfirstder} \\
(\ell_1^{\psi})''(v) &=& \frac{-(1-\tilde{p}) k'(\tilde{p}) + k(\tilde{p})}{\tilde{\psi}'(\tilde{p})} , \label{binarysecder1}\\
(\ell_2^{\psi})''(v) &=& \frac{\tilde{p} k'(\tilde{p}) + k(\tilde{p})}{\tilde{\psi}'(\tilde{p})} , \label{binarysecder2}
\end{eqnarray}
where $k(\tilde{p})=\frac{-{\Lubartil_\ell}''(\tilde{p})}{{\tilde{\psi}}'(\tilde{p})} \geq 0$ and so $\frac{d}{dv}k(\tilde{p})=\frac{d}{d\tilde{p}}k(\tilde{p}) \cdot \frac{d}{dv}\tilde{p}=\frac{k'(\tilde{p})}{{\tilde{\psi}}'(\tilde{p})}$.


A loss $\ell:\Delta^n \rightarrow \mathbb{R}_+^n$ is $\alpha$-exp-concave (i.e. $\Delta^n \ni q \mapsto \ell_y(q)$ is $\alpha$-exp-concave for all $y \in [n]$) if and only if the map $\Delta^n \ni q \mapsto L_{\ell}(p,q)=p' \cdot \ell(q)$ is $\alpha$-exp-concave for all $p \in \Delta^n$. It can be easily shown that the maps $v \mapsto \ell_y^\psi (v)$ are $\alpha$-exp-concave if and only if $\textsf{H}\ell_y^\psi (v) \succcurlyeq \alpha \textsf{D}\ell_y^\psi(v)' \cdot \textsf{D}\ell_y^\psi(v)$. By applying Proposition \ref{multiderivatives} we obtain the following characterization of the $\alpha$-exp-concavity of the composite loss $\ell^\psi$.
\begin{proposition}
\label{multiexpprop}
A proper composite loss $\ell^\psi=\ell \circ \psi^{-1}$ is $\alpha$-exp-concave (with $\alpha > 0$ and $v=\tilde{\psi}(\tilde{p})$) if and only if for all $\tilde{p} \in \mathring{\tilde{\Delta}}^n$ and for all $i \in [n]$
\begin{equation}
\left( (e_i^{n-1}-\tilde{p})' \otimes I_{n-1} \right) \cdot \textnormal{\textsf{D}}_v[k(\tilde{p})] \preccurlyeq k(\tilde{p})' \cdot [\textnormal{\textsf{D}}\tilde{\psi}(\tilde{p})]^{-1} - \alpha k(\tilde{p})' \cdot (e_i^{n-1}-\tilde{p}) \cdot (e_i^{n-1}-\tilde{p})' \cdot k(\tilde{p}). \label{multiexpcondition}
\end{equation}
\end{proposition}
Based on this characterization, we can determine which loss functions can be exp-concavified by a chosen link function and how much a link function can exp-concavify a given loss function. In the binary case ($n=2$), the above proposition reduces to the following.
\begin{proposition}
\label{complexversion}
Let $\tilde{\psi}:[0,1]\rightarrow \mathcal{V} \subseteq \mathbb{R}$ be an invertible link and $\ell:\Delta^2 \rightarrow \mathbb{R}_+^2$ be a strictly proper binary loss with weight function $w(\tilde{p}):=-\textnormal{\textsf{H}}\Lubartil_\ell(\tilde{p})=-{\Lubartil_\ell}''(\tilde{p})$. Then the binary composite loss $\ell^\psi := \ell \circ \Pi_\Delta^{-1} \circ \tilde{\psi}^{-1}$ is $\alpha$-exp-concave (with $\alpha > 0$) if and only if
\begin{equation}
\label{maineq}
-\frac{1}{\tilde{p}} + \alpha w(\tilde{p}) \tilde{p} \enspace \leq \enspace \frac{w'(\tilde{p})}{w(\tilde{p})} - \frac{{\tilde{\psi}}''(\tilde{p})}{{\tilde{\psi}}'(\tilde{p})} \enspace \leq \enspace \frac{1}{1-\tilde{p}} - \alpha w(\tilde{p}) (1-\tilde{p}), \quad \forall \tilde{p} \in (0,1).
\end{equation}
\end{proposition}

The following proposition gives an easier to check necessary condition for the binary proper losses that generate an $\alpha$-exp-concave (with $\alpha > 0$) binary composite loss given a particular link function. Since scaling a loss function will not affect what a sensible learning algorithm will do, it is possible to normalize the loss functions by normalizing their weight functions by setting $w(\frac{1}{2}) = 1$. By this normalization we are scaling the original loss function by $\frac{1}{w(\frac{1}{2})}$ and the super-prediction set is scaled by the same factor. If the original loss function is $\beta$-mixable (resp. $\alpha$-exp-concave), then the normalized loss function is $\beta w(\frac{1}{2})$-mixable (resp. $\alpha w(\frac{1}{2})$-exp-concave).
\begin{proposition}
\label{simplerversion}
Let $\tilde{\psi}:[0,1]\rightarrow \mathcal{V} \subseteq \mathbb{R}$ be an invertible link and $\ell:\Delta^2 \rightarrow \mathbb{R}_+^2$ be a strictly proper binary loss with weight function $w(\tilde{p}):=-\textnormal{\textsf{H}}\Lubartil_\ell(\tilde{p})=-{\Lubartil_\ell}''(\tilde{p})$ normalised such that $w(\frac{1}{2}) = 1$. Then the binary composite loss $\ell^\psi := \ell \circ \Pi_\Delta^{-1} \circ \tilde{\psi}^{-1}$ is $\alpha$-exp-concave (with $\alpha > 0$) only if
\begin{equation}
\label{compsimple}
\frac{{\tilde{\psi}}'(\tilde{p})}{\tilde{p} (2{\tilde{\psi}}'(\frac{1}{2}) - \alpha ({\tilde{\psi}}(\tilde{p}) - {\tilde{\psi}}(\frac{1}{2})))} \lesseqgtr w(\tilde{p}) \lesseqgtr \frac{{\tilde{\psi}}'(\tilde{p})}{(1 - \tilde{p}) (2{\tilde{\psi}}'(\frac{1}{2}) + \alpha ({\tilde{\psi}}(\tilde{p}) - {\tilde{\psi}}(\frac{1}{2})))}, \quad \forall \tilde{p} \in (0,1),
\end{equation}
where $\lesseqgtr$ denotes $\leq$ for $\tilde{p} \geq \frac{1}{2}$ and denotes $\geq$ for $\tilde{p} \leq \frac{1}{2}$.
\end{proposition}

Proposition~\ref{complexversion} provides necessary \textit{and} sufficient conditions for the exp-concavity of binary composite losses, whereas Proposition~\ref{simplerversion} provides simple necessary \textit{but not} sufficient conditions. By setting $\alpha = 0$ in all the above results we have obtained for exp-concavity, we recover the convexity conditions for proper and composite losses which are already derived by \cite{reid2010composite} for the binary case and \cite{vernet2011composite} for multi-class.

\subsection{Link functions}
\label{sec:link}

\begin{figure*}[t]
\begin{minipage}[b]{0.48\textwidth}
\includegraphics[width=1\textwidth]{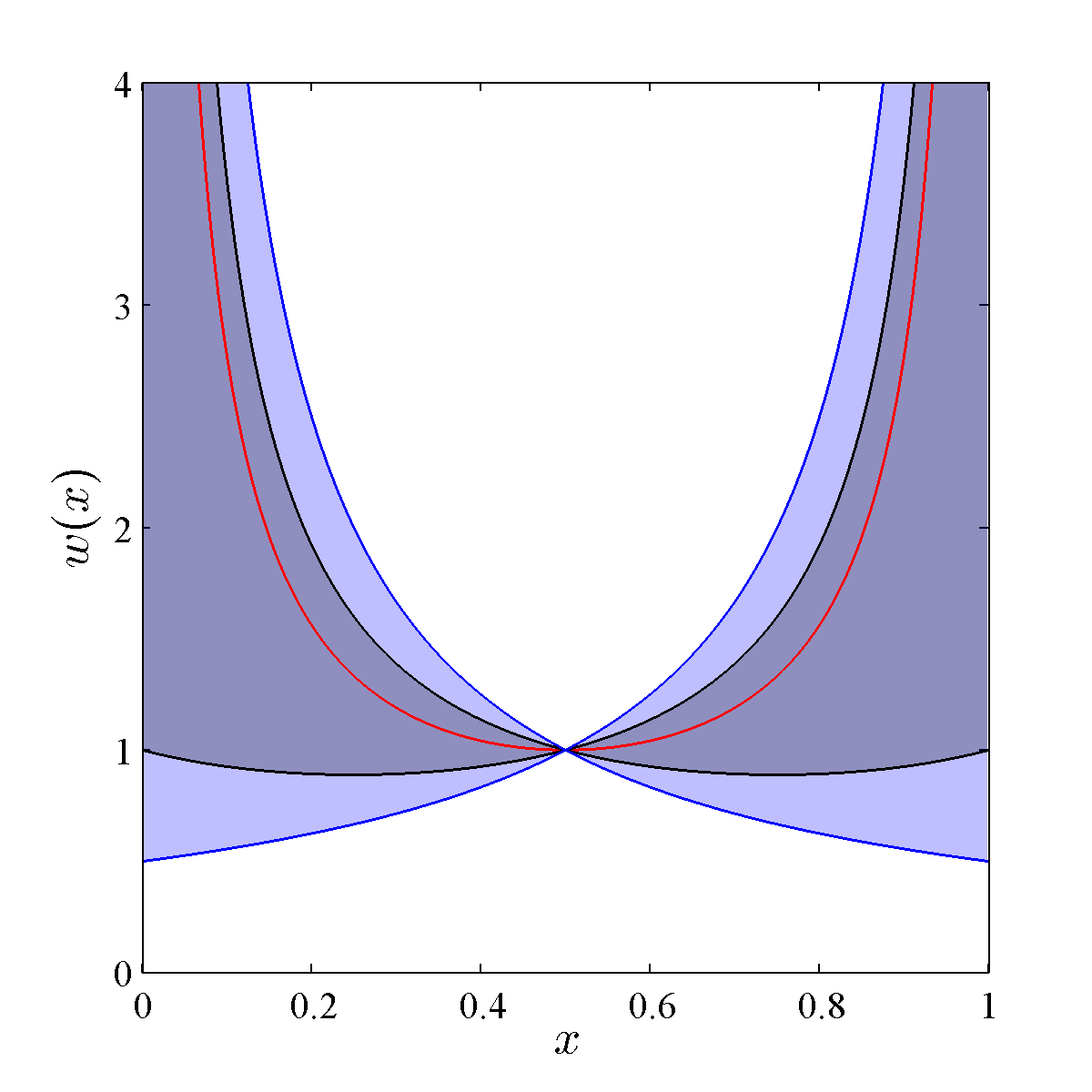}
  \caption{Necessary but not sufficient region of normalised weight functions to ensure $\alpha$-exp-concavity and convexity of proper losses (\textcolor{red}{---} $\alpha = 4$; \textcolor{black}{---} $\alpha = 2$; \textcolor{blue}{---} convexity).}
  \label{fig:identity}
\end{minipage}
~
\begin{minipage}[b]{0.48\textwidth}
\includegraphics[width=1\textwidth]{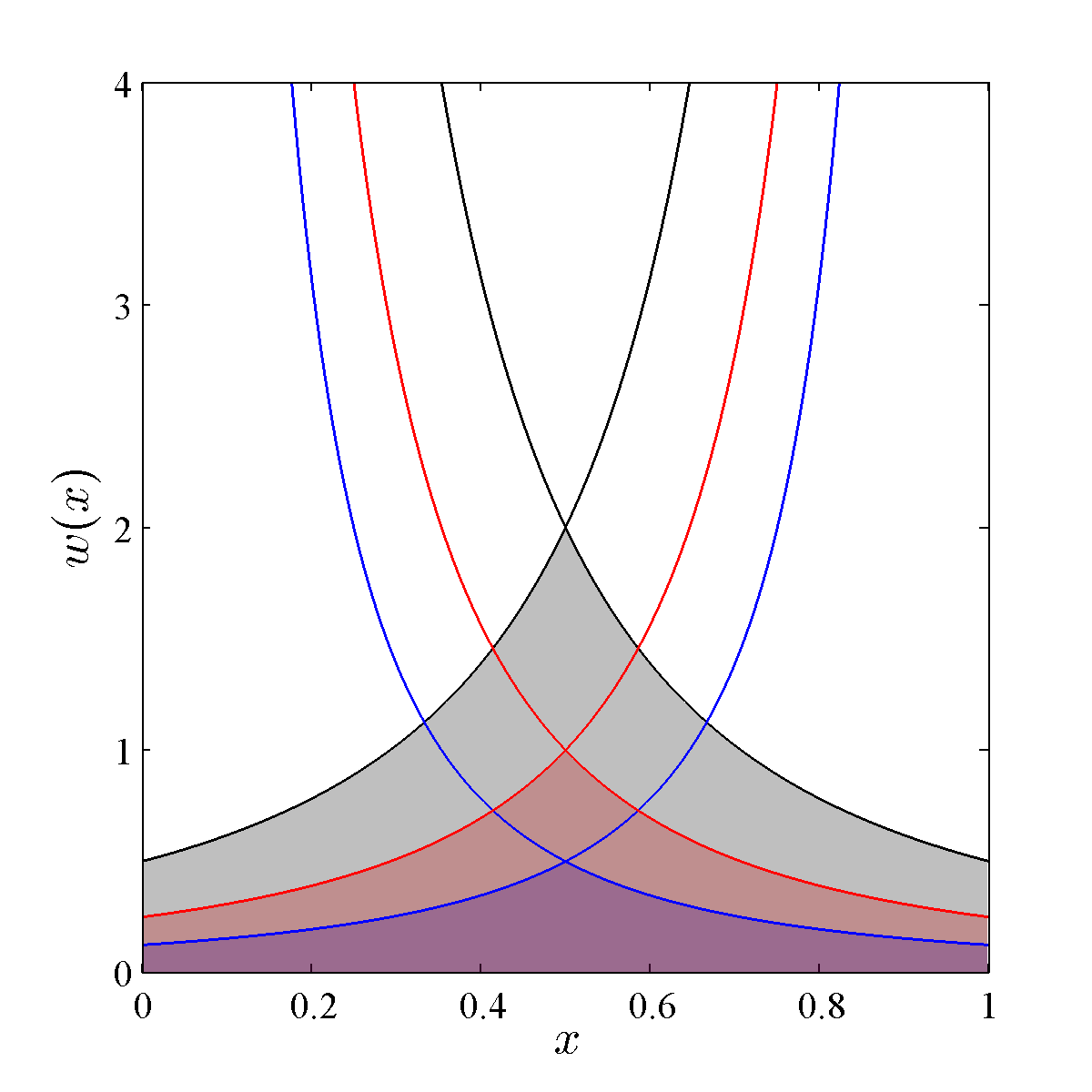}
  \caption{Necessary and sufficient region of unnormalised weight functions to ensure $\alpha$-exp-concavity of composite losses with canonical link (\textcolor{black}{---} $\alpha = 2$; \textcolor{red}{---} $\alpha = 4$; \textcolor{blue}{---} $\alpha = 8$).}
  \label{fig:can}
\end{minipage}
\end{figure*}

%
A proper loss can be exp-concavified ($\alpha > 0$) by some link function only if the loss is mixable ($\beta_{\ell} > 0$) and the maximum possible value for exp-concavity constant is the mixability constant of the loss (since the link function won't change the super-prediction set and an $\alpha$-exp-concave loss is always $\beta$-mixable for some $\beta \geq \alpha$).

By applying the \textit{identity link} ${\tilde{\psi}}(\tilde{p})=\tilde{p}$ in \eqref{maineq} we obtain the necessary and sufficient conditions for a binary proper loss to be $\alpha$-exp-concave (with $\alpha > 0$) as given by,
\begin{equation}
\label{identitynscond}
-\frac{1}{\tilde{p}} + \alpha w(\tilde{p}) \tilde{p} \enspace \leq \enspace \frac{w'(\tilde{p})}{w(\tilde{p})} \enspace \leq \enspace \frac{1}{1-\tilde{p}} - \alpha w(\tilde{p}) (1-\tilde{p}), \quad \forall \tilde{p} \in (0,1).
\end{equation}
By substituting ${\tilde{\psi}}(\tilde{p}) = \tilde{p}$ in \eqref{compsimple} we obtain the following necessary but not sufficient (simpler) constraints for a normalized binary proper loss to be $\alpha$-exp-concave
\begin{equation}
\label{propersimple}
\frac{1}{\tilde{p} (2 - \alpha (\tilde{p} - \frac{1}{2}))} \enspace \lesseqgtr \enspace w(\tilde{p}) \enspace \lesseqgtr \enspace \frac{1}{(1 - \tilde{p}) (2 + \alpha (\tilde{p} - \frac{1}{2}))}, \quad \forall \tilde{p} \in (0,1),
\end{equation}
which are illustrated as the shaded region in Figure~\ref{fig:identity} for different values of $\alpha$. Observe that normalized proper losses can be $\alpha$-exp-concave only for $0 < \alpha \leq 4$. When $\alpha = 4$, only the normalized weight function of log loss ($w_{\ell^{\mathrm{log}}}(\tilde{p})=\frac{1}{4\tilde{p}(1-\tilde{p})}$) will satisfy \eqref{propersimple}, and when $\alpha > 4$, the allowable (necessary) $w(\tilde{p})$ region to ensure $\alpha$-exp-concavity vanishes. Thus normalized log loss is the most exp-concave normalized proper loss. Observe (from Figure~\ref{fig:identity}) that normalized square loss ($w_{\ell^{\mathrm{sq}}}(\tilde{p})=1$) is at most 2-exp-concave. Further from \eqref{propersimple}, if $\alpha' > \alpha$, then the allowable $w(\tilde{p})$ region to ensure $\alpha'$-exp-concavity will be within the region for $\alpha$-exp-concavity, and also any allowable $w(\tilde{p})$ region to ensure $\alpha$-exp-concavity will be within the region for convexity, which is obtained by setting $\alpha = 0$ in \eqref{propersimple}. Here we recall the fact that, if the normalized loss function is $\alpha$-exp-concave, then the original loss function is $\frac{\alpha}{w(\frac{1}{2})}$-exp-concave. The following theorem provides \textit{sufficient} conditions for the exp-concavity of binary proper losses.

\begin{theorem}
\label{propexpsuff}
A binary proper loss $\ell:\Delta^2 \rightarrow \mathbb{R}_+^2$ with the weight function $w(\tilde{p})=-{\Lubartil_\ell}''(\tilde{p})$ normalized such that $w(\frac{1}{2})=1$ is $\alpha$-exp-concave (with $\alpha > 0$) if
\begin{align*}
	w(\tilde{p}) &~=~ \frac{1}{\tilde{p} \left( 2 - \int_{\tilde{p}}^{1/2}a(t)dt \right)} ~\text{ for } a(\tilde{p}) \text{ s.t. } \\
	& \left[ \frac{\alpha (1-\tilde{p})}{\tilde{p}} - \frac{2}{\tilde{p}(1-\tilde{p})} \right] + \frac{1}{\tilde{p}(1-\tilde{p})} \int_{\tilde{p}}^{1/2}a(t)dt ~\leq~ a(\tilde{p}) ~\leq~ -\alpha, \, \forall{\tilde{p} \in (0,1/2]}, \\
	\intertext{and}
	w(\tilde{p}) &~=~ \frac{1}{(1-\tilde{p}) \left( 2 - \int_{\frac{1}{2}}^{\tilde{p}}b(t)dt \right)} ~\text{ for } b(\tilde{p}) \text{ s.t. } \\
	& \left[ \frac{\alpha \tilde{p}}{(1-\tilde{p})} - \frac{2}{\tilde{p}(1-\tilde{p})} \right] + \frac{1}{\tilde{p}(1-\tilde{p})} \int_{1/2}^{\tilde{p}}b(t)dt ~\leq~ b(\tilde{p}) ~\leq~ -\alpha, \, \forall{\tilde{p} \in [\textstyle\frac{1}{2},1)}.
\end{align*}
\end{theorem}
For square loss we can find that $a(\tilde{p})=\frac{-1}{\tilde{p}^2}$ and $b(\tilde{p})=\frac{-1}{(1-\tilde{p})^2}$ will satisfy the above sufficient condition with $\alpha=2$ and for log loss $a(\tilde{p})=b(\tilde{p})=-4$ will satisfy the sufficient condition with $\alpha=4$. It is also easy to see that for symmetric losses $a(\tilde{p})$ and $b(\tilde{p})$ will be symmetric.

When the \textit{canonical link} function ${\tilde{\psi}}_{\ell}(\tilde{p}) := - \textsf{D}\Lubartil_\ell(\tilde{p})'$ is combined with a strictly proper loss to form $\ell^{\psi_\ell}$, since $\textsf{D}{\tilde{\psi}}_{\ell}(\tilde{p}) = - \textsf{H}\Lubartil_\ell(\tilde{p})$, the first and second derivatives of the composite loss become considerably simpler as follows
\begin{eqnarray}
\textsf{D}\ell_i^{\psi_\ell}(v) &=& - (e_i^{n-1}-\tilde{p})', \label{canlinkcond1} \\
\textsf{H}\ell_i^{\psi_\ell}(v) &=& - [\textsf{H}\Lubartil_\ell(\tilde{p})]^{-1}. \label{canlinkcond2}
\end{eqnarray}
Since a proper loss $\ell$ is $\beta$-mixable if and only if $\beta \textsf{H}\Lubartil_\ell(\tilde{p}) \succcurlyeq \textsf{H}\Lubartil_{\ell^{\mathrm{log}}}(\tilde{p})$ for all $\tilde{p} \in \mathring{\tilde{\Delta}}^n$ (\cite{van2012mixability}), by applying the canonical link any $\beta$-mixable proper loss will be transformed to $\alpha$-exp-concave proper composite loss (with $\beta \geq \alpha > 0$) but $\alpha = \beta$ is not guaranteed in general. In the binary case, since ${\tilde{\psi}}_{\ell}'(\tilde{p}) = - {\Lubartil_\ell}''(\tilde{p}) = w(\tilde{p})$, we get
\begin{equation}
\label{binarycanonical}
w(\tilde{p}) \leq \frac{1}{\alpha \tilde{p}^2} \quad \text{and} \quad w(\tilde{p}) \leq \frac{1}{\alpha (1-\tilde{p})^2}, \quad \forall \tilde{p} \in (0,1),
\end{equation}
as the necessary and sufficient conditions for $\ell^{\psi_{\ell}}$ to be $\alpha$-exp-concave. In this case when $\alpha \rightarrow \infty$ the allowed region vanishes (since for proper losses $w(\tilde{p}) \geq 0$). From Figure~\ref{fig:can} it can be seen that, if the normalized loss function satisfies
\begin{equation*}
w(\tilde{p}) \leq \frac{1}{4\tilde{p}^2} \quad \text{and} \quad w(\tilde{p}) \leq \frac{1}{4(1-\tilde{p})^2}, \quad \forall \tilde{p} \in (0,1),
\end{equation*}
then the composite loss obtained by applying the canonical link function on the unnormalized loss with weight function $w_{\mathrm{org}}(\tilde{p})$ is $\frac{4}{w_{\mathrm{org}}(\frac{1}{2})}$-exp-concave.

We now consider whether one can always find a link function that can transform a $\beta$-mixable proper loss into $\beta$-exp-concave composite loss. In the binary case, such a link function exists and is given in the following corollary.
\begin{corollary}
\label{specialcoro}
Let $w_{\ell}(\tilde{p})=-{\Lubartil_\ell}''(\tilde{p})$. The exp-concavifying link function $\tilde{\psi}_{\ell}^*$ defined via
\begin{equation}
\tilde{\psi}_{\ell}^*(\tilde{p}) = \frac{w_{\ell^{\mathrm{log}}}(\frac{1}{2})}{w_{\ell}(\frac{1}{2})} \int_{0}^{\tilde{p}}{\frac{w_{\ell}(v)}{w_{\ell^{\mathrm{log}}}(v)}dv}, \quad \forall{\tilde{p} \in [0,1]}
\end{equation}
(which is a valid strictly increasing link function) will always transform a $\beta$-mixable proper loss $\ell$ into $\beta$-exp-concave composite loss $\ell^{\psi_{\ell}^*}$, where $\ell^{\psi_{\ell}^*}_y(v)=\ell_y \circ \Pi_\Delta^{-1} \circ (\tilde{\psi}_{\ell}^*)^{-1}(v)$.
\end{corollary}
For log loss, the exp-concavifying link is equal to the identity link and the canonical link could be written as $\int_{0}^{\tilde{p}}{w_{\ell}(v)dv}$. If $\ell$ is a binary proper loss with weight function $w_\ell(\tilde{p})$, then we can define a new proper loss $\ell^{\mathrm{mix}}$ with weight function $w_{\ell^{\mathrm{mix}}}(\tilde{p})=\frac{w_\ell(\tilde{p})}{w_{\ell^{\mathrm{log}}}(\tilde{p})}$. Then applying the exp-concavifying link $\tilde{\psi}_{\ell}^*$ on the original loss $\ell$ is equivalent to applying the canonical link ${\tilde{\psi}}_{\ell}$ on the new loss $\ell^{\mathrm{mix}}$.

The links constructed by the geometric and calculus approaches can be completely different (see Appendix D-F). The former can be further varied by replacing $\one_n$ with any direction in the positive orthant, and the latter can be arbitrarily rescaled. Furthermore, as both links satisfy \eqref{maineq} with $\alpha=\beta$, any appropriate interpolation also works.

\section{Conclusions}
\label{sec:conc}

If a loss is $\beta$-mixable, one can run the Aggregating Algorithm with learning rate $\beta$ and obtain a $\frac{\log N}{\beta}$ regret bound. Similarly a $\frac{\log N}{\alpha}$ regret bound can be attained by the Weighted Average Algorithm with learning rate $\alpha$, when the loss is $\alpha$-exp-concave. \cite{vovk2001competitive} observed that the weighted average of the expert predictions (\cite{kivinen1999averaging}) will be a \textit{perfect} (in the technical sense defined in \cite{vovk2001competitive}) substitution function for the Aggregating Algorithm if and only if the loss function is exp-concave. Thus if we have to use a proper, mixable but non-exp-concave loss function $\ell$ for a sequential prediction (online learning) problem, an $O(1)$ regret bound could be achieved by the following two approaches:
\begin{itemize}
\item{Use the Aggregating Algorithm (\cite{vovk1995game}) with the \textit{inverse loss} $\ell^{-1}$ (\cite{williamson2014geometry}) as the universal substitution function.}
\item{Apply the exp-concavifying link ($\tilde{\psi}_{\ell}^*$) on $\ell$, derive the $\beta_{\ell}$-exp-concave composite loss $\ell^{\psi_{\ell}^{*}}$. Then use the Weighted Average Algorithm (\cite{kivinen1999averaging}) with $\ell^{\psi_{\ell}^{*}}$ to obtain the learner's prediction in the transformed domain ($v_{\mathrm{avg}} \in \psi_{\ell}^{*}(\tilde{\Delta}^n)$). Finally output the inverse link value of this prediction ($(\psi_{\ell}^{*})^{-1}(v_{\mathrm{avg}})$).}
\end{itemize}
In either approach we are faced with a computational problem of evaluating an inverse function. But in the binary class case the inverse of a strictly monotone function can be efficiently evaluated using one sided bisection method (or lookup table). So in conclusion, the latter approach can be more convenient and efficient in computation than the former.


\acks{This work was supported by the ARC and NICTA, which is funded by the Australian Government through the Department of Communications and the Australian Research Council through the ICT Centre of Excellence Program. Thanks to the referees for helpful comments.
}

\bibliographystyle{plainnat}
\bibliography{reflist}

\appendix

\section{Substitution Functions}
\label{sec:subfunc}

\begin{figure}
\centering
    \includegraphics[width=8cm]{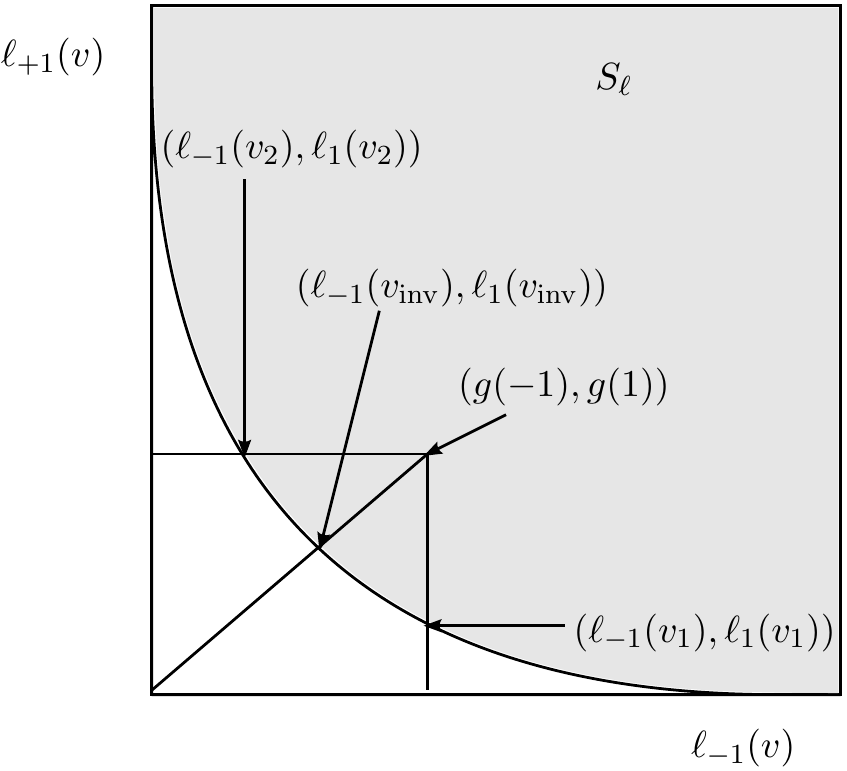}
  \caption{Super-prediction set ($S_\ell$) of a binary game $(\{-1,1\},[0,1],\ell)$, where a generalized prediction is represented by $(g(-1),g(1))$.}
  \label{fig:losscurve}
\end{figure}

We consider the following choices of substitution functions --- \textit{Best look ahead}, \textit{Worst look ahead}, \textit{Inverse loss} (\cite{williamson2014geometry}) and \textit{Weighted average} (\cite{kivinen1999averaging}). The first two choices are just hypothetical ones as they look ahead the actual outcome first and then choose their predictions to incur low and high loss respectively, still they give us an idea about how well or worse the Aggregating Algorithm performs over a given outcome sequence. The weighted average is a computationally efficient and easily implementable substitution function, but it is applicable only for exp-concave losses.

For a binary game represented by $\left(\{-1,1\},[0,1],\ell \right)$ and shown in the Figure \ref{fig:losscurve},
\begin{itemize}
\item{If the outcome $y=-1$, the Best look ahead and the Worst look ahead will choose the predictions $v_2$ and $v_1$ and incur losses $\ell_{-1}(v_2)$ and $\ell_{-1}(v_1)=g(-1)$ respectively; and if $y=1$, they will choose $v_1$ and $v_2$ and suffer losses $\ell_{1}(v_1)$ and $\ell_{1}(v_2)=g(1)$ respectively.}
\item{The Inverse loss will choose the prediction $v_{\mathrm{inv}}$ such that $\frac{\ell_{1}(v_{\mathrm{inv}})}{\ell_{-1}(v_{\mathrm{inv}})}=\frac{g(1)}{g(-1)}$ and will incur a loss $\ell(y,v_{\mathrm{inv}})$, and the Weighted average will choose $v_{\mathrm{avg}}=\sum_i w^i v^i$ (where $w^i$ and $v^i$ are the weight and the prediction of the $i$-the expert respectively) and will incur a loss $\ell(y,v_{\mathrm{avg}})$.}
\end{itemize}
Further if the loss function $\ell$ is chosen to be the square loss (which is both 2-mixable and $\frac{1}{2}$-exp-concave), then we have $v_1=\sqrt{g(-1)}$ (since $\ell_{-1}(v_1)=v_1^2=g(-1)$), $v_2=1-\sqrt{g(1)}$ (since $\ell_{1}(v_2)=(1-v_2)^2=g(1)$), and $v_{\mathrm{inv}}=\frac{\sqrt{g(-1)}}{\sqrt{g(-1)}+\sqrt{g(1)}}$ (since $\frac{(1-v_{\mathrm{inv}})^2}{v_{\mathrm{inv}}^2}=\frac{g(1)}{g(-1)}$). Thus for a binary square loss game over an outcome sequence  $y_1,...,y_T$, the cumulative losses of the Aggregating Algorithm for different choices of substitution function are given as follows:
\begin{itemize}
\item{Best look ahead: $\sum_1^T \left(1-\sqrt{g_t(-y_t)}\right)^2$}
\item{Worst look ahead: $\sum_1^T g_t(y_t)$}
\item{Inverse loss: $\sum_1^T \left(y_t-\frac{\sqrt{g_t(0)}}{\sqrt{g_t(0)}+\sqrt{g_t(1)}}\right)^2$}
\item{Weighted average: $\sum_1^T \left(y_t-\sum_i w_i^t v_i^t\right)^2$}
\end{itemize}

Some experiments are conducted on a binary square loss game to compare these substitution functions. For this, binary outcome sequences of 100 elements are generated using the Bernoulli distribution with success probabilities 0.5, 0.7, 0.9, and 1.0 (these sequences are represented by $\{y_t\}_{p=0.5}$, $\{y_t\}_{p=0.7}$, $\{y_t\}_{p=0.9}$ and $\{y_t\}_{p=1.0}$ respectively). Furthermore the following expert settings are used:
\begin{itemize}
\item{2 experts where one expert always make the prediction $v=0$, and the other one always makes the prediction $v=1$. This setting is represented by $\{E_t\}_{\text{set.}1}$.}
\item{3 experts where two experts are as in the previous setting, and the other one is always accurate expert. This setting is represented by $\{E_t\}_{\text{set.}2}$.}
\item{101 constant experts where the prediction values of the experts are from 0 to 1 with equal interval. This setting is represented by $\{E_t\}_{\text{set.}3}$.}
\end{itemize}
The results of these experiments are presented in the figures \ref{fig:p0d5},\ref{fig:p0d7},\ref{fig:p0d9}, and \ref{fig:p1d0}. From these figures, it can be seen that for the expert setting $\{E_t\}_{\text{set.}1}$, the difference between the regret values of the worst look ahead and the best look ahead substitution functions relative to the theoretical regret bound is very high, whereas that relative difference is very low for the expert setting $\{E_t\}_{\text{set.}3}$. Further the performance of the Aggregating Algorithm over a real dataset is shown in the Figure \ref{fig:footballdata}. From these results for both simulated dataset (for all three expert settings) and real dataset, observe that the difference between the regret values of the inverse loss and the weighted average substitution functions relative to the theoretical regret bound is very low.

\begin{figure}[htbp]
\floatconts
  {fig:p0d5}
  {\caption{Cumulative regret of the Aggregating Algorithm over the outcome sequence $\{y_t\}_{p=0.5}$ for different choices of substitution functions (Best look ahead(\textcolor{blue}{---}), Worst look ahead(\textcolor{black}{---}), Inverse loss(\textcolor{green}{---}), and Weighted average(\textcolor{magenta}{---})) with learning rate $\eta$ and expert setting $\{E_t\}_i$ (theoretical regret bound is shown by \textcolor{red}{- - -}).}}
  {
    \subfigure[$\eta = 0.1, \{E_t\}_{\text{set.}1}$]{\label{fig:p0d5_n0d1_e1}
      \includegraphics[width=0.32\linewidth]{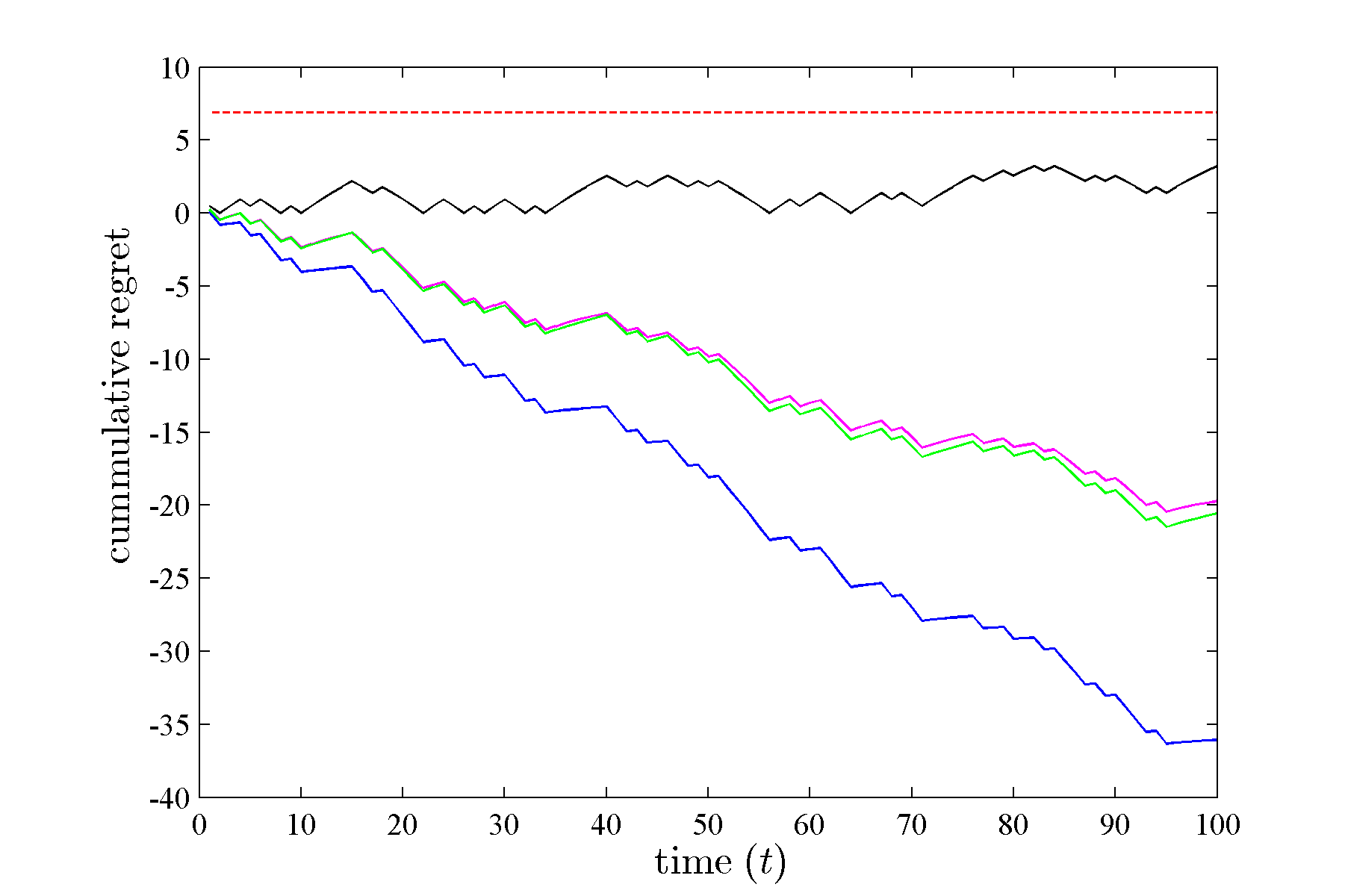}}
    \subfigure[$\eta = 0.3, \{E_t\}_{\text{set.}1}$]{\label{fig:p0d5_n0d3_e1}
      \includegraphics[width=0.32\linewidth]{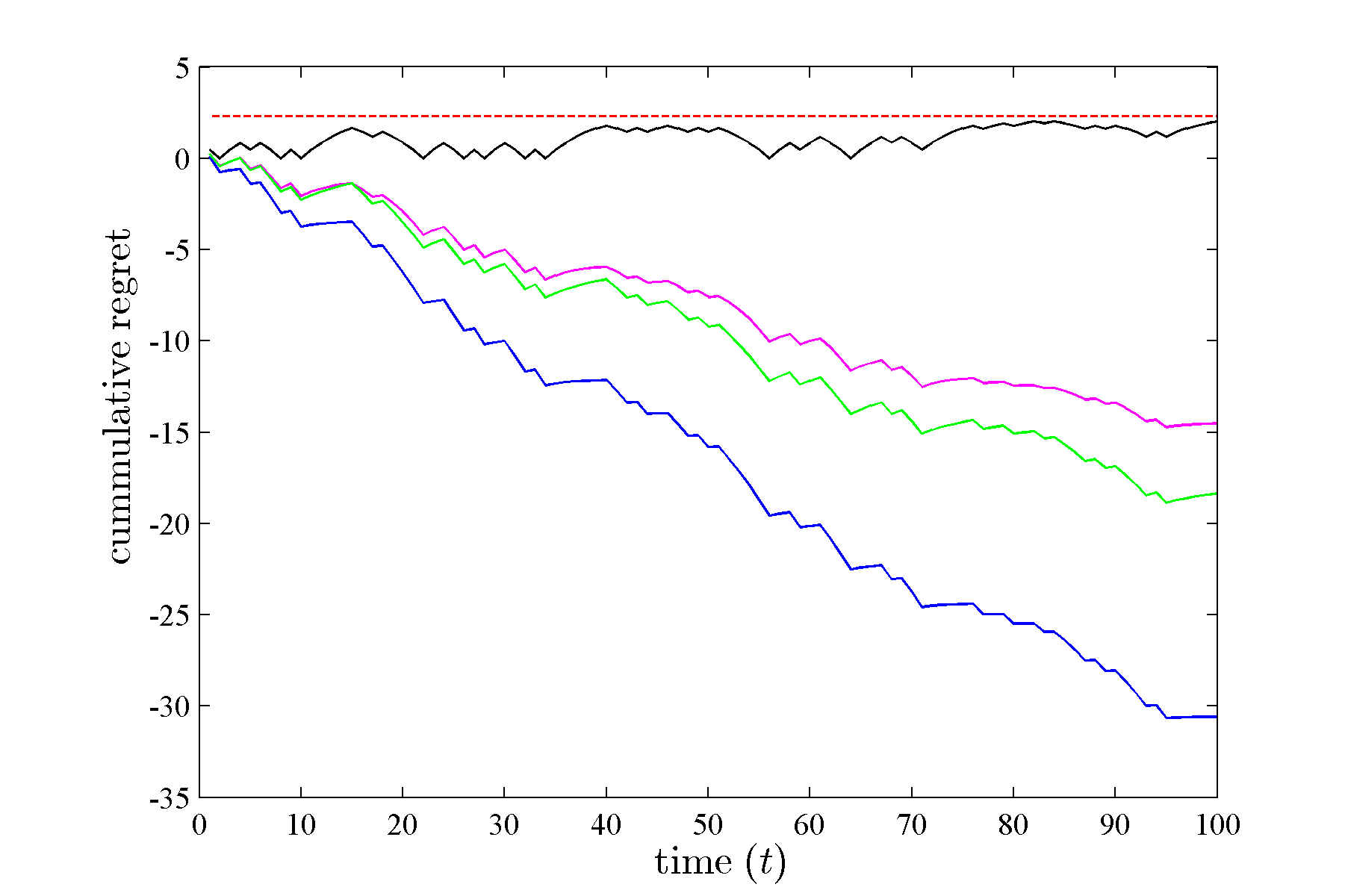}}
    \subfigure[$\eta = 0.5, \{E_t\}_{\text{set.}1}$]{\label{fig:p0d5_n0d5_e1}
      \includegraphics[width=0.32\linewidth]{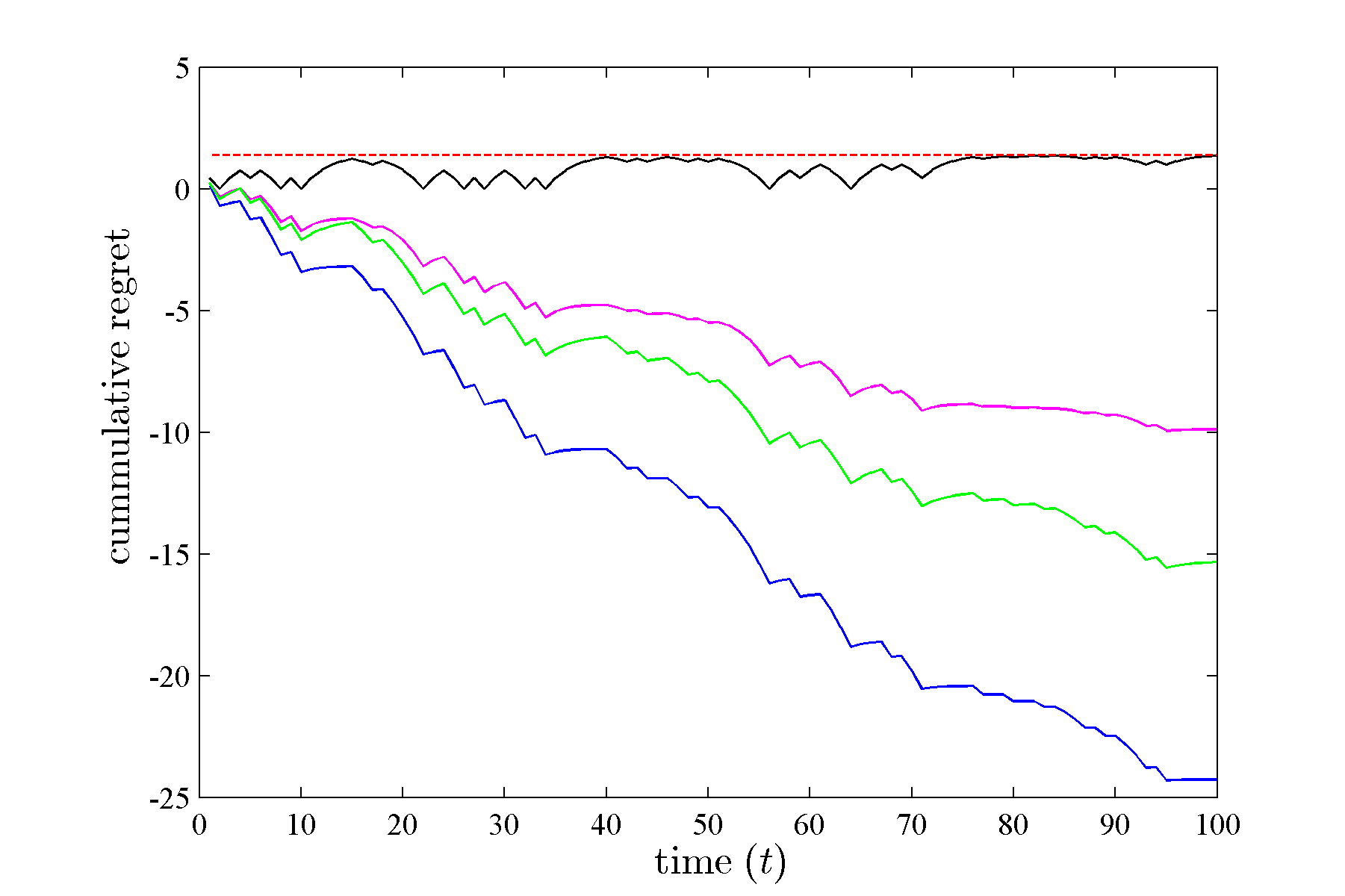}}

    \subfigure[$\eta = 0.1, \{E_t\}_{\text{set.}2}$]{\label{fig:p0d5_n0d1_e2}
      \includegraphics[width=0.32\linewidth]{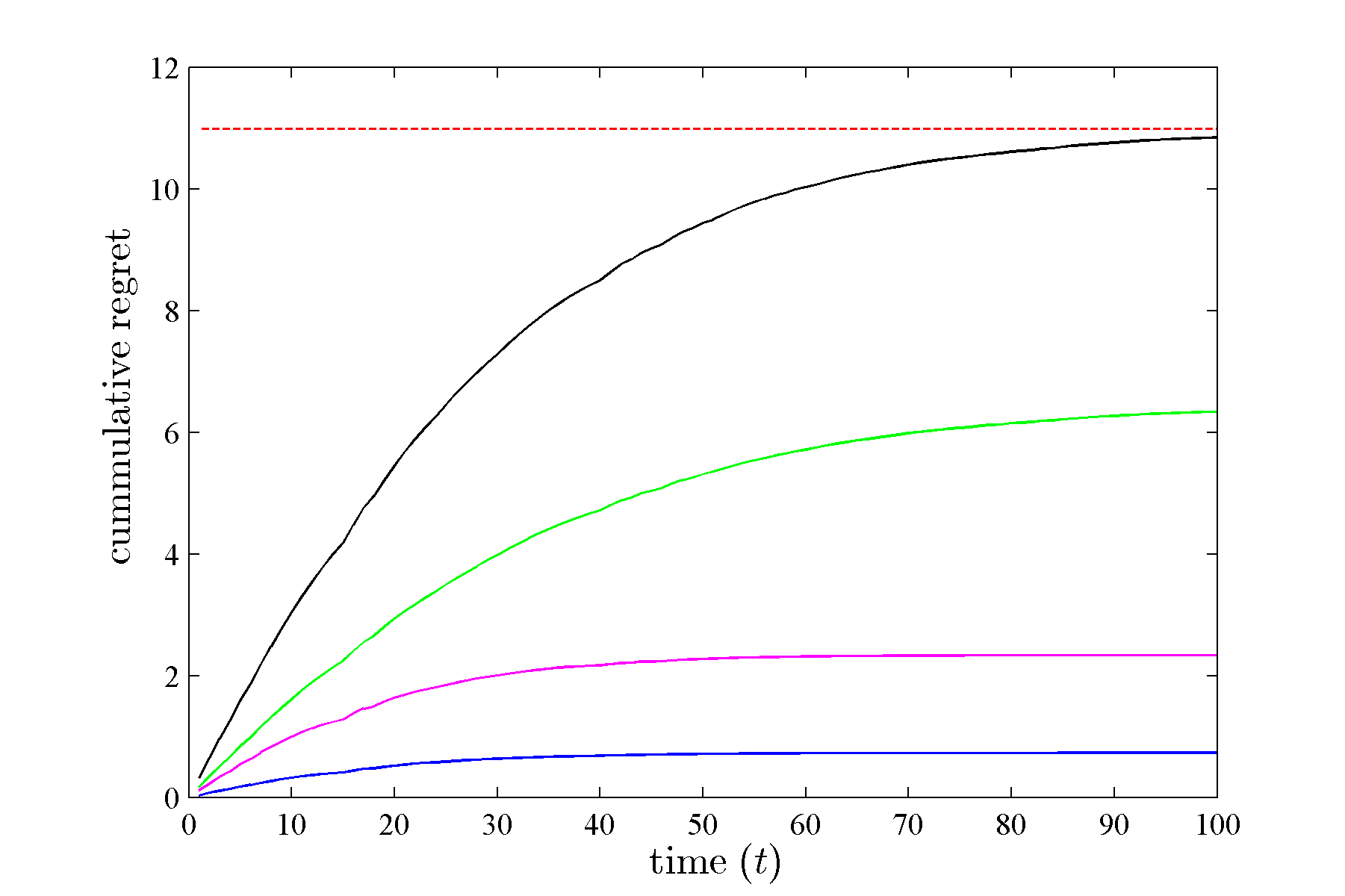}}
    \subfigure[$\eta = 0.3, \{E_t\}_{\text{set.}2}$]{\label{fig:p0d5_n0d3_e2}
      \includegraphics[width=0.32\linewidth]{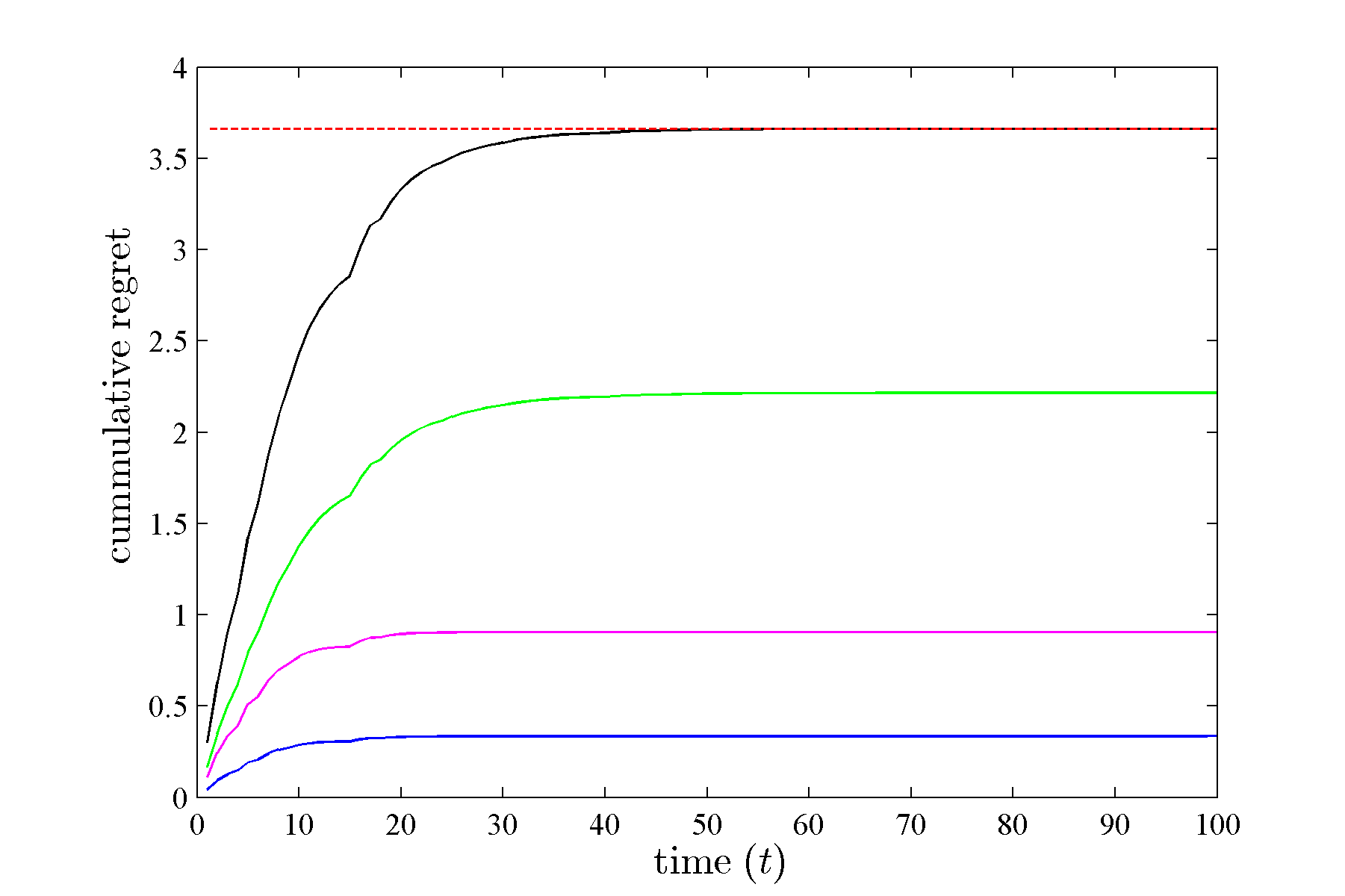}}
    \subfigure[$\eta = 0.5, \{E_t\}_{\text{set.}2}$]{\label{fig:p0d5_n0d5_e2}
      \includegraphics[width=0.32\linewidth]{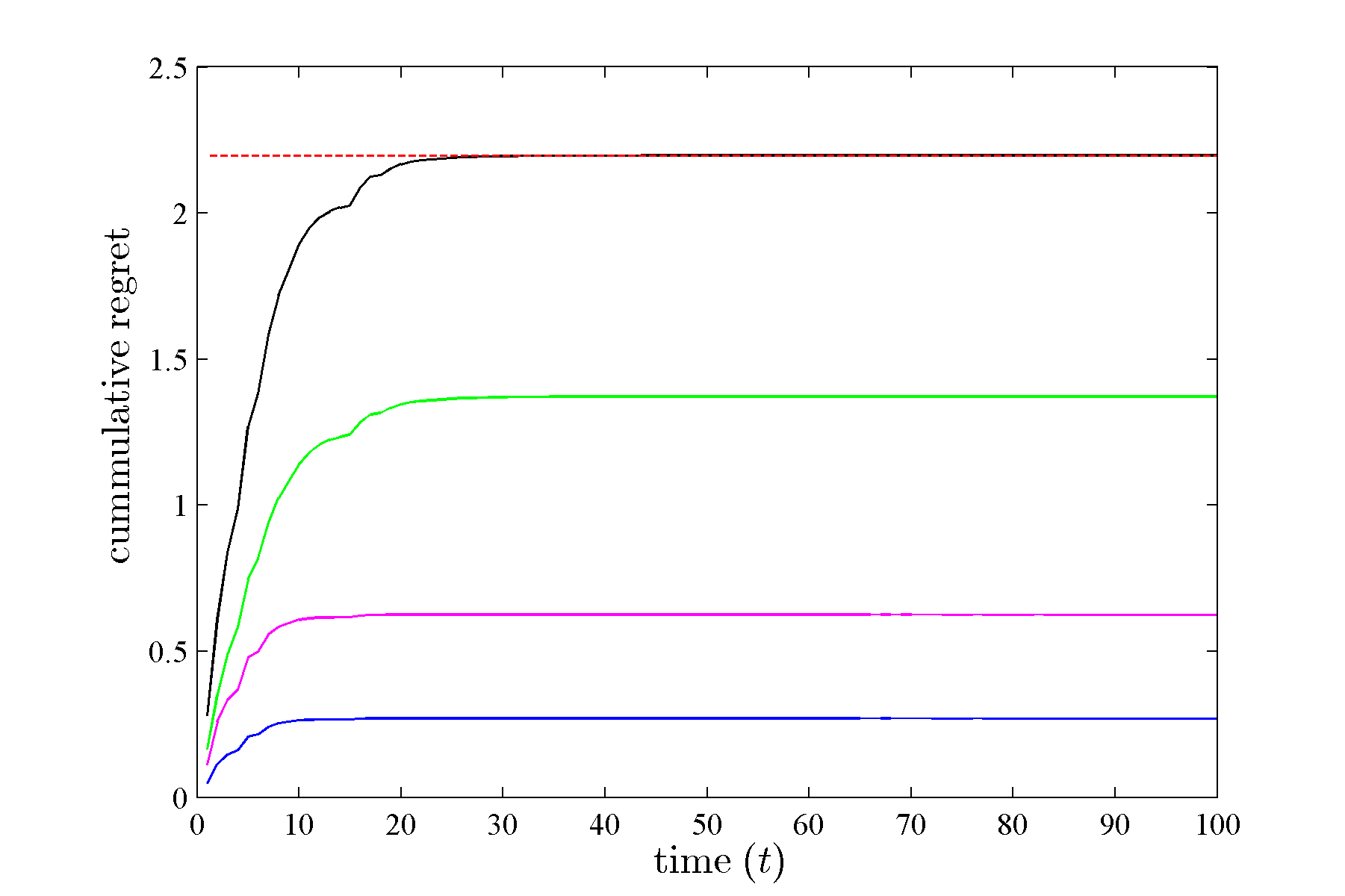}}

    \subfigure[$\eta = 0.1, \{E_t\}_{\text{set.}3}$]{\label{fig:p0d5_n0d1_e3}
      \includegraphics[width=0.32\linewidth]{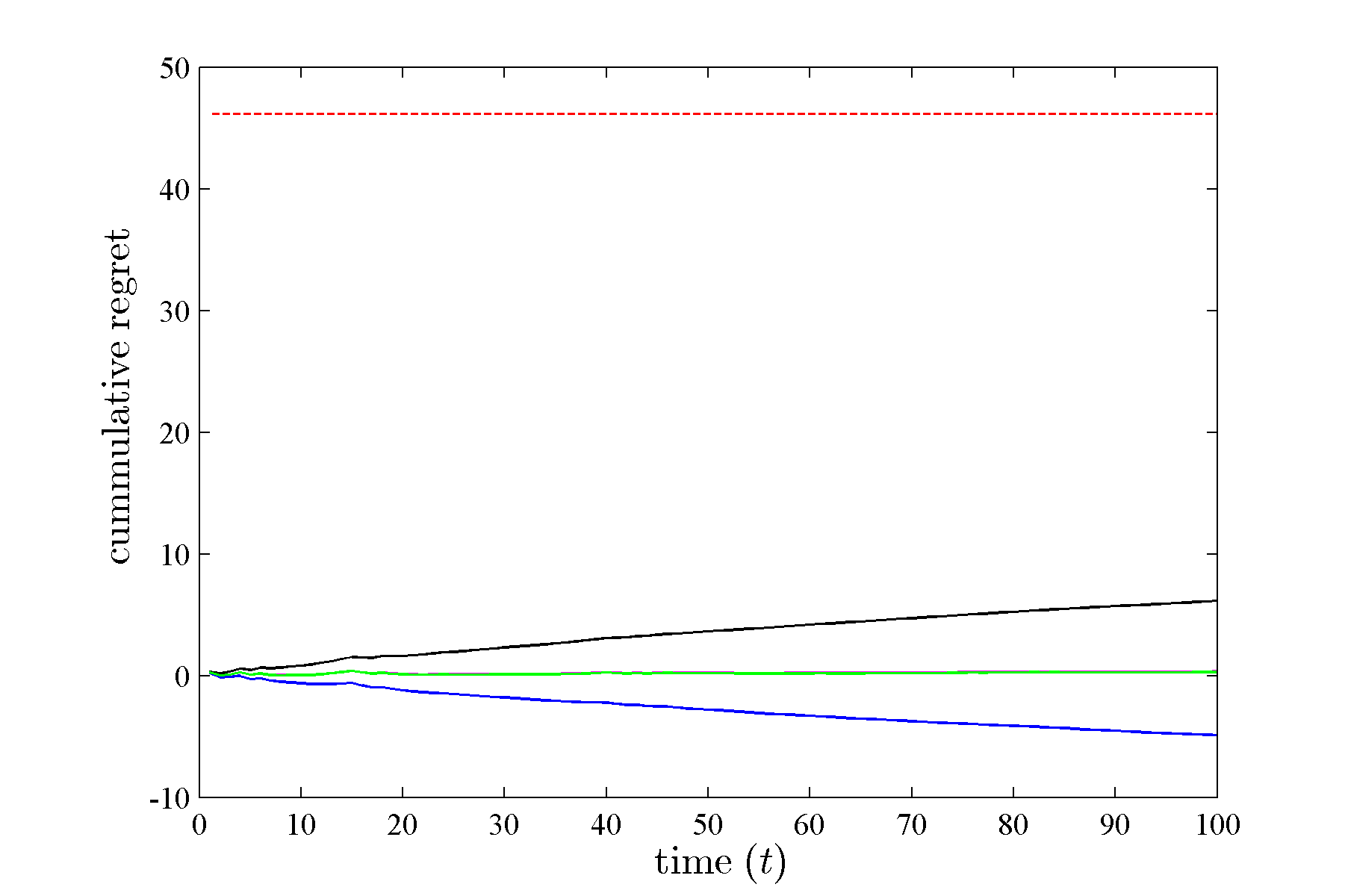}}
    \subfigure[$\eta = 0.3, \{E_t\}_{\text{set.}3}$]{\label{fig:p0d5_n0d3_e3}
      \includegraphics[width=0.32\linewidth]{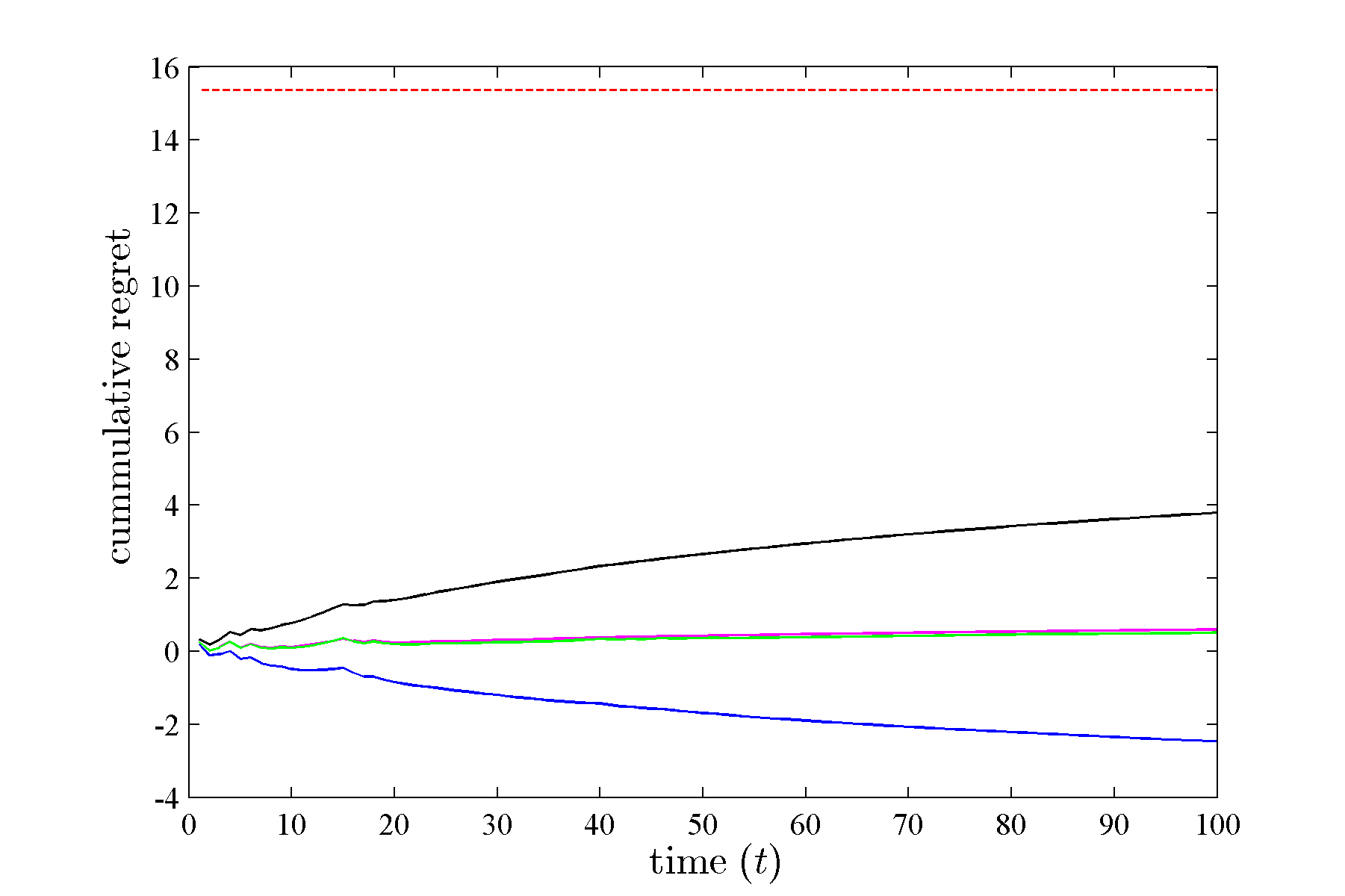}}
    \subfigure[$\eta = 0.5, \{E_t\}_{\text{set.}3}$]{\label{fig:p0d5_n0d5_e3}
      \includegraphics[width=0.32\linewidth]{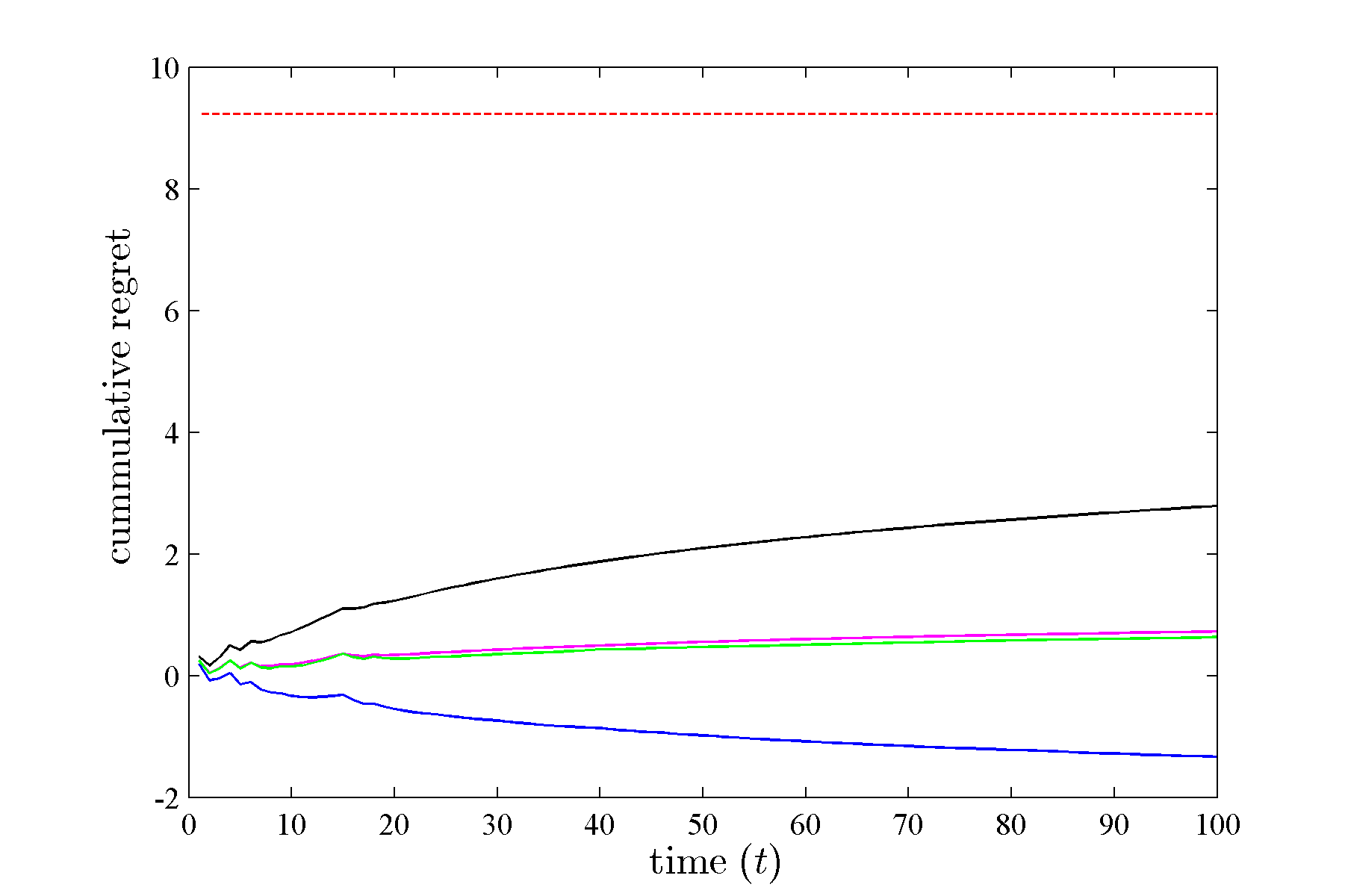}}
  }
\end{figure}

\begin{figure}[htbp]
\floatconts
  {fig:p0d7}
  {\caption{Cumulative regret of the Aggregating Algorithm over the outcome sequence $\{y_t\}_{p=0.7}$ for different choices of substitution functions (Best look ahead(\textcolor{blue}{---}), Worst look ahead(\textcolor{black}{---}), Inverse loss(\textcolor{green}{---}), and Weighted average(\textcolor{magenta}{---})) with learning rate $\eta$ and expert setting $\{E_t\}_i$ (theoretical regret bound is shown by \textcolor{red}{- - -}).}}
  {
    \subfigure[$\eta = 0.1, \{E_t\}_{\text{set.}1}$]{\label{fig:p0d7_n0d1_e1}
      \includegraphics[width=0.32\linewidth]{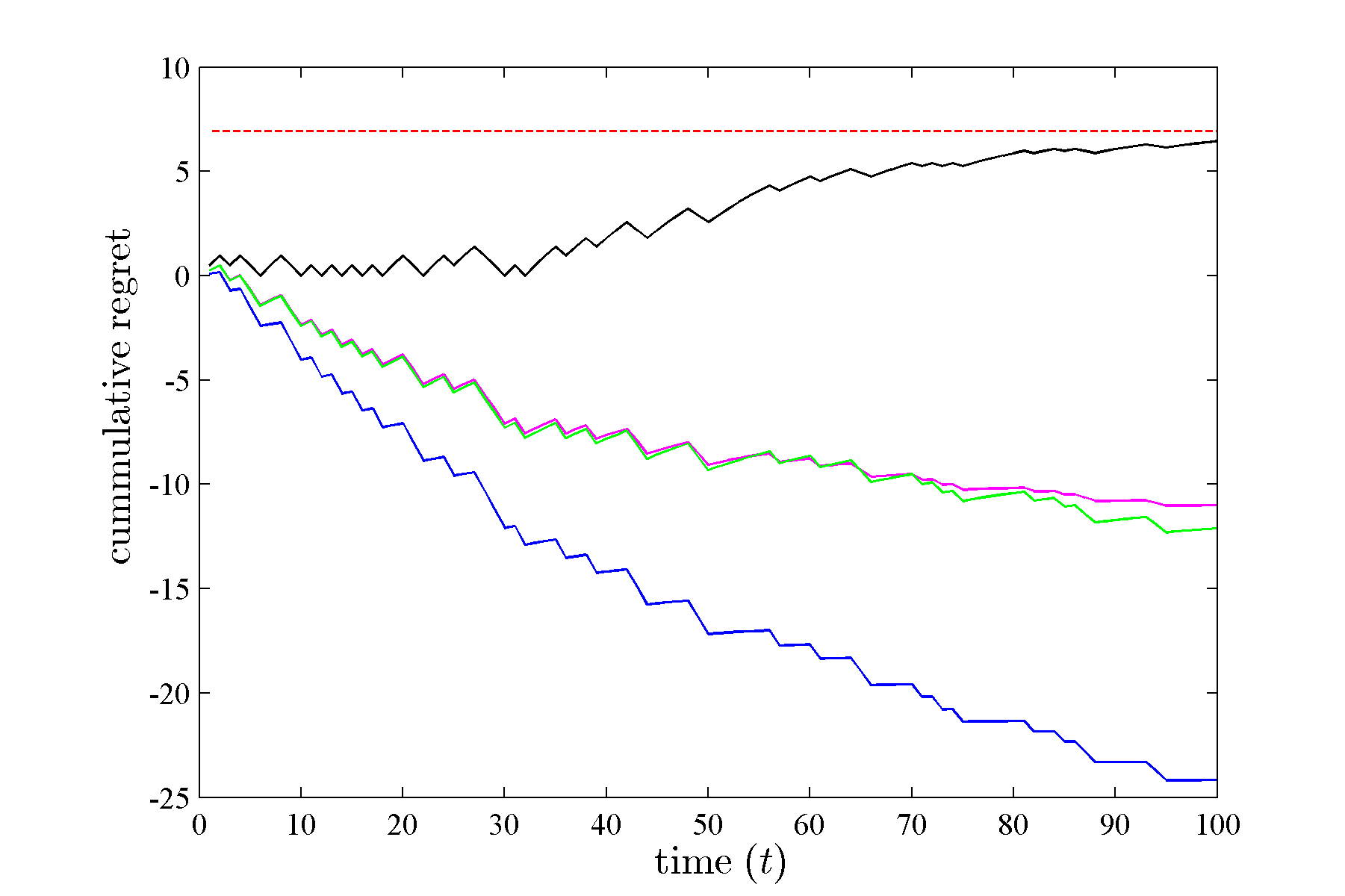}}
    \subfigure[$\eta = 0.3, \{E_t\}_{\text{set.}1}$]{\label{fig:p0d7_n0d3_e1}
      \includegraphics[width=0.32\linewidth]{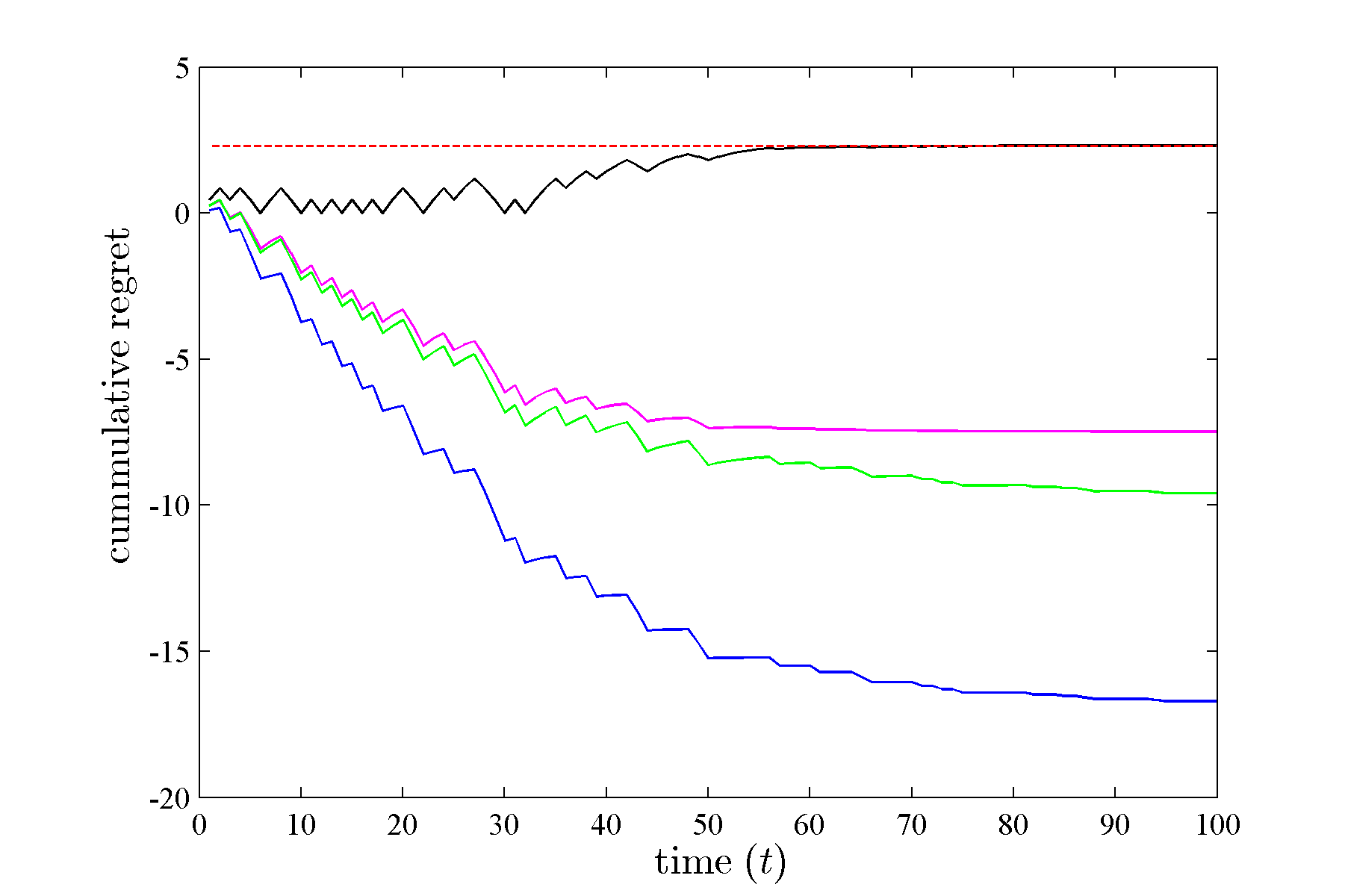}}
    \subfigure[$\eta = 0.5, \{E_t\}_{\text{set.}1}$]{\label{fig:p0d7_n0d5_e1}
      \includegraphics[width=0.32\linewidth]{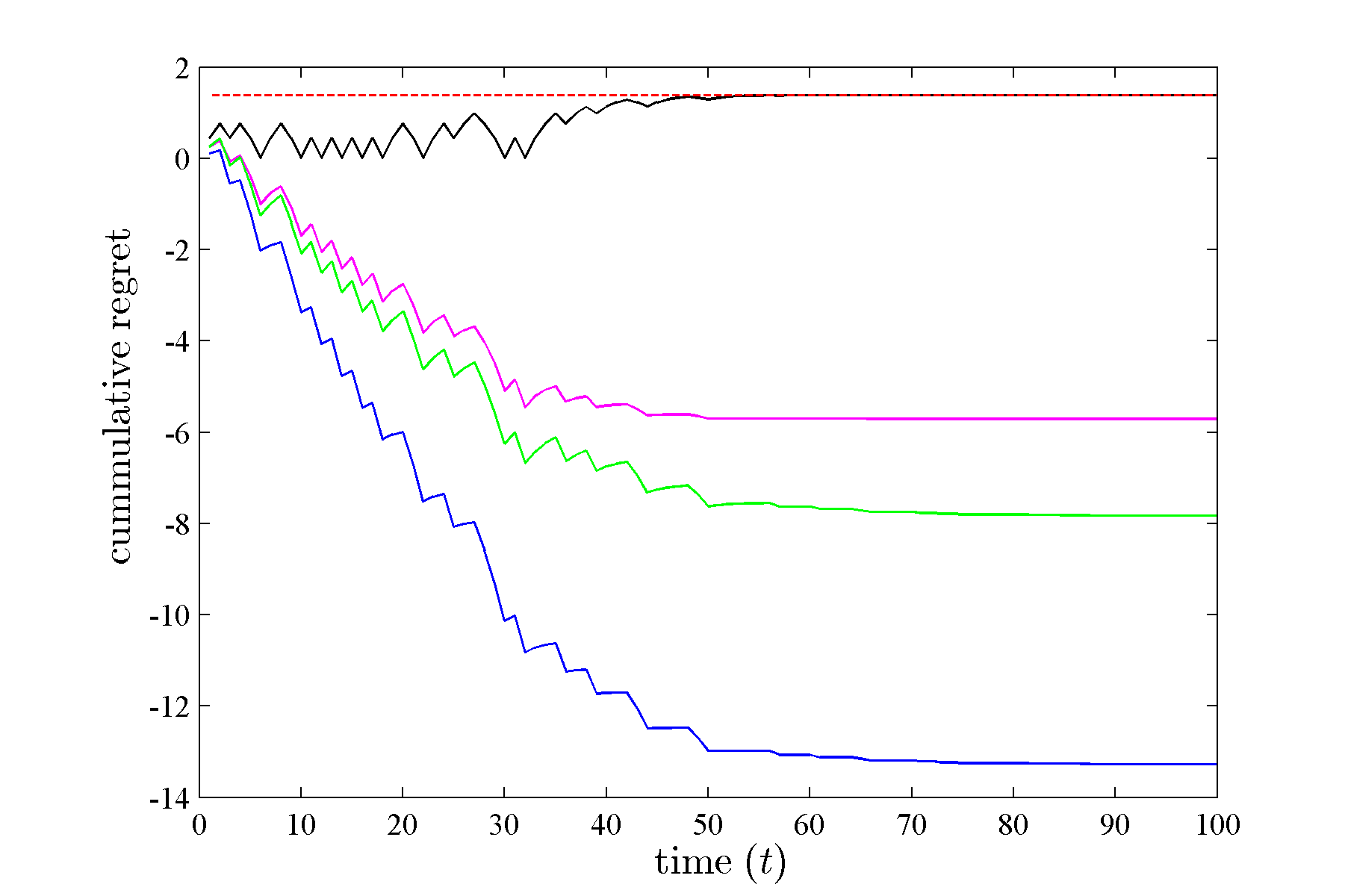}}

    \subfigure[$\eta = 0.1, \{E_t\}_{\text{set.}2}$]{\label{fig:p0d7_n0d1_e2}
      \includegraphics[width=0.32\linewidth]{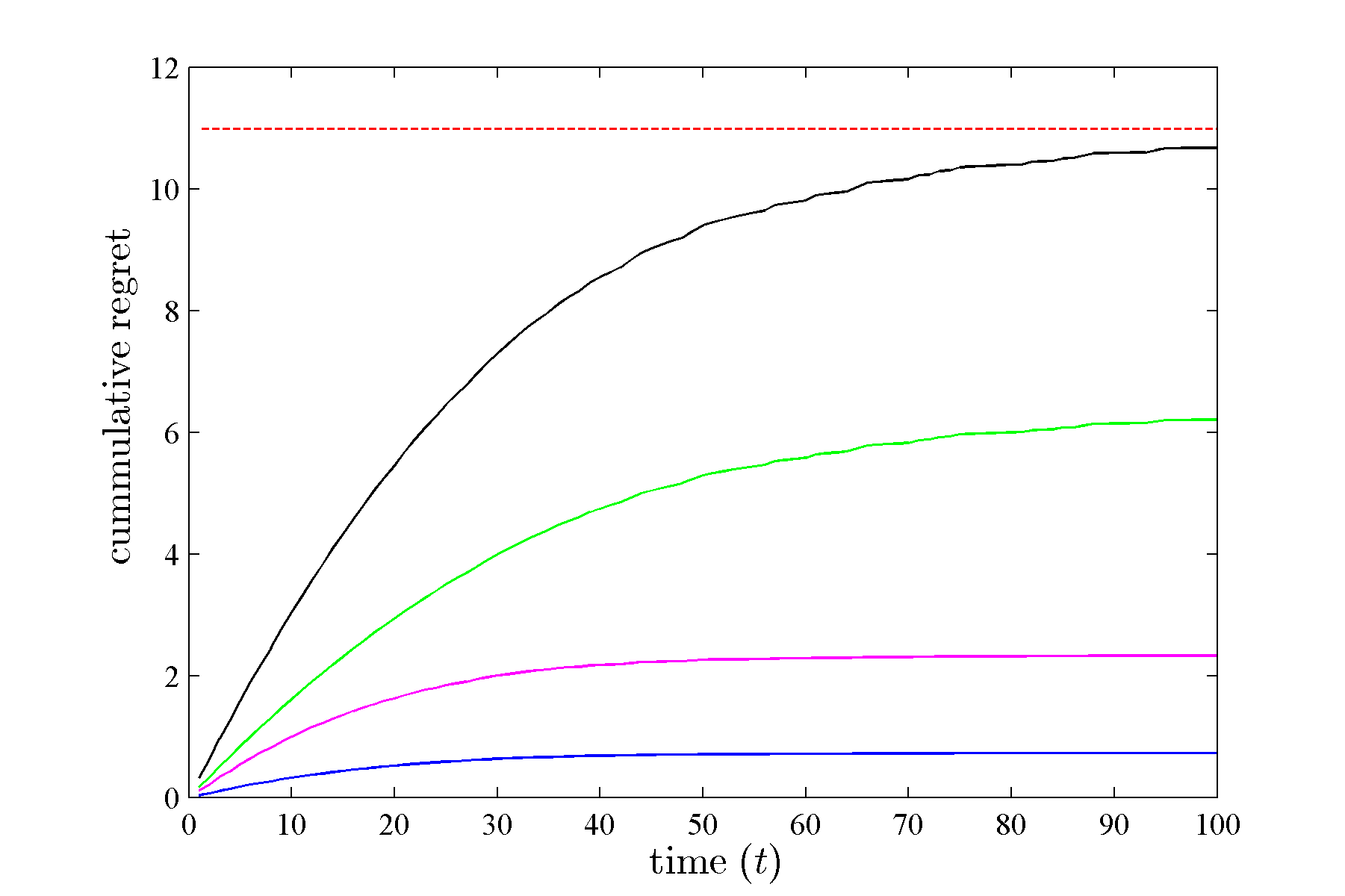}}
    \subfigure[$\eta = 0.3, \{E_t\}_{\text{set.}2}$]{\label{fig:p0d7_n0d3_e2}
      \includegraphics[width=0.32\linewidth]{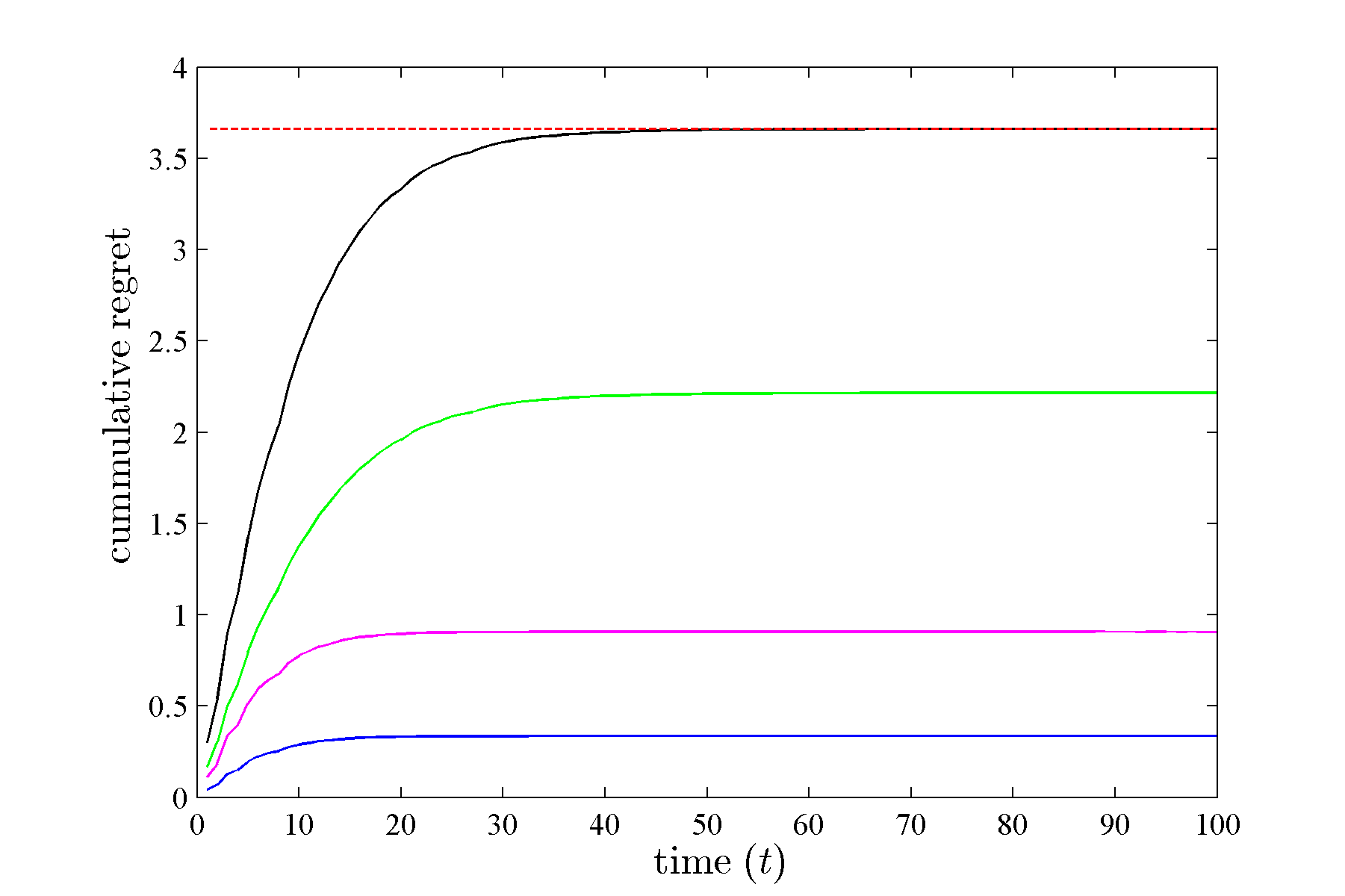}}
    \subfigure[$\eta = 0.5, \{E_t\}_{\text{set.}2}$]{\label{fig:p0d7_n0d5_e2}
      \includegraphics[width=0.32\linewidth]{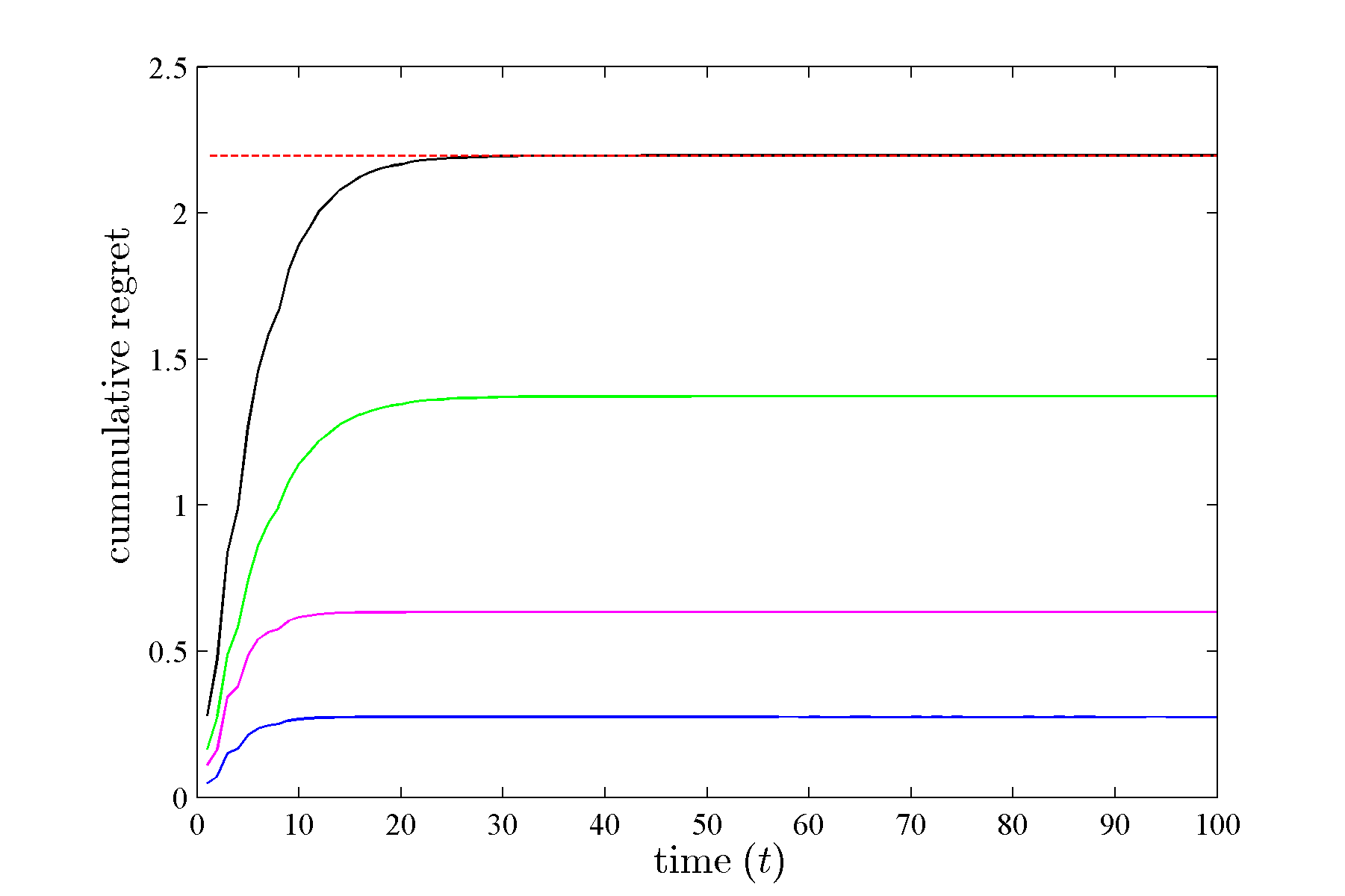}}

    \subfigure[$\eta = 0.1, \{E_t\}_{\text{set.}3}$]{\label{fig:p0d7_n0d1_e3}
      \includegraphics[width=0.32\linewidth]{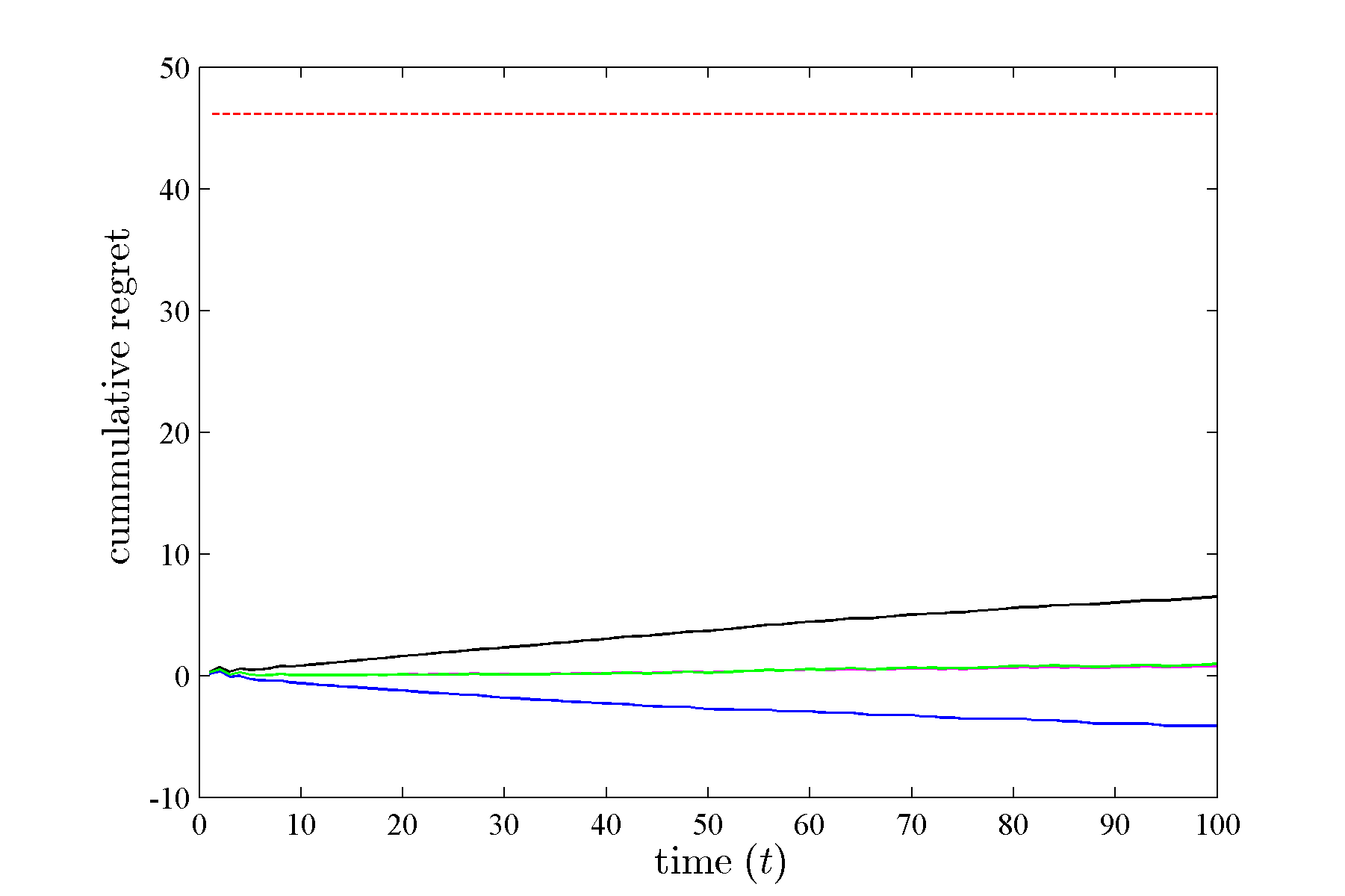}}
    \subfigure[$\eta = 0.3, \{E_t\}_{\text{set.}3}$]{\label{fig:p0d7_n0d3_e3}
      \includegraphics[width=0.32\linewidth]{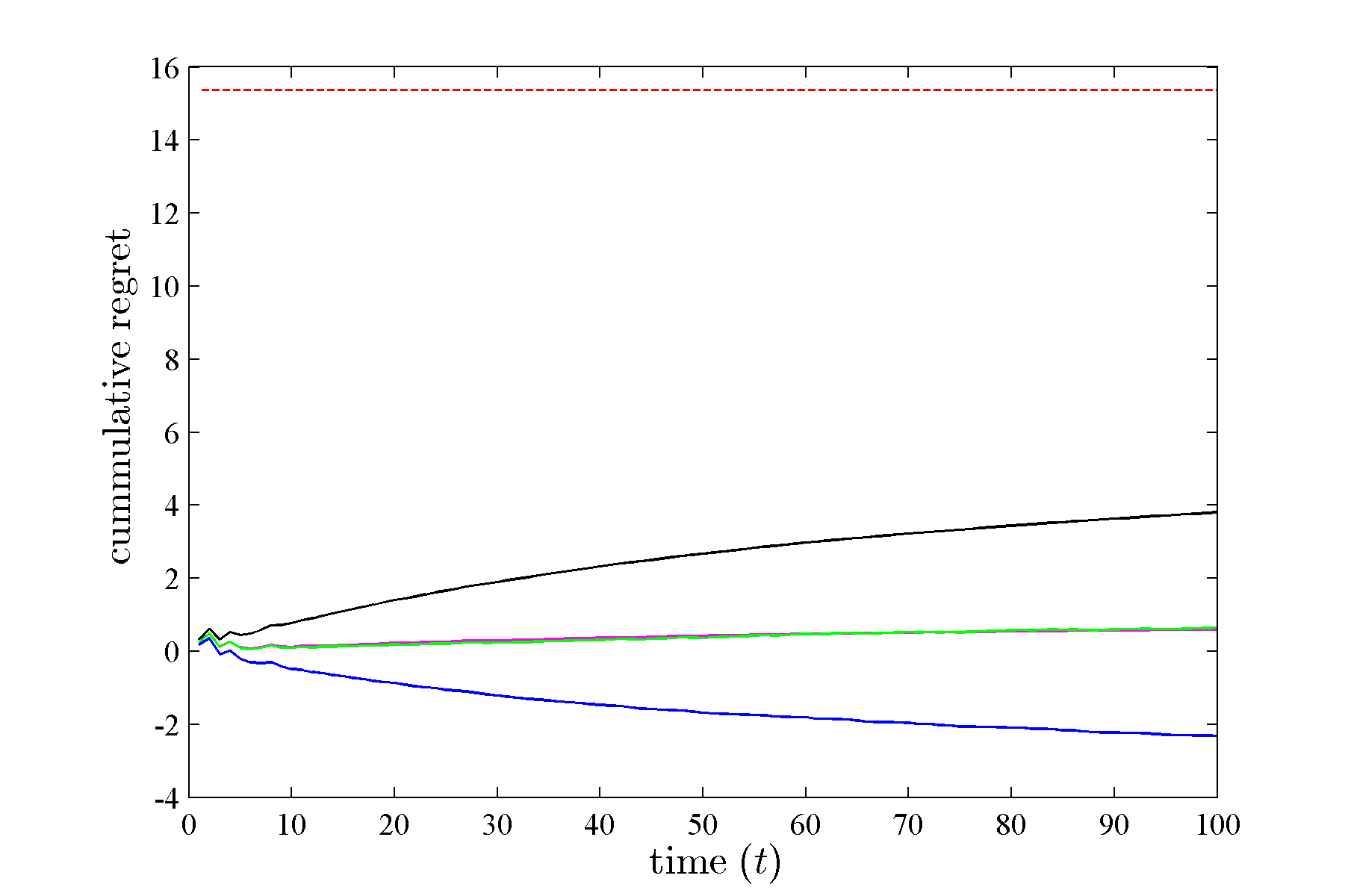}}
    \subfigure[$\eta = 0.5, \{E_t\}_{\text{set.}3}$]{\label{fig:p0d7_n0d5_e3}
      \includegraphics[width=0.32\linewidth]{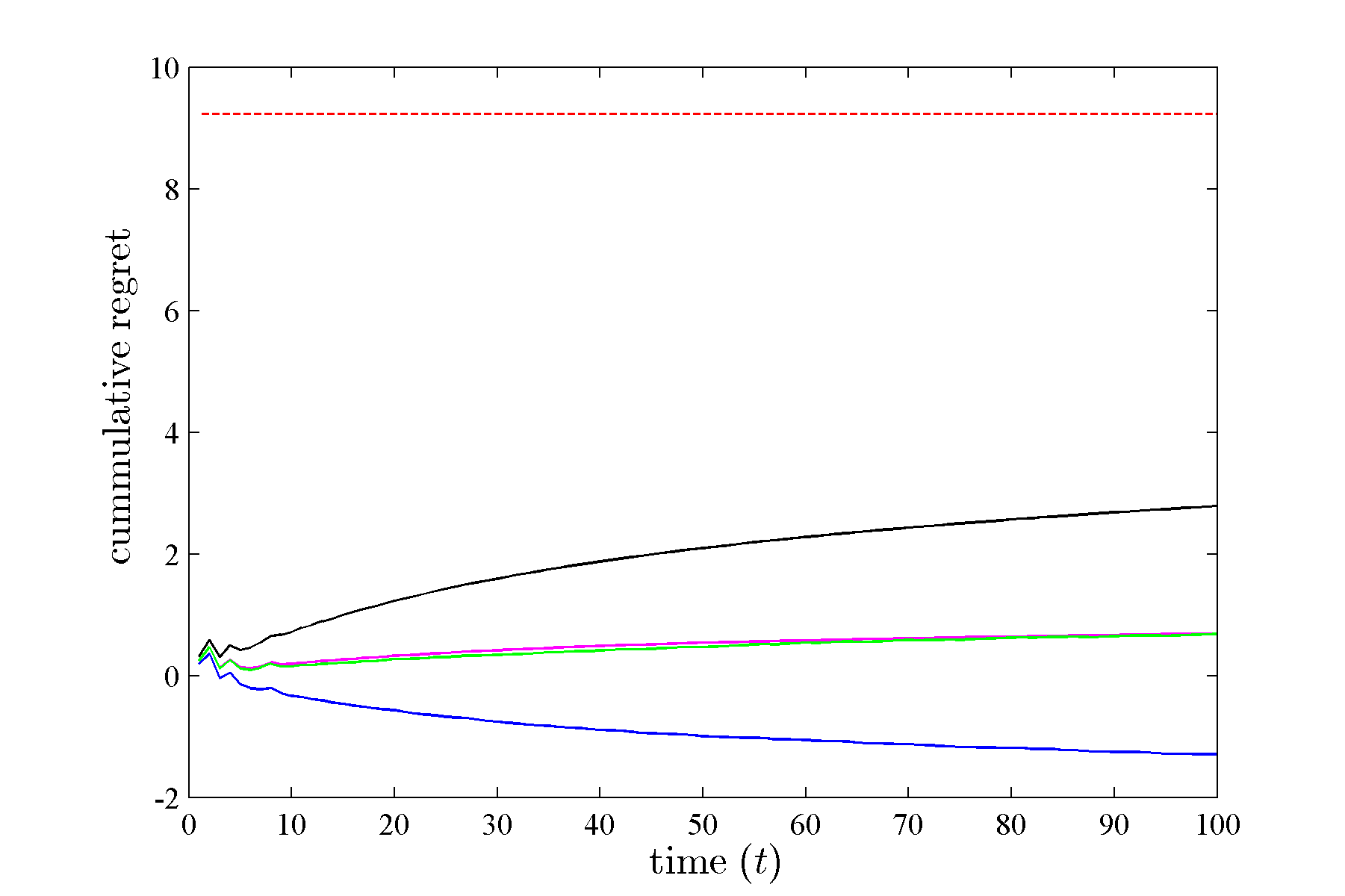}}
  }
\end{figure}

\begin{figure}[htbp]
\floatconts
  {fig:p0d9}
  {\caption{Cumulative regret of the Aggregating Algorithm over the outcome sequence $\{y_t\}_{p=0.9}$ for different choices of substitution functions (Best look ahead(\textcolor{blue}{---}), Worst look ahead(\textcolor{black}{---}), Inverse loss(\textcolor{green}{---}), and Weighted average(\textcolor{magenta}{---})) with learning rate $\eta$ and expert setting $\{E_t\}_i$ (theoretical regret bound is shown by \textcolor{red}{- - -}).}}
  {
    \subfigure[$\eta = 0.1, \{E_t\}_{\text{set.}1}$]{\label{fig:p0d9_n0d1_e1}
      \includegraphics[width=0.32\linewidth]{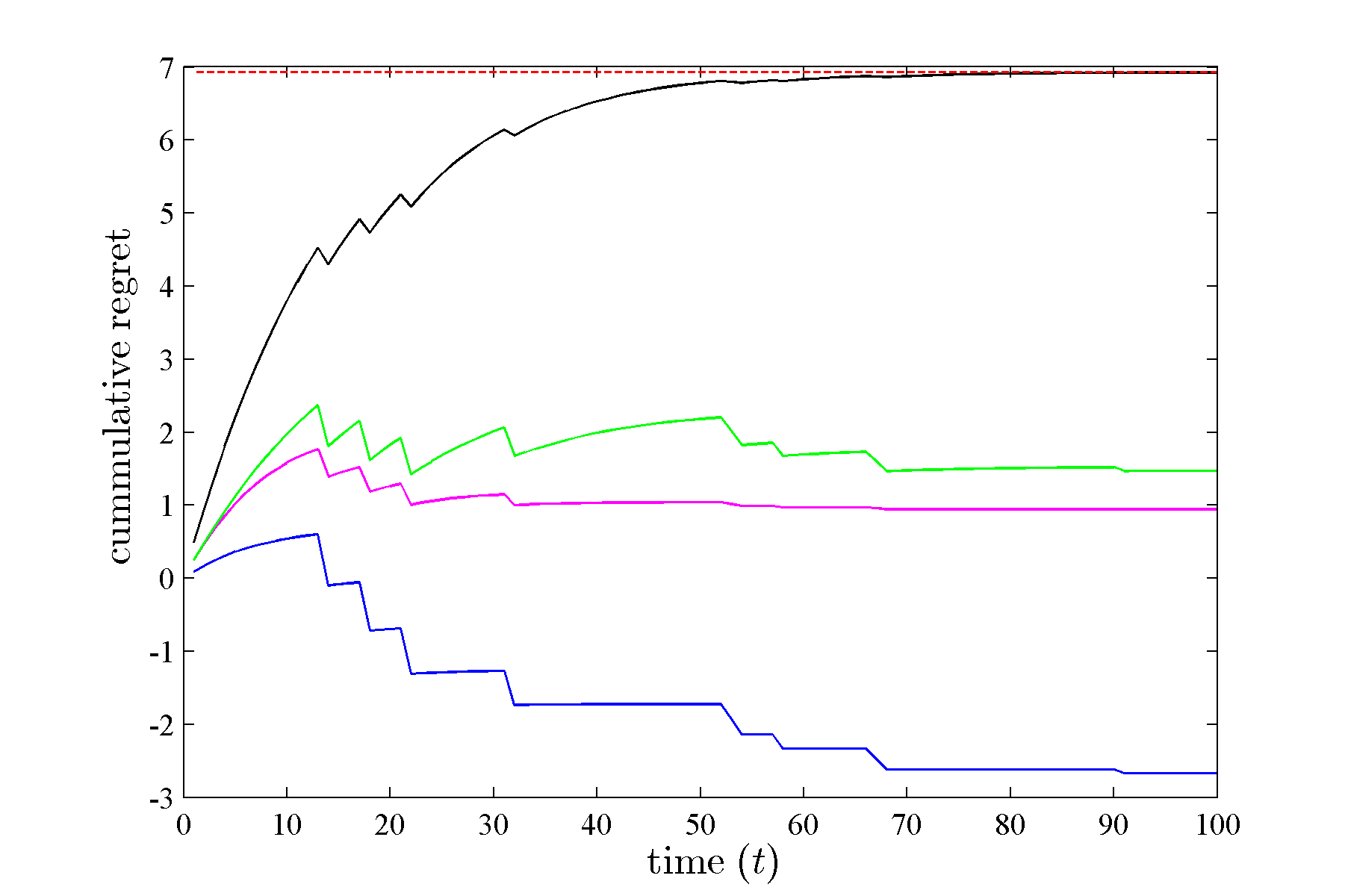}}
    \subfigure[$\eta = 0.3, \{E_t\}_{\text{set.}1}$]{\label{fig:p0d9_n0d3_e1}
      \includegraphics[width=0.32\linewidth]{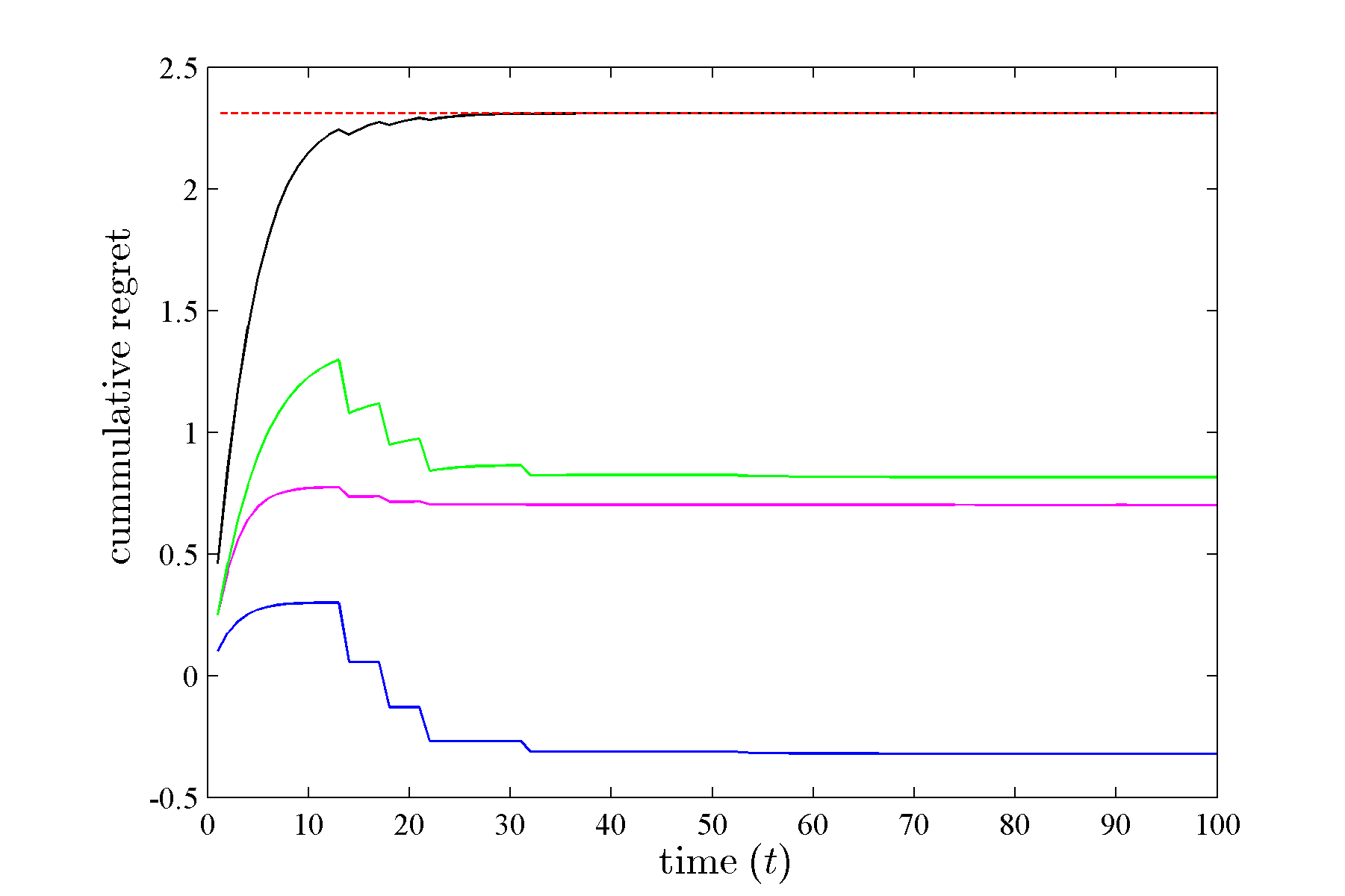}}
    \subfigure[$\eta = 0.5, \{E_t\}_{\text{set.}1}$]{\label{fig:p0d9_n0d5_e1}
      \includegraphics[width=0.32\linewidth]{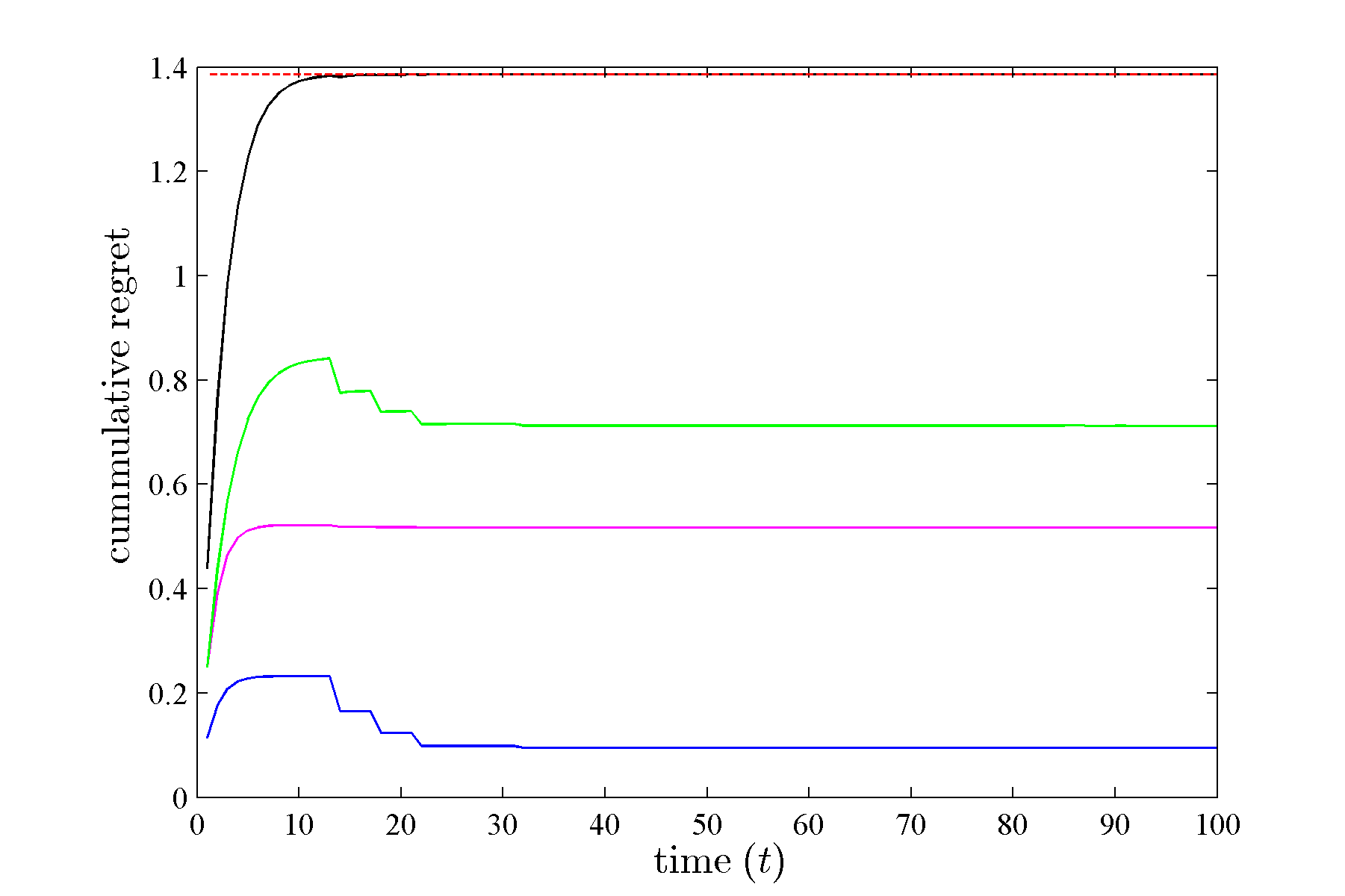}}

    \subfigure[$\eta = 0.1, \{E_t\}_{\text{set.}2}$]{\label{fig:p0d9_n0d1_e2}
      \includegraphics[width=0.32\linewidth]{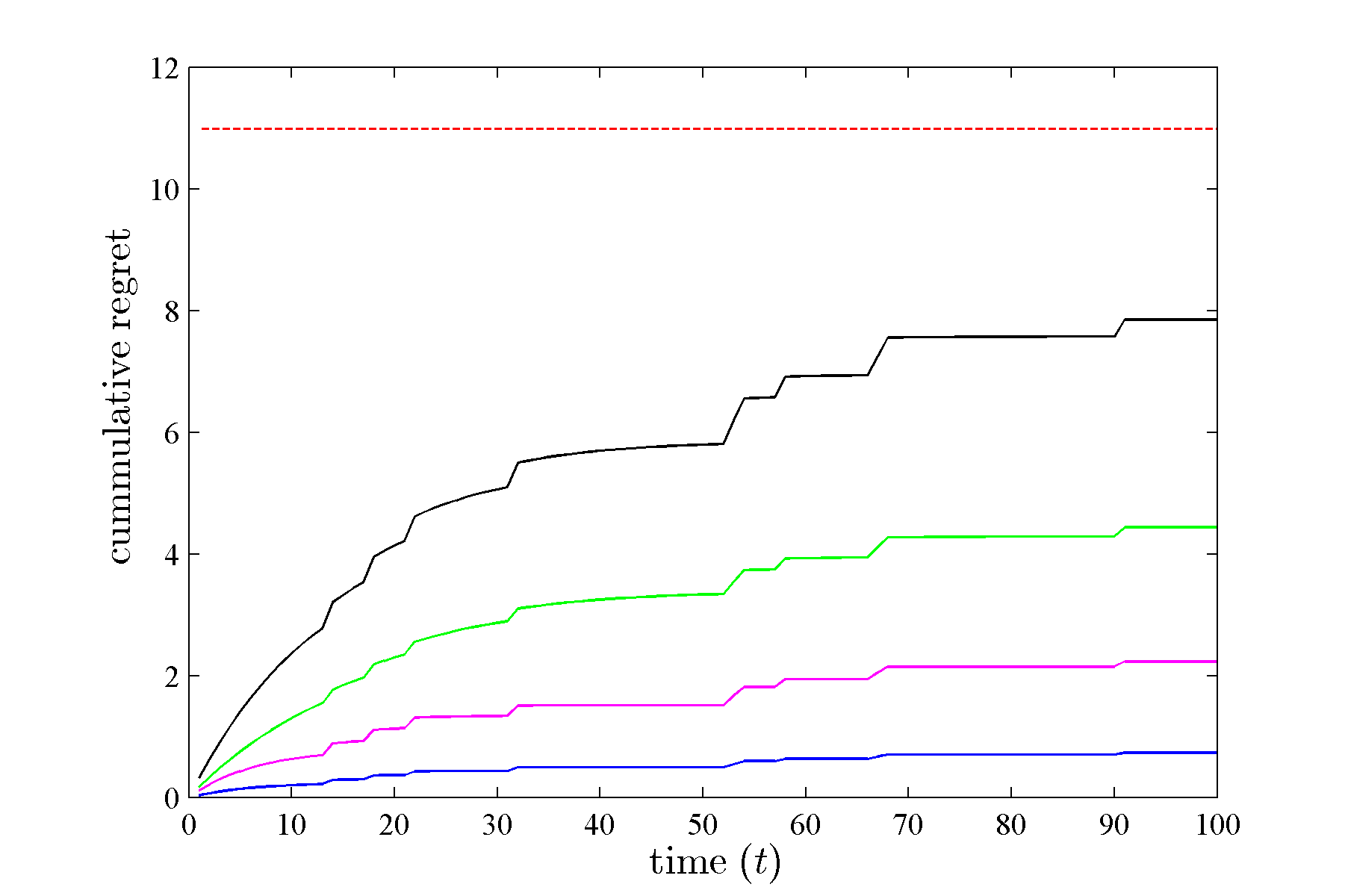}}
    \subfigure[$\eta = 0.3, \{E_t\}_{\text{set.}2}$]{\label{fig:p0d9_n0d3_e2}
      \includegraphics[width=0.32\linewidth]{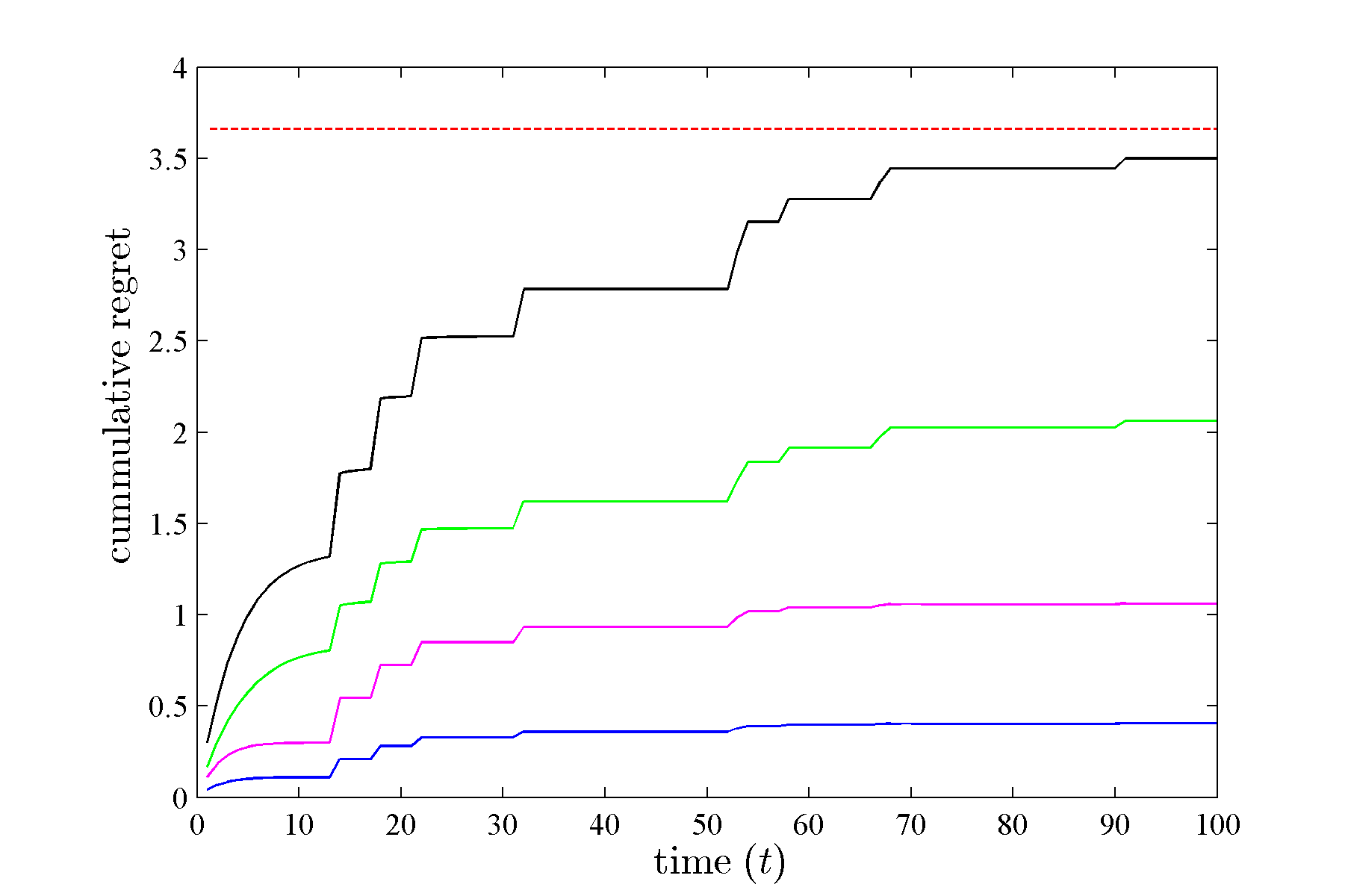}}
    \subfigure[$\eta = 0.5, \{E_t\}_{\text{set.}2}$]{\label{fig:p0d9_n0d5_e2}
      \includegraphics[width=0.32\linewidth]{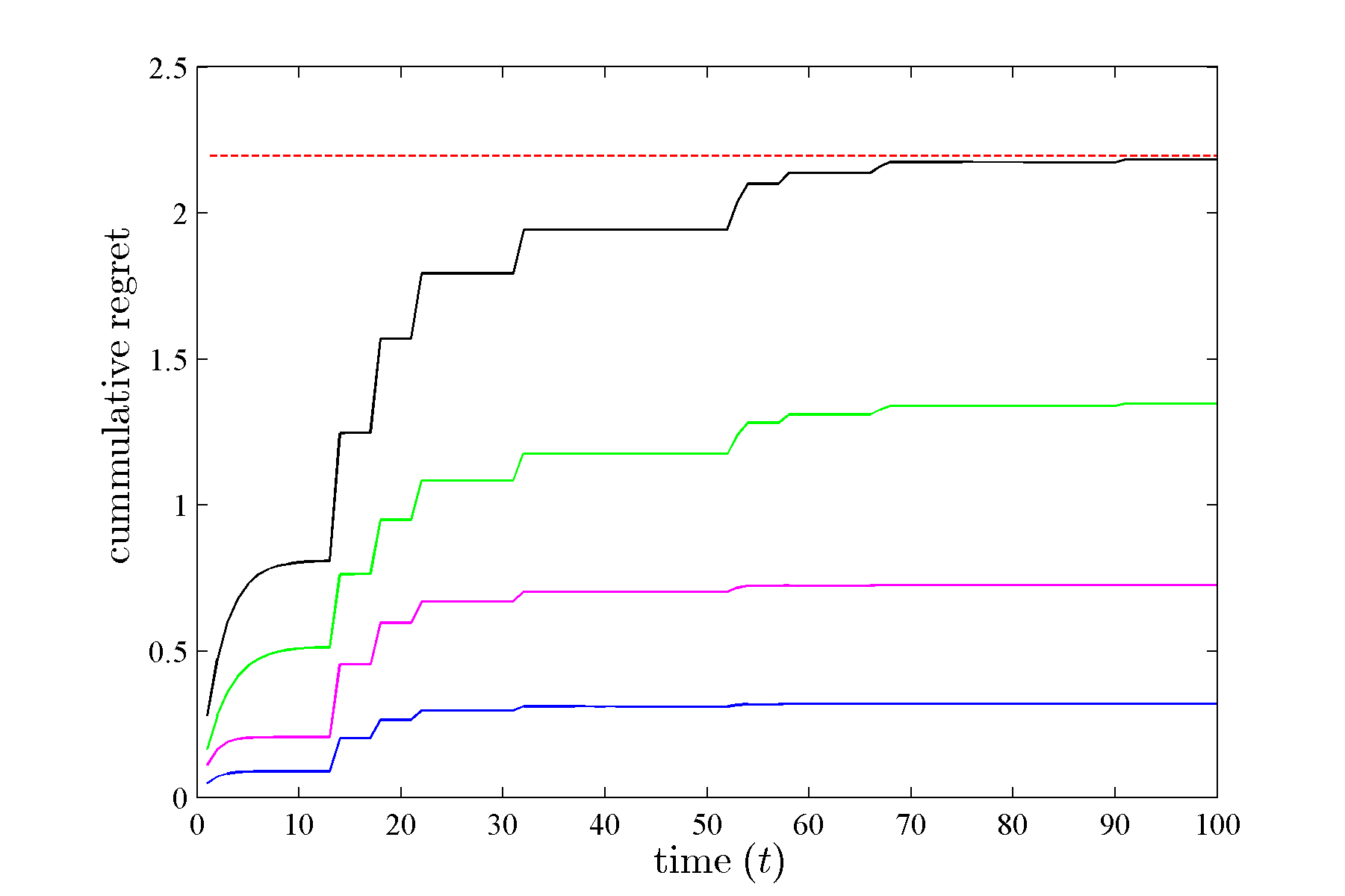}}

    \subfigure[$\eta = 0.1, \{E_t\}_{\text{set.}3}$]{\label{fig:p0d9_n0d1_e3}
      \includegraphics[width=0.32\linewidth]{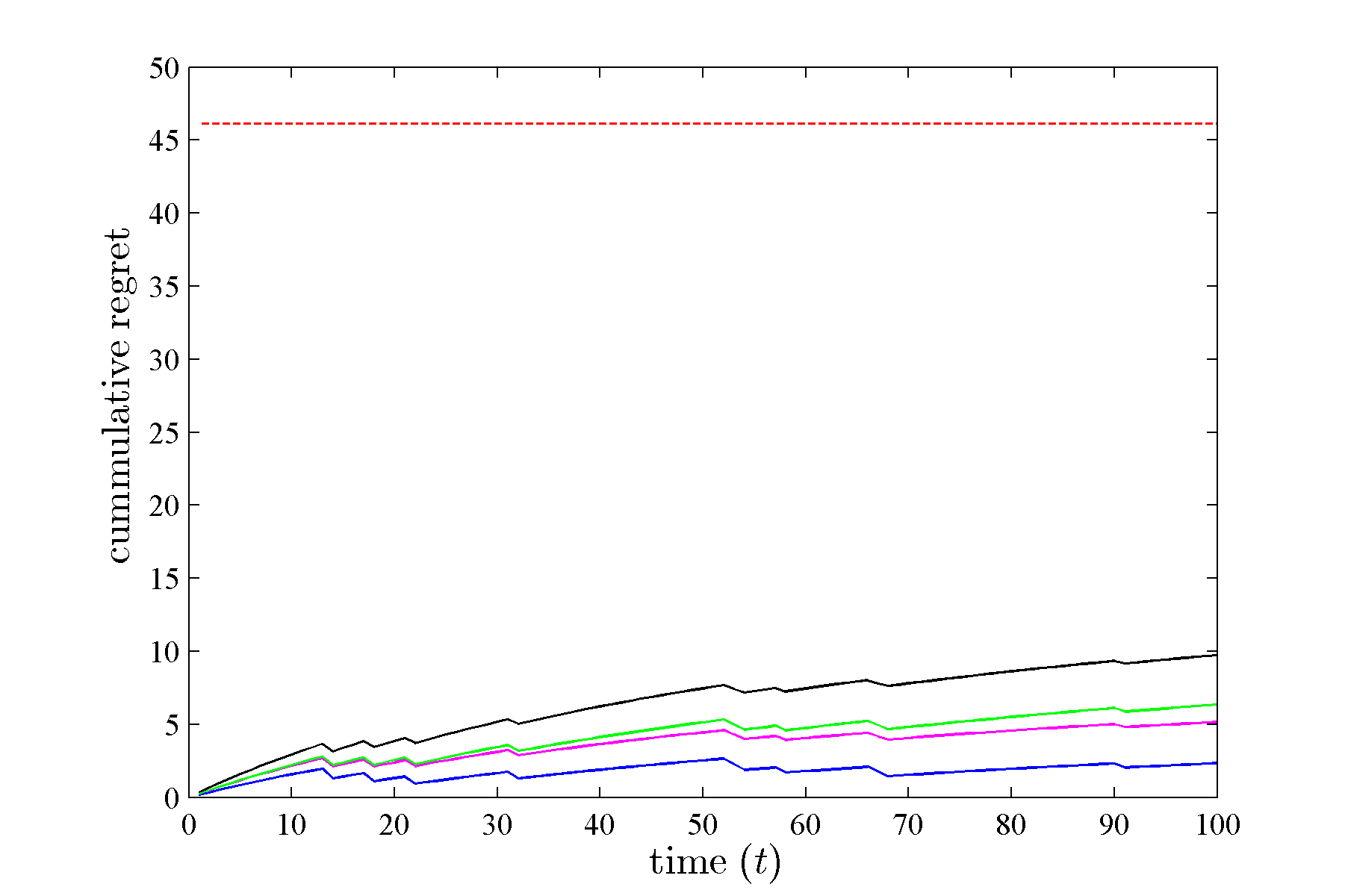}}
    \subfigure[$\eta = 0.3, \{E_t\}_{\text{set.}3}$]{\label{fig:p0d9_n0d3_e3}
      \includegraphics[width=0.32\linewidth]{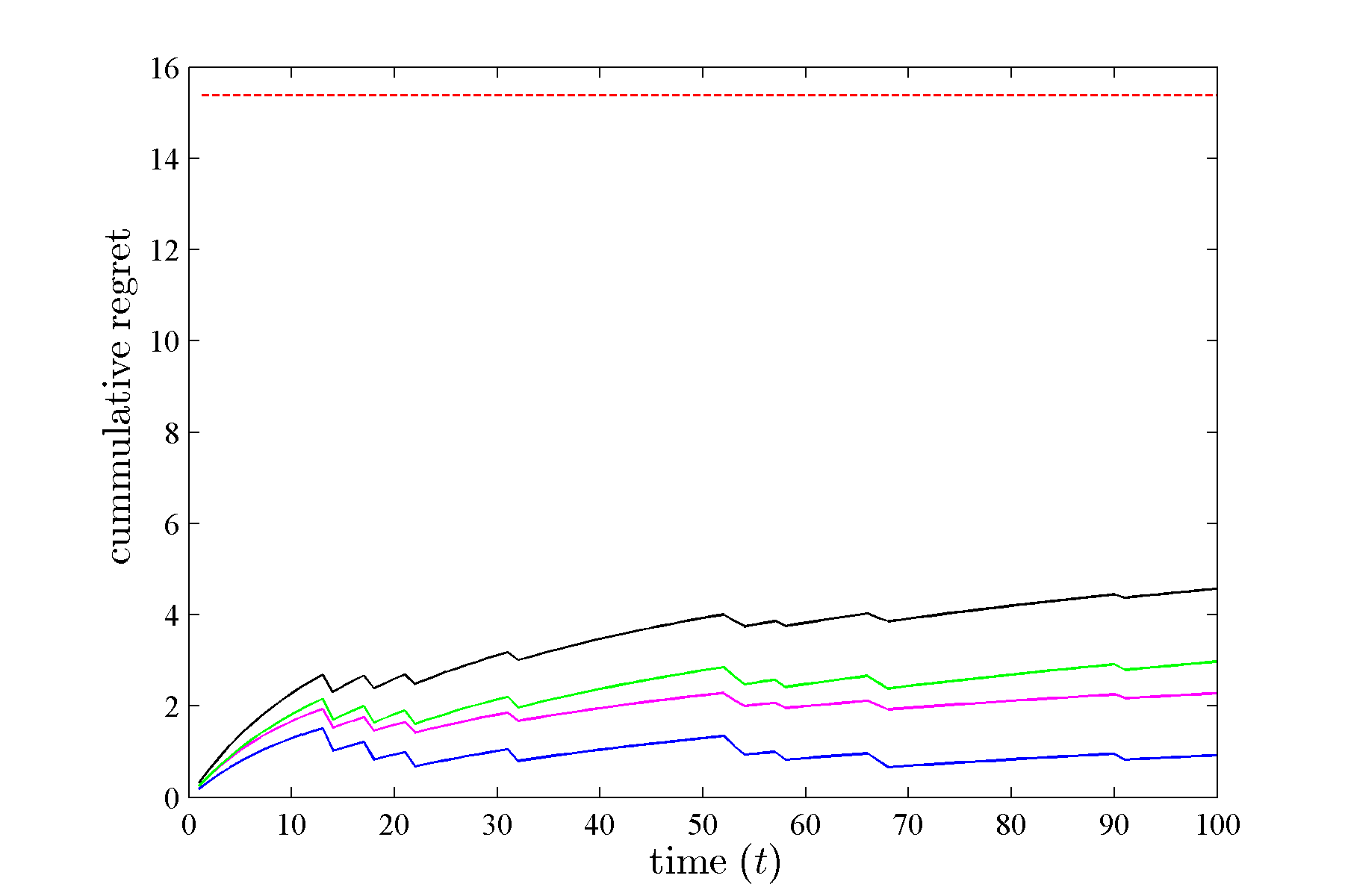}}
    \subfigure[$\eta = 0.5, \{E_t\}_{\text{set.}3}$]{\label{fig:p0d9_n0d5_e3}
      \includegraphics[width=0.32\linewidth]{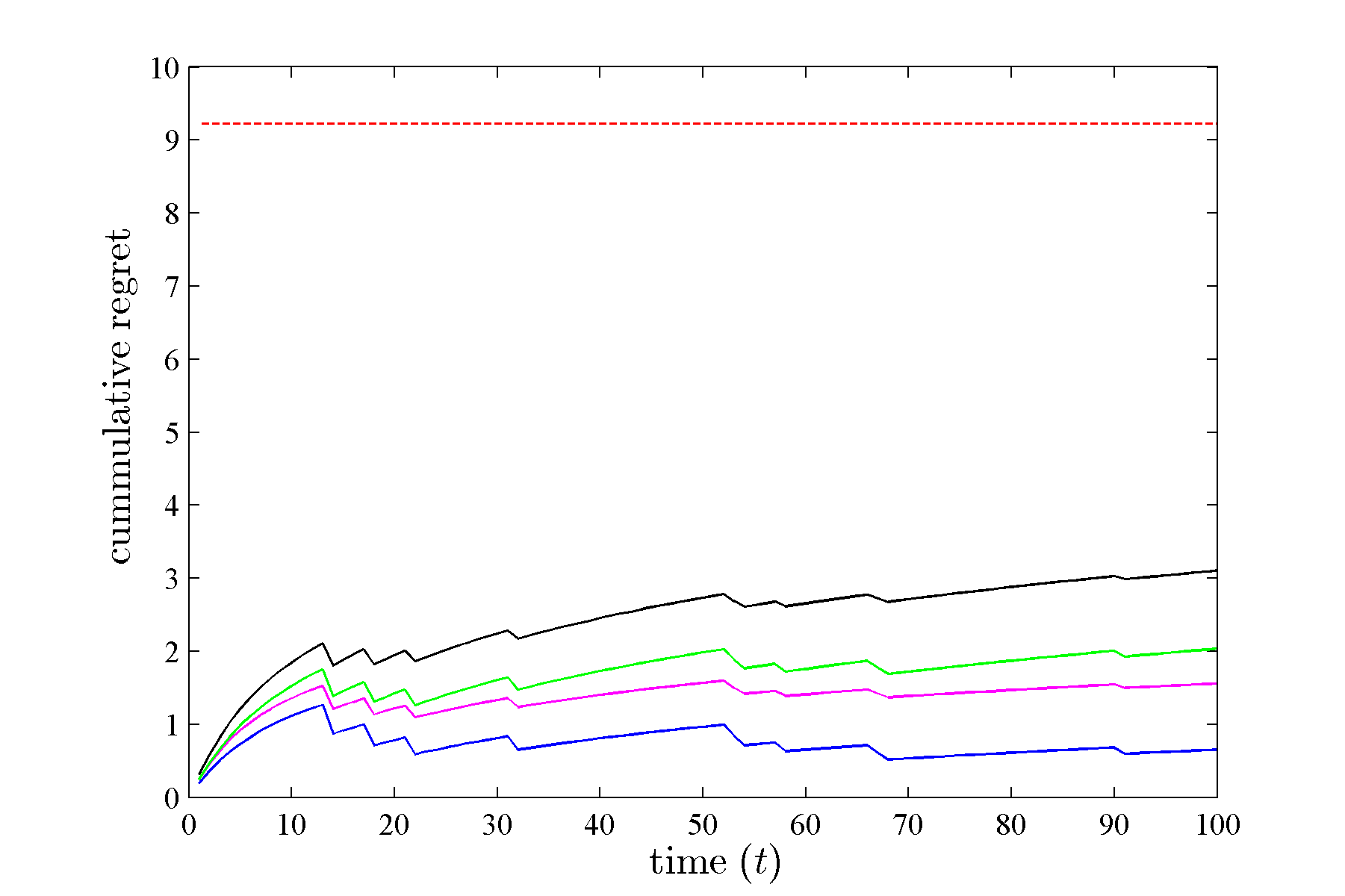}}
  }
\end{figure}

\begin{figure}[htbp]
\floatconts
  {fig:p1d0}
  {\caption{Cumulative regret of the Aggregating Algorithm over the outcome sequence $\{y_t\}_{p=1.0}$ for different choices of substitution functions (Best look ahead(\textcolor{blue}{---}), Worst look ahead(\textcolor{black}{---}), Inverse loss(\textcolor{green}{---}), and Weighted average(\textcolor{magenta}{---})) with learning rate $\eta$ and expert setting $\{E_t\}_i$ (theoretical regret bound is shown by \textcolor{red}{- - -}).}}
  {
    \subfigure[$\eta = 0.1, \{E_t\}_{\text{set.}1}$]{\label{fig:p0d7_n1d0_e1}
      \includegraphics[width=0.32\linewidth]{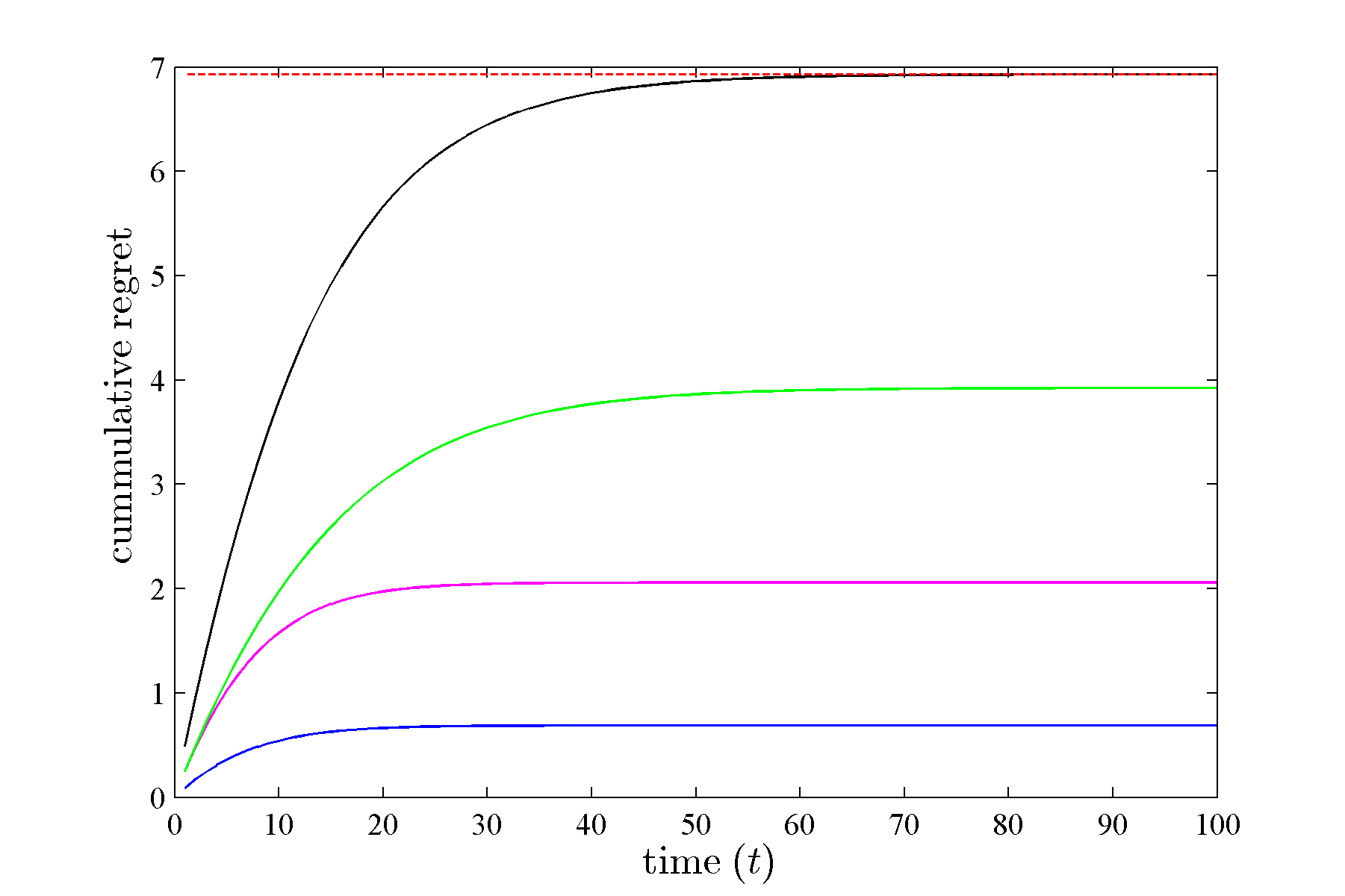}}
    \subfigure[$\eta = 0.3, \{E_t\}_{\text{set.}1}$]{\label{fig:p0d7_n1d0_e1}
      \includegraphics[width=0.32\linewidth]{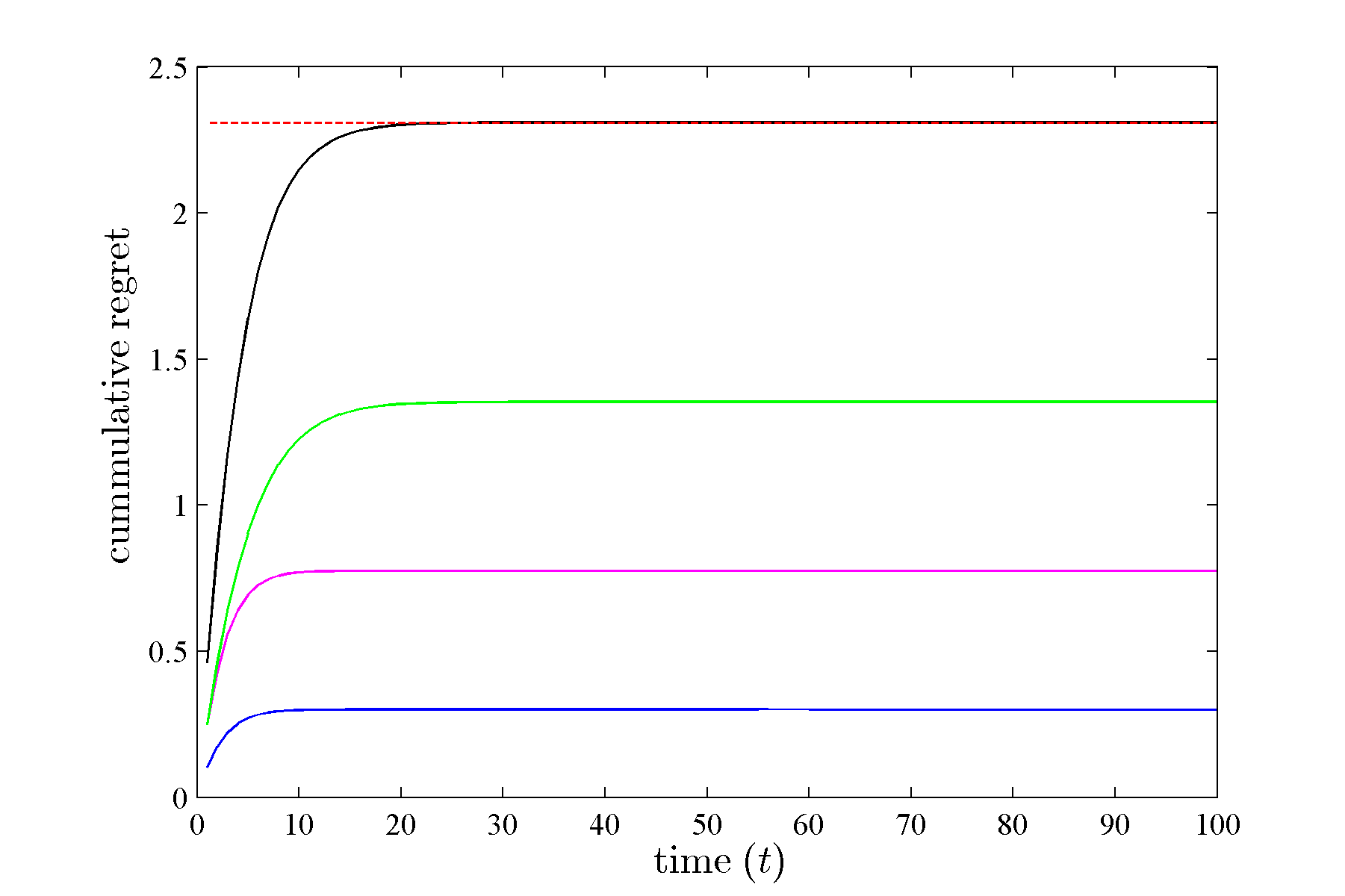}}
    \subfigure[$\eta = 0.5, \{E_t\}_{\text{set.}1}$]{\label{fig:p0d7_n1d0_e1}
      \includegraphics[width=0.32\linewidth]{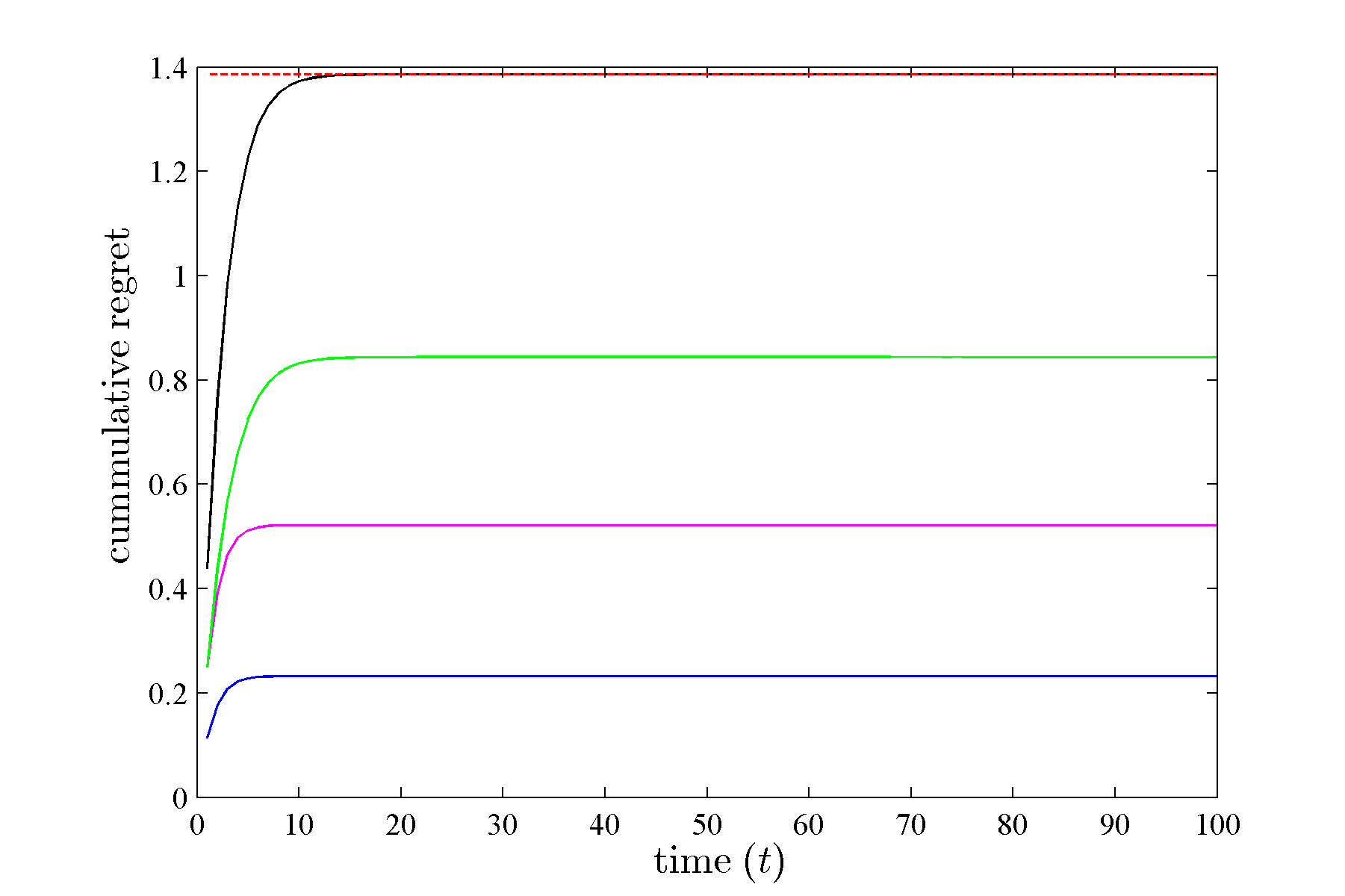}}

    \subfigure[$\eta = 0.1, \{E_t\}_{\text{set.}2}$]{\label{fig:p0d7_n1d0_e2}
      \includegraphics[width=0.32\linewidth]{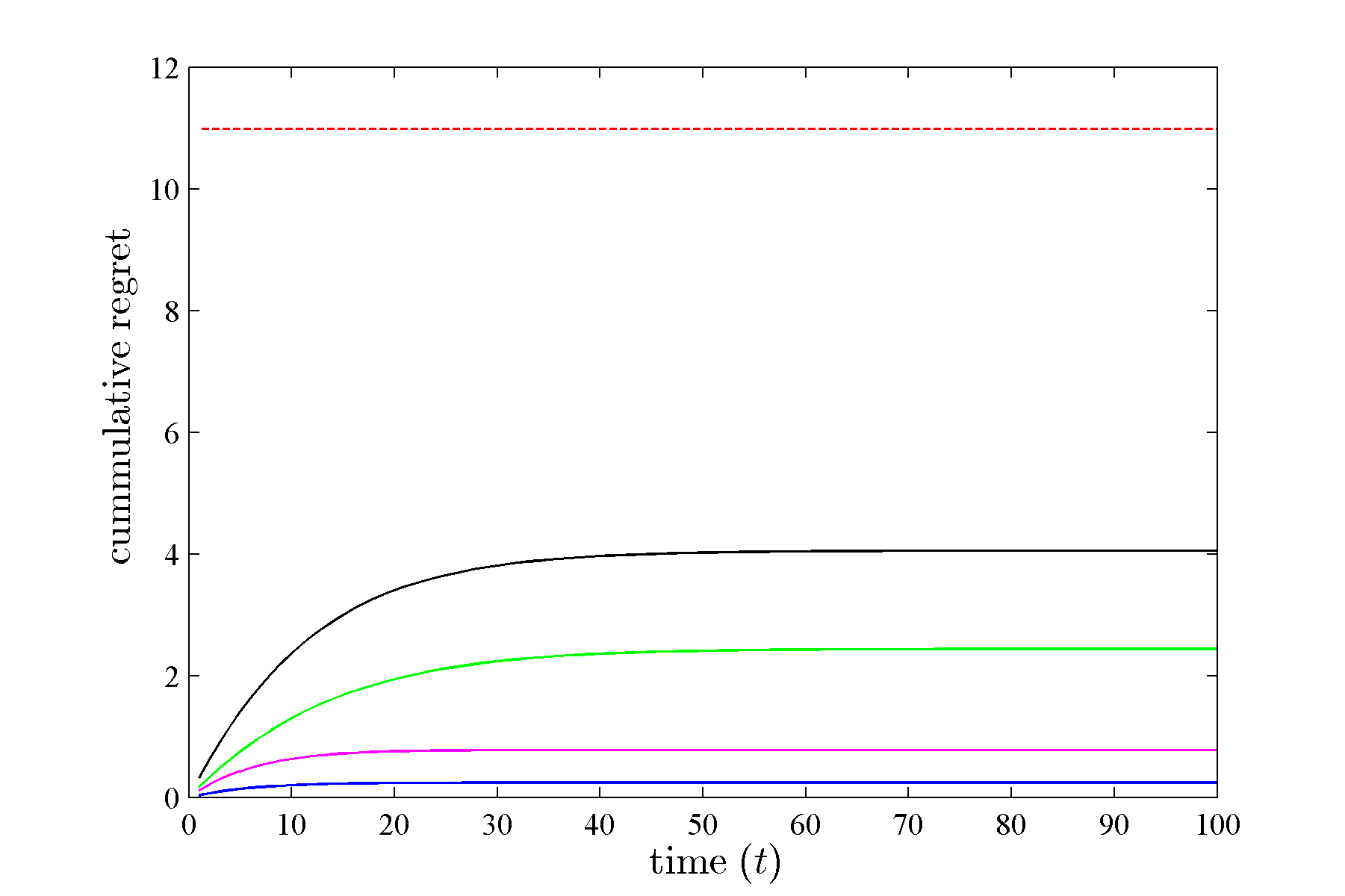}}
    \subfigure[$\eta = 0.3, \{E_t\}_{\text{set.}2}$]{\label{fig:p0d7_n1d0_e2}
      \includegraphics[width=0.32\linewidth]{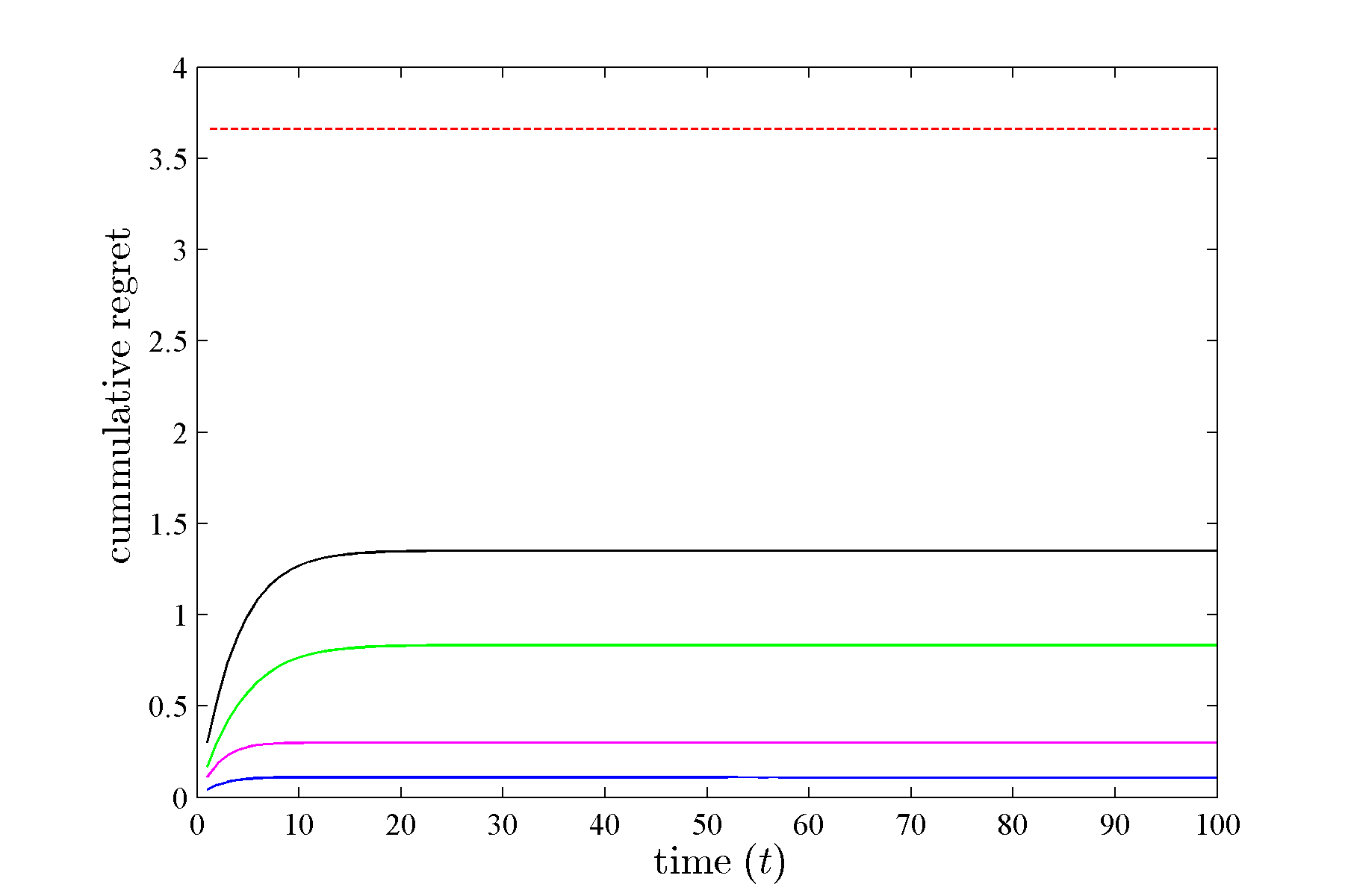}}
    \subfigure[$\eta = 0.5, \{E_t\}_{\text{set.}2}$]{\label{fig:p0d7_n1d0_e2}
      \includegraphics[width=0.32\linewidth]{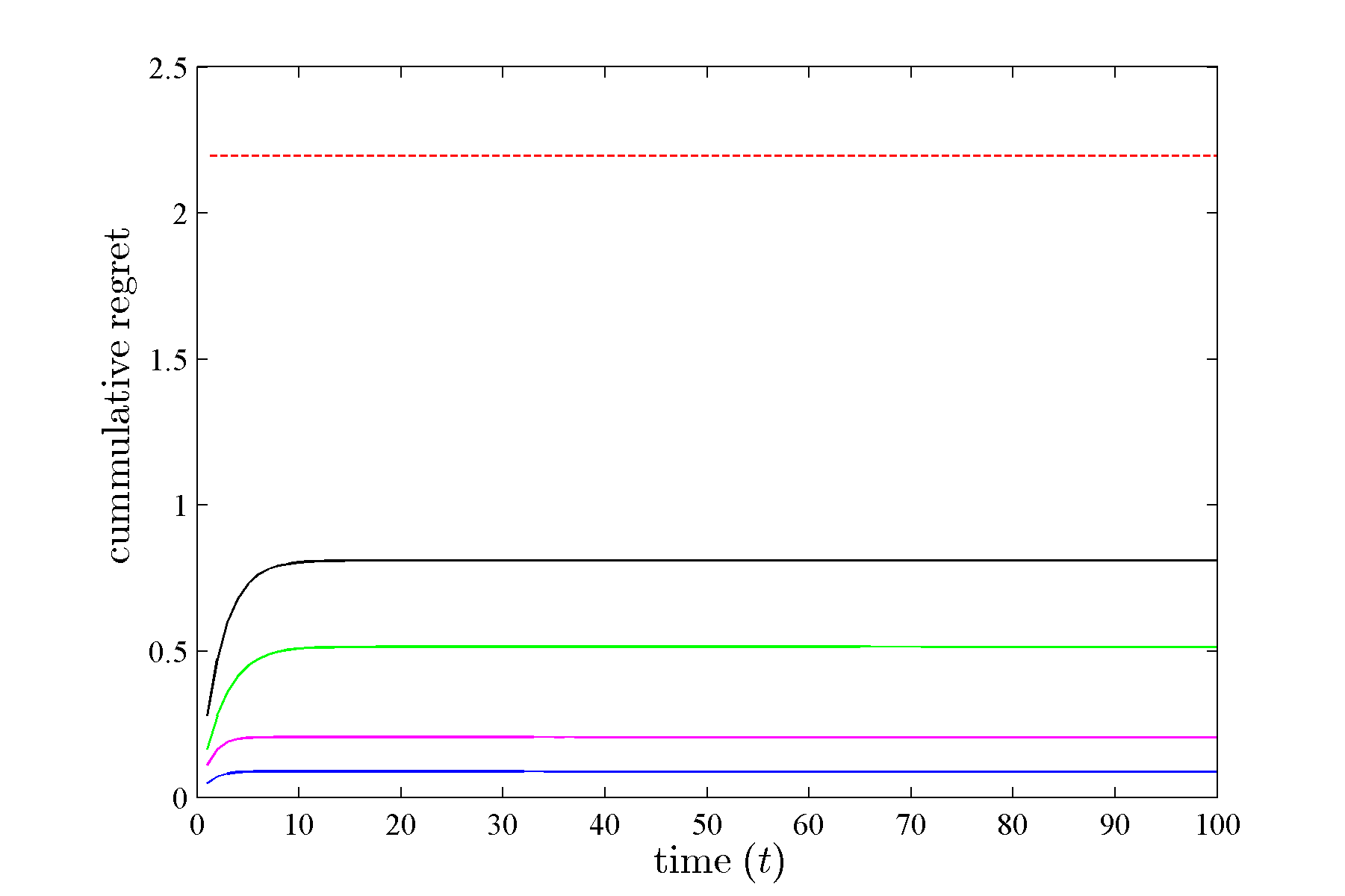}}

    \subfigure[$\eta = 0.1, \{E_t\}_{\text{set.}3}$]{\label{fig:p0d7_n1d0_e3}
      \includegraphics[width=0.32\linewidth]{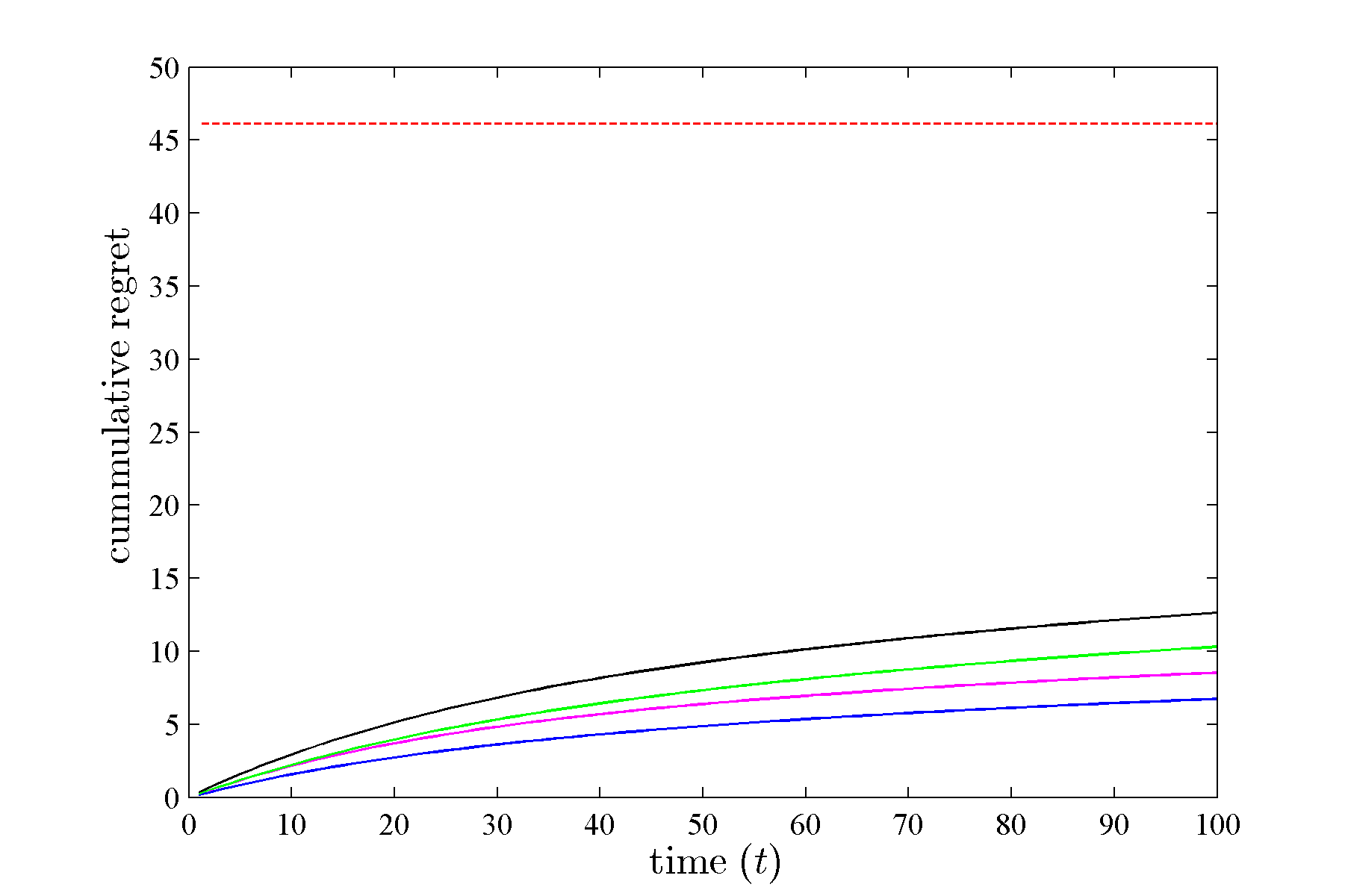}}
    \subfigure[$\eta = 0.3, \{E_t\}_{\text{set.}3}$]{\label{fig:p0d7_n1d0_e3}
      \includegraphics[width=0.32\linewidth]{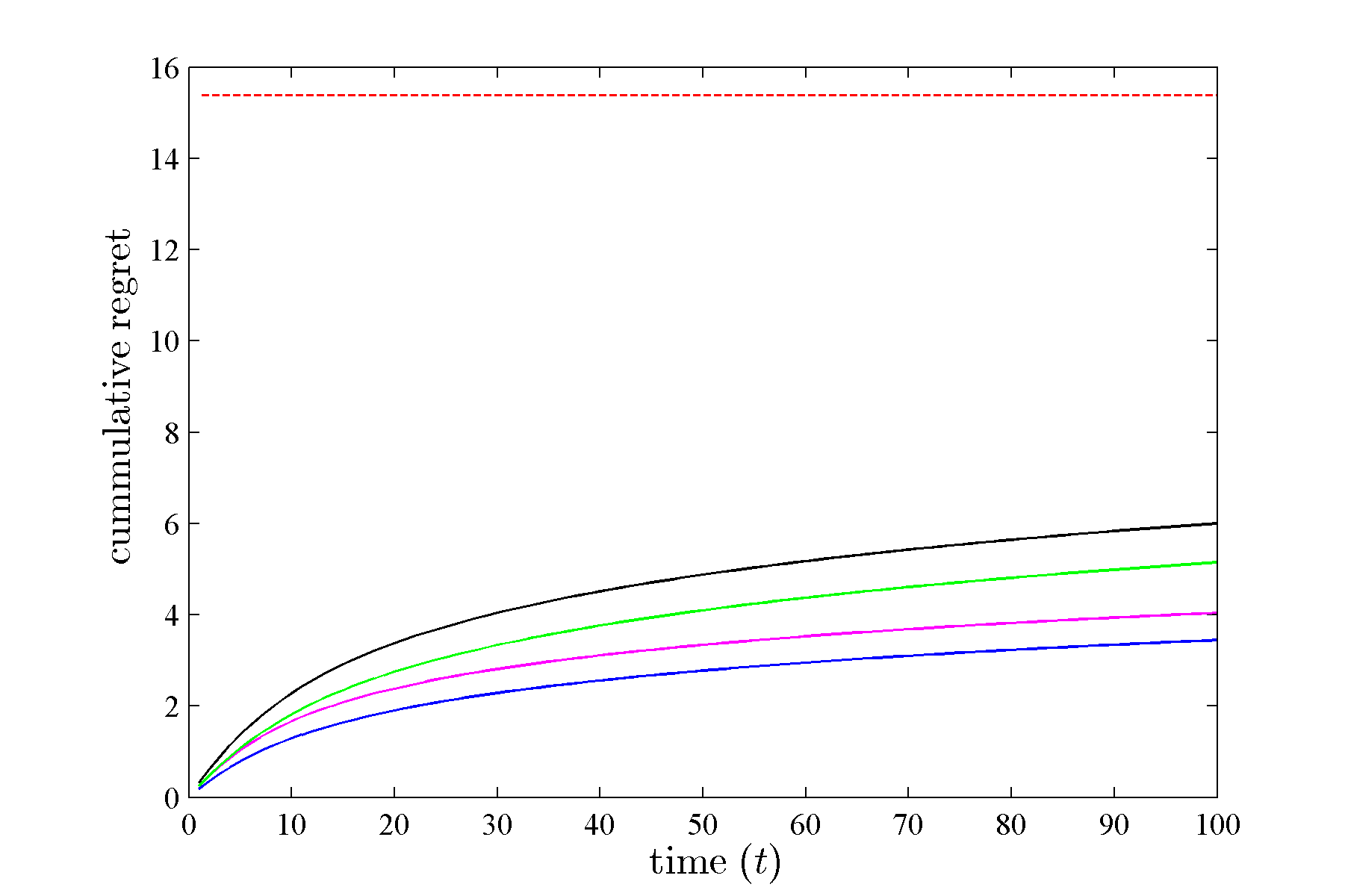}}
    \subfigure[$\eta = 0.5, \{E_t\}_{\text{set.}3}$]{\label{fig:p0d7_n1d0_e3}
      \includegraphics[width=0.32\linewidth]{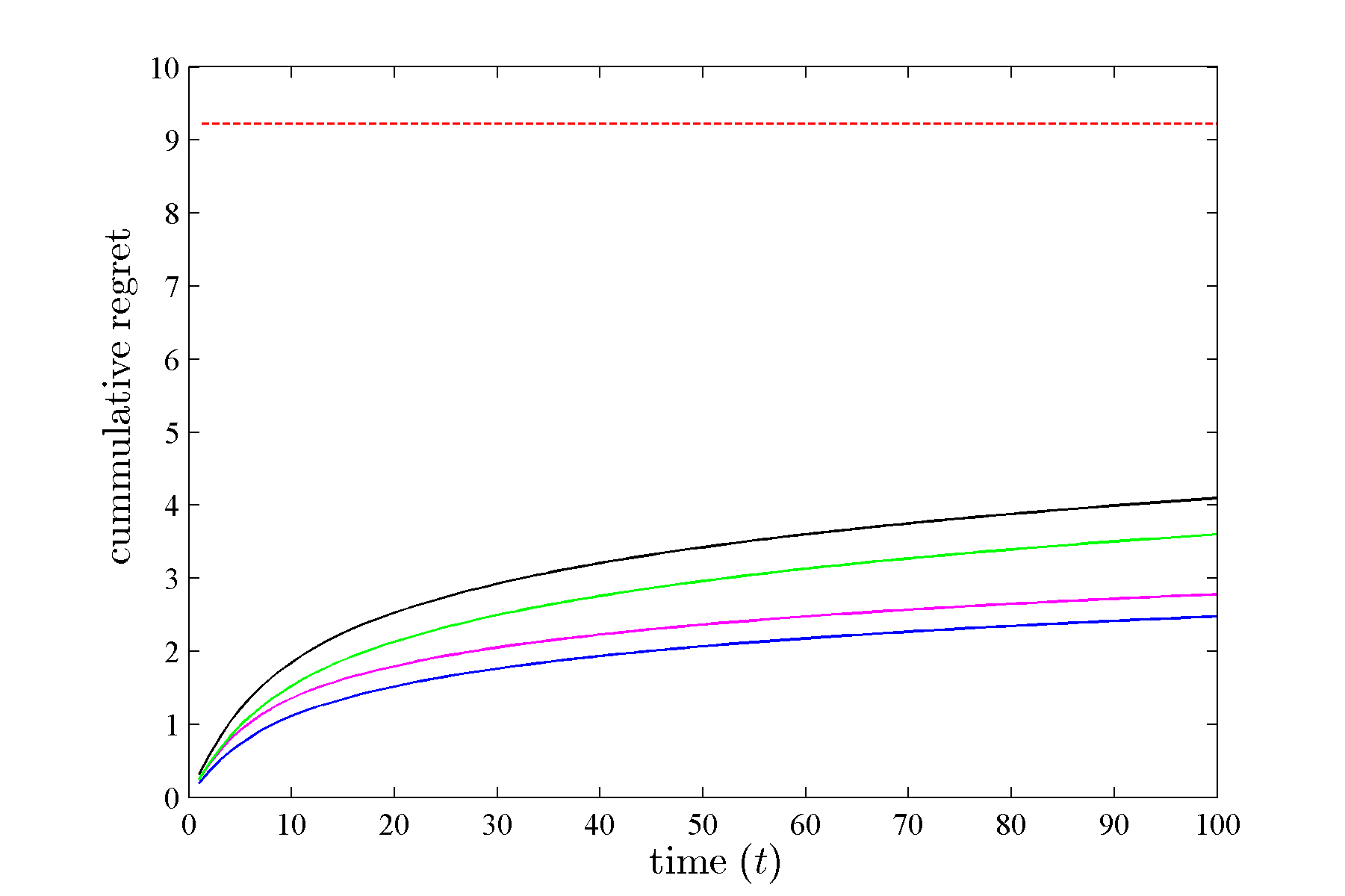}}
  }
\end{figure}

\begin{figure}[htbp]
\floatconts
  {fig:footballdata}
  {\caption{Cumulative regret of the Aggregating Algorithm over the football dataset as used by \cite{vovk2009prediction}, for different choices of substitution functions (Best look ahead(\textcolor{blue}{---}), Worst look ahead(\textcolor{black}{---}), Inverse loss(\textcolor{green}{---}), and Weighted average(\textcolor{magenta}{---})) with learning rate $\eta$ (theoretical regret bound is shown by \textcolor{red}{- - -}).}}
  {
    \subfigure[$\eta = 0.1$]{\label{fig:real_n1d0_e1}
      \includegraphics[width=0.32\linewidth]{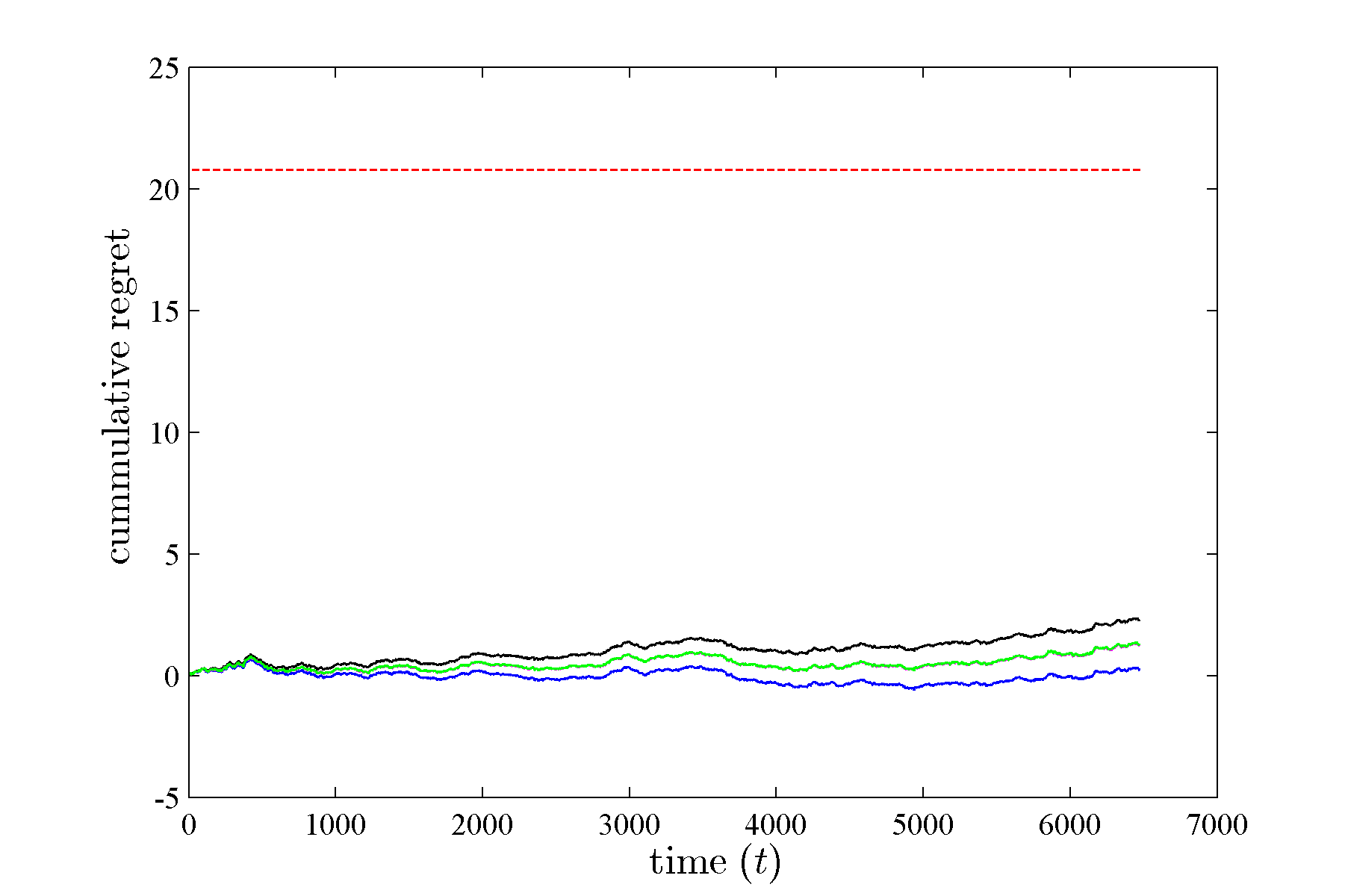}}
    \subfigure[$\eta = 0.3$]{\label{fig:real_n1d0_e1}
      \includegraphics[width=0.32\linewidth]{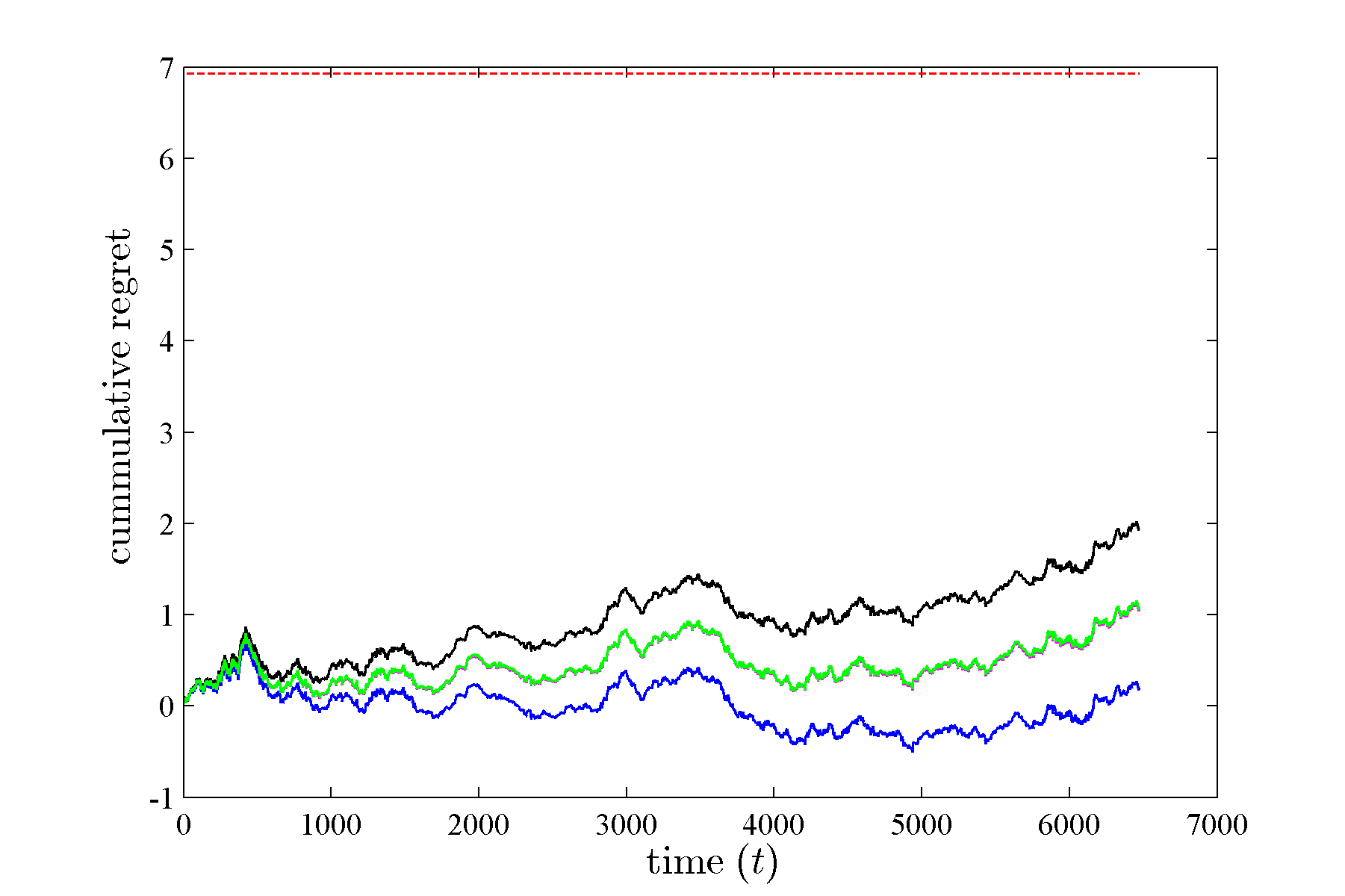}}
    \subfigure[$\eta = 0.5$]{\label{fig:real_n1d0_e1}
      \includegraphics[width=0.32\linewidth]{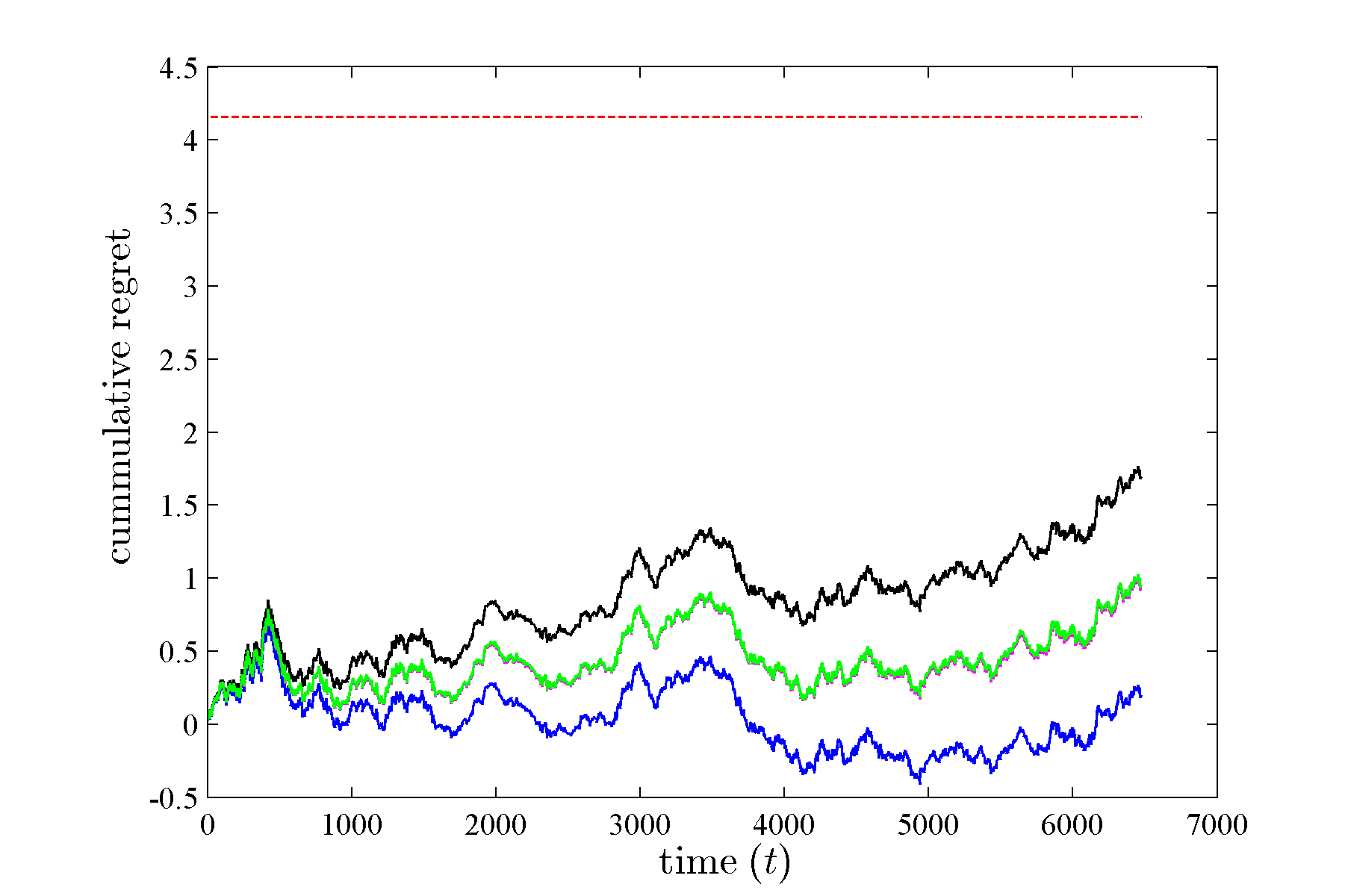}}
  }
\end{figure}

\section{Probability Games with Continuous outcome space}
\label{sec:probability}
We consider an important class of prediction problem called \textit{probability games} (as explained by \cite{vovk2001competitive}), in which the prediction $v$ and the outcome $y$ are probability distributions in some set (for example a finite set of the form $[n]$). A special class of loss functions called \textit{Bregman Loss Functions} (defined below) would be an appropriate choice for the probability games.

Given a differentiable convex function $\phi:\mathcal{S}\rightarrow \mathbb{R}$ defined on a convex set $\mathcal{S} \subset \mathbb{R}^d$ and two points $s_0,s \in \mathcal{S}$ the \textit{Bregman divergence} of $s$ from $s_0$ is defined as
\begin{equation*}
B_{\phi}(s,s_0) := \phi(s) - \phi(s_0) - (s - s_0)' \cdot \textnormal{\textsf{D}} \phi(s_0) ,
\end{equation*}
where $\textnormal{\textsf{D}} \phi(s_0)$ is the gradient of $\phi$ at $s_0$. For any strictly convex function $\phi : \tilde{\Delta}^n \rightarrow \mathbb{R}$, differentiable over the interior of $\tilde{\Delta}^n$, the \textit{Bregman Loss Function} (BLF, \cite{banerjee2005optimality}) $\ell_{\phi}:\Delta^n \times \Delta^n \rightarrow \RR_+$ with generator $\phi$ is given by
\begin{equation}
\label{blfdef}
\ell_{\phi}(y,v) := B_{\phi}(\tilde{y},\tilde{v}) = \phi(\tilde{y}) - \phi(\tilde{v}) - (\tilde{y} - \tilde{v})' \cdot \textnormal{\textsf{D}} \phi(\tilde{v}); \quad y,v \in \Delta^n,
\end{equation}
where $\tilde{y}=\Pi_\Delta(y)$, and $\tilde{v}=\Pi_\Delta(v)$. Since the conditional Bayes risk of a strictly proper loss is strictly concave, any differentiable strictly proper loss $\ell:\Delta^n \rightarrow \RR_+^n$ will generate a BLF $\ell_{\phi}$ with generator $\phi=-\Lubartil_\ell$. Further if $\ell$ is fair, $\ell_{i}(v)=\ell_{\phi}(e_i^n,v)$; i.e. reconstruction is possible. For example the \textit{Kullback-Leibler loss} given by $\ell_{KL}(y,v):=\sum_{i=1}^{n}{y(i)\log{\frac{y(i)}{v(i)}}}$, is a BLF generated by the log loss which is strictly proper.

The following lemma (multi-class extension of a result given by \cite{haussler1998sequential}) provides the mixability condition for probability games.
\begin{lemma}
\label{mixctsn}
For given $\ell:\Delta^n \times \Delta^n \rightarrow \RR_+$, assume that for all $\tilde{y},\tilde{v_1},\tilde{v_2} \in \tilde{\Delta}^n$ (let $y=\Pi_\Delta^{-1}(\tilde{y}), v_1=\Pi_\Delta^{-1}(\tilde{v_1})$, and $v_2=\Pi_\Delta^{-1}(\tilde{v_2})$), the function $g$ defined by
\begin{equation}
\label{gfunc}
g(\tilde{y},\tilde{v_1},\tilde{v_2})=\frac{\beta}{c(\beta)} \ell(y,v_1) - \beta \ell(y,v_2)
\end{equation}
satisfies
\begin{equation}
\label{condmixcts}
\textnormal{\textsf{H}}_{\tilde{y}} g(\tilde{y},\tilde{v_1},\tilde{v_2}) + \textnormal{\textsf{D}}_{\tilde{y}} g(\tilde{y},\tilde{v_1},\tilde{v_2}) \cdot (\textnormal{\textsf{D}}_{\tilde{y}} g(\tilde{y},\tilde{v_1},\tilde{v_2}))' \succcurlyeq 0.
\end{equation}
If
\begin{equation}
\label{mixreqgen}
\exists{\tilde{v^*} \in \tilde{\Delta}^n} \text{ s.t. } \ell(y,v^*) \leq - \frac{c(\beta)}{\beta} \log{\int{e^{-\beta \ell(y,v)}P(d\tilde{v})}}
\end{equation}
holds for the vertices $\tilde{y} = e_i^{n-1}, i \in [n]$, then it holds for all values $\tilde{y} \in \tilde{\Delta}^n$ (where $y=\Pi_\Delta^{-1}(\tilde{y})$, $v^*=\Pi_\Delta^{-1}(\tilde{v^*})$ and $v=\Pi_\Delta^{-1}(\tilde{v})$).
\end{lemma}
\begin{proof}
From \eqref{mixreqgen}
\begin{equation*}
\frac{\beta}{c(\beta)} \ell(y,v^*) + \log{\int{e^{-\beta \ell(y,v)}P(d\tilde{v})}} \leq 0.
\end{equation*}
By exponentiating both sides we get
\begin{equation*}
e^{\frac{\beta}{c(\beta)} \ell(y,v^*)} \cdot \int{e^{-\beta \ell(y,v)}P(d\tilde{v})} \leq 1.
\end{equation*}
Denoting the left hand side of the above inequality by $f(\tilde{y})$ we have
\begin{equation*}
f(\tilde{y}) = \int{e^{g(\tilde{y},\tilde{v^*},\tilde{v})}P(d\tilde{v})}.
\end{equation*}
Since the Hessian of $f$ w.r.t. $\tilde{y}$ given by
\begin{equation*}
\textsf{H}_{\tilde{y}} f(\tilde{y}) = \int{e^{g(\tilde{y},\tilde{v^*},\tilde{v})} \left( \textsf{H}_{\tilde{y}} g(\tilde{y},\tilde{v^*},\tilde{v}) + \textnormal{\textsf{D}}_{\tilde{y}} g(\tilde{y},\tilde{v^*},\tilde{v}) \cdot (\textnormal{\textsf{D}}_{\tilde{y}} g(\tilde{y},\tilde{v^*},\tilde{v}))' \right) P(d\tilde{v})}
\end{equation*}
is positive semi-definite (by \eqref{condmixcts}), $f(\tilde{y})$ is convex in $\tilde{y}$. So the maximum values of $f$ for $\tilde{y} \in \tilde{\Delta}^n$ occurs for some $\tilde{y} = e_i^{n-1}, i \in [n]$. And by noting that, \eqref{mixreqgen} is equivalent to $f(\tilde{y}) \leq 1$ for $\tilde{y} = e_i^{n-1}, i \in [n]$, the proof is completed.
\end{proof}

The next proposition shows that the mixability and exp-concavity of a strictly proper loss is carried over to the BLF generated by it.
\begin{proposition}
\label{properblfconnect}
For a strictly proper fair loss $\ell:\Delta^n \rightarrow \RR_+^n$, and the BLF $\ell_{\phi}:\Delta^n \times \Delta^n \rightarrow \RR_+$ generated by $\ell$ with $\phi=-\Lubartil_\ell$ , if $\ell$ is $\beta$-mixable (resp. $\alpha$-exp-concave), then $\ell_{\phi}$ is also $\beta$-mixable (resp. $\alpha$-exp-concave).
\end{proposition}
\begin{proof}
 From \eqref{gfunc} and \eqref{blfdef}, for the BLF $\ell_{\phi}$ we have
\begin{align*}
g(\tilde{y},\tilde{v_1},\tilde{v_2}) &= \frac{\beta}{c(\beta)} \{ \phi(\tilde{y}) - \phi(\tilde{v_1}) - (\tilde{y} - \tilde{v_1})' \cdot \textsf{D}\phi(\tilde{v_1}) \} - \beta \{ \phi(\tilde{y}) - \phi(\tilde{v_2}) - (\tilde{y} - \tilde{v_2})' \cdot \textsf{D}\phi(\tilde{v_2}) \}, \\
\textsf{D}_{\tilde{y}} g(\tilde{y},\tilde{v_1},\tilde{v_2}) &= \frac{\beta}{c(\beta)} \{ \textsf{D}\phi(\tilde{y}) - \textsf{D}\phi(\tilde{v_1}) \} - \beta \{ \textsf{D}\phi(\tilde{y}) - \textsf{D}\phi(\tilde{v_2}) \}, \\
\textsf{H}_{\tilde{y}} g(\tilde{y},\tilde{v_1},\tilde{v_2}) &= \frac{\beta}{c(\beta)} \textsf{H}\phi(\tilde{y}) - \beta \textsf{H}\phi(\tilde{y}).
\end{align*}
And since $x \cdot x' \succcurlyeq 0, \forall{x \in \RR^n}$, \eqref{condmixcts} is satisfied for all $\tilde{y},\tilde{v_1},\tilde{v_2} \in \tilde{\Delta}^n$ when $c(\beta)=1$, which is the mixability condition (in addition requiring $\tilde{v^*}=\int{\tilde{v}P(d\tilde{v})}$ in \eqref{mixreqgen} is the exp-concavity condition). Then by applying Lemma~\ref{mixctsn} proof is completed.
\end{proof}
As an application of Proposition~\ref{properblfconnect}, we can see that both Kullback-Leibler loss and log loss are 1-mixable and 1-exp-concave.

\section{Proofs}
\label{sec:proof}

\begin{proof} \textbf{(Proposition \ref{geoprop})}
We first prove that the set $\psi(\Delta^n)$ is convex.
For any $p,q \in \Delta^n$, by assumption there exists $c \ge 0$ and $r \in \Delta^n$ such that
$\frac{1}{2}(E_\beta(\ell(p))+E_\beta(\ell(q))) + c \vone_n =E_\beta(\ell(r))$.
Therefore $\frac{1}{2}(\psi(p)+\psi(q))=\psi(r)$,
which implies the convexity of the set $\psi(\Delta^n)$.

Let $T:\RR^{n-1} \ni (e^{-\beta z_1}-e^{-\beta z_n},...,e^{-\beta z_{n-1}}-e^{-\beta z_n})' \rightarrow (e^{-\beta z_1},...,e^{-\beta z_n})' \in [0,1]^n$.
Note this mapping from low dimension to high dimension is well defined because if there are two different $z$ and $\zbar$ in $\ell(\Vcal)$ such that $J E_\beta(z)$ = $J E_\beta(\zbar)$,
then there must be $c \neq 0$ such that $E_\beta(z) + c \vone = E_\beta(\zbar)$.
This means $z > \zbar$ or $z < \zbar$, which violates the strict properness of $\ell$.

Since for any $v=\psi(p)=J E_\beta(\ell(p))$ we have $p=\ell^{-1}\left( E_\beta^{-1}(Tv)\right)$, the link $\psi$ is invertible ($\ell$ is invertible if it is strictly proper (\cite{vernet2011composite}), and $E_\beta$ is invertible for $\beta > 0$).

Now $\ell \circ \psi^{-1}$ is $\beta$-exp-concave if for all $p,q \in \Delta^n$
\begin{equation}
\label{expproof}
E_\beta \left( \ell \circ \psi^{-1} \left( \frac{1}{2}(\psi(p)+\psi(q)) \right) \right) \, \geq \, \frac{1}{2} E_\beta \left( \ell \circ \psi^{-1} (\psi(p)) \right) + \frac{1}{2}E_\beta \left( \ell \circ \psi^{-1} (\psi(q)) \right).
\end{equation}
The right-hand side is obviously $\frac{1}{2}\left(E_\beta(\ell(p))+E_\beta(\ell(q))\right)$. Let $r = \psi^{-1} \left( \frac{1}{2}(\psi(p)+\psi(q)) \right) \in \Delta^n$. Then
\begin{equation*}
J E_\beta(\ell(r))=\psi(r)=\frac{1}{2}\left( J E_\beta(\ell(p)) + J E_\beta(\ell(q)) \right).
\end{equation*}
Therefore $\frac{1}{2}\left( E_\beta(\ell(p)) + E_\beta(\ell(q)) \right) = E_\beta(\ell(r)) + c \vone_n$ for some $c \in \RR$. To establish \eqref{expproof}, it suffices to show $c \leq 0$. But this is guaranteed by the condition assumed.
\end{proof}

\input{multiclass_proof}

\begin{proof} \textbf{(Proposition \ref{complexversion})}
When $n=2$, \eqref{binaryfirstder}, \eqref{binarysecder1} and \eqref{binarysecder2} and the positivity of $\tilde{\psi}'$ simplify \eqref{multiexpcondition} to the two conditions:
\begin{eqnarray*}
(1-\tilde{p}) \, k'(\tilde{p}) &\leq& k(\tilde{p}) - \alpha \, \tilde{\psi}'(\tilde{p}) \, k(\tilde{p})^2 \, (1-\tilde{p})^2, \\
-\tilde{p} \, k'(\tilde{p}) &\leq& k(\tilde{p}) - \alpha \, \tilde{\psi}'(\tilde{p}) \, k(\tilde{p})^2 \, \tilde{p}^2,
\end{eqnarray*}
for all $\tilde{p} \in (0,1)$. These two conditions can be merged as follows
\begin{equation*}
-\frac{1}{\tilde{p}} + \alpha \, \tilde{\psi}'(\tilde{p}) \, k(\tilde{p}) \, \tilde{p} \enspace \leq \enspace \frac{k'(\tilde{p})}{k(\tilde{p})} \enspace \leq \enspace \frac{1}{1-\tilde{p}} - \alpha \, \tilde{\psi}'(\tilde{p}) \, k(\tilde{p}) \, (1-\tilde{p}), \quad \forall{\tilde{p} \in (0,1)}.
\end{equation*}
By noting that $k(\tilde{p})=\frac{w(\tilde{p})}{{\tilde{\psi}}'(\tilde{p})}$ and $k'(\tilde{p})=\frac{w'(\tilde{p}){\tilde{\psi}}'(\tilde{p})-w(\tilde{p}){\tilde{\psi}}''(\tilde{p})}{{{\tilde{\psi}}'(\tilde{p})}^2}$ completes the proof.
\end{proof}

\begin{proof} \textbf{(Proposition \ref{simplerversion})}
Let $g(\tilde{p}) = \frac{1}{w(\tilde{p})}$ and so $g'(\tilde{p}) = - \frac{1}{w(\tilde{p})^2} w'(\tilde{p}), \, g(v) = \int_{\frac{1}{2}}^{v}g'(\tilde{p})d\tilde{p} + g(\frac{1}{2})$ and $g(\frac{1}{2}) = \frac{1}{w(\frac{1}{2})} = 1$. By dividing all sides of \eqref{maineq} by $-w(\tilde{p})$ and applying the substitution we get,
\begin{equation}
\label{temp1}
\frac{1}{\tilde{p}}g(\tilde{p}) - \alpha \tilde{p} \enspace \geq \enspace  g'(\tilde{p}) - \Phi_{\tilde{\psi}}(\tilde{p}) g(\tilde{p}) \enspace \geq \enspace -\frac{1}{1-\tilde{p}}g(\tilde{p}) + \alpha (1-\tilde{p}), \quad \forall \tilde{p} \in (0,1),
\end{equation}
where $\Phi_{\tilde{\psi}}(\tilde{p}):=-\frac{\tilde{\psi}''(\tilde{p})}{\tilde{\psi}'(\tilde{p})}$. If we take the first inequality of \eqref{temp1} and rearrange it we obtain,
\begin{equation}
\label{temp2}
-\alpha \geq \left( g'(\tilde{p})\frac{1}{\tilde{p}} - g(\tilde{p})\frac{1}{\tilde{p}^2} \right) - \Phi_{\tilde{\psi}}(\tilde{p}) \frac{g(\tilde{p})}{\tilde{p}} = \left( \frac{g(\tilde{p})}{\tilde{p}} \right)' - \Phi_{\tilde{\psi}}(\tilde{p}) \left( \frac{g(\tilde{p})}{\tilde{p}} \right), \quad \forall \tilde{p} \in (0,1).
\end{equation}
Multiplying \eqref{temp2} by $e^{-\int_{0}^{\tilde{p}}\Phi_{\tilde{\psi}}(t)dt}$ will result in,
\begin{eqnarray}
-\alpha e^{-\int_{0}^{\tilde{p}}\Phi_{\tilde{\psi}}(t)dt} &\geq& \left( \frac{g(\tilde{p})}{\tilde{p}} \right)' e^{-\int_{0}^{\tilde{p}}\Phi_{\tilde{\psi}}(t)dt} + \left( \frac{g(\tilde{p})}{\tilde{p}} \right) e^{-\int_{0}^{\tilde{p}}\Phi_{\tilde{\psi}}(t)dt} (- \Phi_{\tilde{\psi}}(\tilde{p})) \nonumber \\
&=& \left( \frac{g(\tilde{p})}{\tilde{p}} e^{-\int_{0}^{\tilde{p}}\Phi_{\tilde{\psi}}(t)dt} \right)', \quad \forall \tilde{p} \in (0,1). \label{temp3}
\end{eqnarray}
Since
\begin{equation*}
- \int_{0}^{\tilde{p}}\Phi_{\tilde{\psi}}(t)dt = - \int_{0}^{\tilde{p}} - \frac{\tilde{\psi}''(t)}{\tilde{\psi}'(t)} dt = \int_{0}^{\tilde{p}}(\log{\tilde{\psi}'(t)})'dt = \log{\frac{\tilde{\psi}'(\tilde{p})}{\tilde{\psi}'(0)}},
\end{equation*}
\eqref{temp3} is reduced to
\begin{eqnarray*}
-\alpha \frac{\tilde{\psi}'(\tilde{p})}{\tilde{\psi}'(0)} &\geq& \left( \frac{g(\tilde{p})}{\tilde{p}} \frac{\tilde{\psi}'(\tilde{p})}{\tilde{\psi}'(0)} \right)', \quad \forall \tilde{p} \in (0,1) \\
\Rightarrow -\alpha \tilde{\psi}'(\tilde{p}) &\geq& \left( \frac{g(\tilde{p})}{\tilde{p}} \tilde{\psi}'(\tilde{p}) \right)', \quad \forall \tilde{p} \in (0,1).
\end{eqnarray*}
For $v \geq \frac{1}{2}$ we thus have
\begin{eqnarray*}
-\alpha \int_{\frac{1}{2}}^{v}\tilde{\psi}'(\tilde{p}) d\tilde{p} &\geq& \int_{\frac{1}{2}}^{v}\left( \frac{g(\tilde{p})}{\tilde{p}} \tilde{\psi}'(\tilde{p}) \right)' d\tilde{p}, \quad \forall v \in [1/2,1) \\
\Rightarrow -\alpha (\tilde{\psi}(v) - \tilde{\psi}(\frac{1}{2})) &\geq& \left( \frac{g(v)}{v} \tilde{\psi}'(v) - \frac{g(\frac{1}{2})}{\frac{1}{2}} \tilde{\psi}'(\frac{1}{2})\right) \\
&=& \left( \frac{1}{w(v)v} \tilde{\psi}'(v) - 2 \tilde{\psi}'(\frac{1}{2})\right), \quad \forall v \in [1/2,1) \\
\Rightarrow w(v) &\geq& \frac{\tilde{\psi}'(v)}{v (2\tilde{\psi}'(\frac{1}{2}) - \alpha (\tilde{\psi}(v) - \tilde{\psi}(\frac{1}{2})))}, \quad \forall v \in [1/2,1).
\end{eqnarray*}
Also by considering $v \leq \frac{1}{2}$ case as above, we get
\begin{equation*}
\frac{\tilde{\psi}'(v)}{v (2\tilde{\psi}'(\frac{1}{2}) - \alpha (\tilde{\psi}(v) - \tilde{\psi}(\frac{1}{2})))} \lesseqgtr w(v), \quad \forall v \in (0,1).
\end{equation*}
Finally by following the similar steps for the second inequality of \eqref{temp1}, the proof will be completed.
\end{proof}

Here we provide an integral inequalities related result (without proof) due to Beesack and presented in \cite{dragomir2000some}. 
\begin{theorem}
	\label{beesacktheo}
	Let $y$ and $k$ be continuous and $f$ and $g$ Riemann integrable functions on $J = [\alpha, \beta]$ with $g$ and $k$ nonnegative on $J$. If
	\begin{equation*}
		y(x) \enspace \geq \enspace f(x) + g(x) \int_{\alpha}^{x}y(t)k(t)dt, \quad x \in J,
	\end{equation*}
	then
	\begin{equation*}
		y(x) \enspace \geq \enspace f(x) + g(x) \int_{\alpha}^{x}f(t)k(t)\exp{\left( \int_{t}^{x}g(r)k(r)dr \right)}dt, \quad  x \in J.
	\end{equation*}
	The result remains valid if $\int_{\alpha}^{x}$ is replaced by $\int_{x}^{\beta}$ and $\int_{t}^{x}$ by $\int_{x}^{t}$ throughout.
\end{theorem}
Using the above theorem, we get the following simplified test for the conditions in Theorem~\ref{propexpsuff}: 
\[
-\alpha + \frac{\alpha}{2\tilde{p}^2} - \frac{2}{\tilde{p}^2} ~\leq~ a(\tilde{p}) ~\implies~ \left[ \frac{\alpha (1-\tilde{p})}{\tilde{p}} - \frac{2}{\tilde{p}(1-\tilde{p})} \right] + \frac{1}{\tilde{p}(1-\tilde{p})} \int_{\tilde{p}}^{1/2}a(t)dt ~\leq~ a(\tilde{p})
\]
and
\[
\frac{\alpha \tilde{p}}{(1-\tilde{p})} + \frac{2 \alpha \tilde{p} \! - \! \alpha \! - \! 4}{2 (1-\tilde{p})^2} ~\leq~ b(\tilde{p}) ~\implies~ \left[ \frac{\alpha \tilde{p}}{(1-\tilde{p})} - \frac{2}{\tilde{p}(1-\tilde{p})} \right] + \frac{1}{\tilde{p}(1-\tilde{p})} \int_{1/2}^{\tilde{p}}b(t)dt ~\leq~ b(\tilde{p}) .
\]
Above two identities are proved in the following proof of Theorem~\ref{propexpsuff}.

\begin{proof} \textbf{(Theorem \ref{propexpsuff})}
	The necessary \textit{and} sufficient condition for the exp-concavity of proper losses is given by \eqref{identitynscond}. But from \eqref{temp2}, we can see that \eqref{identitynscond} is equivalent to
	\begin{equation}
	\label{identitynscondequiv1}
	\left( \frac{g(\tilde{p})}{\tilde{p}} \right)' \leq -\alpha, \quad \forall \tilde{p} \in (0,1),
	\end{equation}
	and
	\begin{equation}
	\label{identitynscondequiv2}
	-\left( \frac{g(\tilde{p})}{1-\tilde{p}} \right)' \leq -\alpha, \quad \forall \tilde{p} \in (0,1),
	\end{equation}
	where $g(\tilde{p})=\frac{1}{w(\tilde{p})}$ with $w(\frac{1}{2})=1$; i.e. \eqref{identitynscond} if and only if \eqref{identitynscondequiv1} \& \eqref{identitynscondequiv2}.
	
	Now if we choose the weight function $w(\tilde{p})$ as follows
	\begin{equation}
	\label{weightchoicea}
	w(\tilde{p})=\frac{1}{\tilde{p} \left( 2 + \int_{1/2}^{\tilde{p}}a(t)dt \right)},
	\end{equation}
	such that $a(t) \leq -\alpha$, then \eqref{identitynscondequiv1} will be satisfied (since $\eqref{weightchoicea} \implies 2 + \int_{1/2}^{\tilde{p}}a(t)dt = \frac{1}{w(\tilde{p})\tilde{p}} = \frac{g(\tilde{p})}{\tilde{p}} \implies a(\tilde{p}) = \left( \frac{g(\tilde{p})}{\tilde{p}} \right)'$). Similarly the weight function $w(\tilde{p})$ given by
	\begin{equation}
	\label{weightchoiceb}
	w(\tilde{p})=\frac{1}{(1-\tilde{p}) \left( 2 - \int_{1/2}^{\tilde{p}}b(t)dt \right)},
	\end{equation}
	with $b(t) \leq -\alpha$ will satisfy \eqref{identitynscondequiv2} (since $\eqref{weightchoiceb} \implies 2 - \int_{1/2}^{\tilde{p}}b(t)dt = \frac{1}{w(\tilde{p})(1-\tilde{p})} = \frac{g(\tilde{p})}{1-\tilde{p}} \implies -b(\tilde{p}) = \left( \frac{g(\tilde{p})}{1-\tilde{p}} \right)'$). To satisfy both \eqref{identitynscondequiv1} and \eqref{identitynscondequiv2} at the same time (then obviously \eqref{identitynscond} will be satisfied), we can make the two forms of the weight function (\eqref{weightchoicea} and \eqref{weightchoiceb}) equivalent with the appropriate choice of $a(t)$ and $b(t)$. This can be done in two cases.
	
	In the first case, for $\tilde{p} \in (0,1/2]$ we can fix the weight function $w(\tilde{p})$ as given by \eqref{weightchoicea} and choose $a(t)$ such that,
	\begin{itemize}
		\item{$a(t) \leq -\alpha$ (then \eqref{identitynscondequiv1} is satisfied) and}
		\item{$\eqref{weightchoicea} = \eqref{weightchoiceb} \implies b(t) \leq -\alpha$ (then \eqref{identitynscondequiv2} is satisfied).}
	\end{itemize}
	But $\eqref{weightchoicea} = \eqref{weightchoiceb}$ for all $\tilde{p} \in (0,1/2]$ if and only if
	\begin{eqnarray*}
		&&\tilde{p} \left( 2 - \int_{\tilde{p}}^{1/2}a(t)dt \right)=(1-\tilde{p}) \left( 2 + \int_{\tilde{p}}^{1/2}b(t)dt \right), \quad \tilde{p} \in (0,1/2] \\
		&\iff&\frac{\tilde{p}}{1-\tilde{p}} \left( 2 - \int_{\tilde{p}}^{1/2}a(t)dt \right) - 2 = \int_{\tilde{p}}^{1/2}b(t)dt, \quad \tilde{p} \in (0,1/2] \\
		&\iff&\frac{\tilde{p}}{1-\tilde{p}}a(\tilde{p}) + \frac{1}{(1-\tilde{p})^2} \left( 2 - \int_{\tilde{p}}^{1/2}a(t)dt \right) = -b(\tilde{p}), \quad \tilde{p} \in (0,1/2],
	\end{eqnarray*}
	where the last step is obtained by differentiating both sides w.r.t $\tilde{p}$. Thus the constraint $\eqref{weightchoicea} = \eqref{weightchoiceb} \implies b(t) \leq -\alpha$ can be given as
	\begin{eqnarray*}
		&&\frac{\tilde{p}}{1-\tilde{p}}a(\tilde{p}) + \frac{1}{(1-\tilde{p})^2} \left( 2 - \int_{\tilde{p}}^{1/2}a(t)dt \right) \enspace \geq \enspace \alpha, \quad \tilde{p} \in (0,1/2] \\
		&\iff& a(\tilde{p}) \enspace \geq \enspace \left[ \frac{\alpha (1-\tilde{p})}{\tilde{p}} - \frac{2}{\tilde{p}(1-\tilde{p})} \right] + \frac{1}{\tilde{p}(1-\tilde{p})} \int_{\tilde{p}}^{1/2}a(t)dt, \quad \tilde{p} \in (0,1/2] \\
		&\iff& a(\tilde{p}) \enspace \geq \enspace f(\tilde{p}) + g(\tilde{p}) \int_{\tilde{p}}^{1/2}a(t)k(t)dt, \quad \tilde{p} \in (0,1/2],
	\end{eqnarray*}
	where $f(\tilde{p}) = \left[ \frac{\alpha (1-\tilde{p})}{\tilde{p}} - \frac{2}{\tilde{p}(1-\tilde{p})} \right]$, $g(\tilde{p}) = \frac{1}{\tilde{p}(1-\tilde{p})} = \frac{1}{\tilde{p}} + \frac{1}{(1-\tilde{p})}$ and $k(t)=1$. Now by applying Theorem~\ref{beesacktheo} we have
	\begin{equation*}
		a(\tilde{p}) \enspace \geq \enspace f(\tilde{p}) + g(\tilde{p}) \int_{\tilde{p}}^{1/2}f(t)k(t)\exp{\left( \int_{\tilde{p}}^{t}g(r)k(r)dr \right)}dt, \quad \tilde{p} \in (0,1/2].
	\end{equation*}
	Since
	\begin{equation*}
		\int_{\tilde{p}}^{t}g(r)k(r)dr = \int_{\tilde{p}}^{t}\frac{1}{r} + \frac{1}{(1-r)}dr = [\ln r - \ln (1-r)]_{\tilde{p}}^{t} = \ln \left(\frac{t}{(1-t)} \frac{(1-\tilde{p})}{\tilde{p}}\right),
	\end{equation*}
	\begin{eqnarray*}
		\int_{\tilde{p}}^{1/2}f(t)k(t)\exp{\left( \int_{\tilde{p}}^{t}g(r)k(r)dr \right)}dt &=& \int_{\tilde{p}}^{1/2}f(t)\frac{t}{(1-t)} \frac{(1-\tilde{p})}{\tilde{p}}dt \\
		&=& \frac{(1-\tilde{p})}{\tilde{p}} \int_{\tilde{p}}^{1/2}\left[ \frac{\alpha (1-t)}{t} - \frac{2}{t(1-t)} \right]\frac{t}{(1-t)} dt \\
		&=& \frac{(1-\tilde{p})}{\tilde{p}} \int_{\tilde{p}}^{1/2} \alpha - \frac{2}{(1-t)^2} dt \\
		&=& \frac{(1-\tilde{p})}{\tilde{p}} \left[ \alpha t - \frac{2}{1-t} \right]_{\tilde{p}}^{1/2} \\
		&=& \frac{(1-\tilde{p})}{\tilde{p}} \left[ \frac{\alpha}{2} - 4 - \alpha \tilde{p} + \frac{2}{1-\tilde{p}} \right],
	\end{eqnarray*}
	we get
	\begin{eqnarray*}
		a(\tilde{p}) &\geq& \left[ \frac{\alpha (1-\tilde{p})}{\tilde{p}} - \frac{2}{\tilde{p}(1-\tilde{p})} \right] + \frac{1}{\tilde{p}(1-\tilde{p})} \frac{(1-\tilde{p})}{\tilde{p}} \left[ \frac{\alpha}{2} - 4 - \alpha \tilde{p} + \frac{2}{1-\tilde{p}} \right], \quad \tilde{p} \in (0,1/2] \\
		&=& -\alpha+ \frac{\alpha}{2\tilde{p}^2} - \frac{2}{\tilde{p}^2}, \quad \tilde{p} \in (0,1/2].
	\end{eqnarray*}
	
	Similarly in the second case, for $\tilde{p} \in [1/2,1)$ we can fix the weight function $w(\tilde{p})$ as given by \eqref{weightchoiceb} and choose $b(t)$ such that,
	\begin{itemize}
		\item{$b(t) \leq -\alpha$ (then \eqref{identitynscondequiv2} is satisfied) and}
		\item{$\eqref{weightchoicea} = \eqref{weightchoiceb} \implies a(t) \leq -\alpha$ (then \eqref{identitynscondequiv1} is satisfied).}
	\end{itemize}
	But $\eqref{weightchoicea} = \eqref{weightchoiceb}$ for all $\tilde{p} \in [1/2,1)$ if and only if
	\begin{eqnarray*}
		&&\tilde{p} \left( 2 + \int_{1/2}^{\tilde{p}}a(t)dt \right)=(1-\tilde{p}) \left( 2 - \int_{1/2}^{\tilde{p}}b(t)dt \right), \quad \tilde{p} \in [1/2,1) \\
		&\iff&\int_{\tilde{p}}^{1/2}a(t)dt = \frac{1-\tilde{p}}{\tilde{p}} \left( 2 - \int_{1/2}^{\tilde{p}}b(t)dt \right) - 2, \quad \tilde{p} \in [1/2,1) \\
		&\iff&a(\tilde{p}) = -\frac{1-\tilde{p}}{\tilde{p}}b(\tilde{p}) - \frac{1}{\tilde{p}^2} \left( 2 - \int_{1/2}^{\tilde{p}}b(t)dt \right), \quad \tilde{p} \in [1/2,1),
	\end{eqnarray*}
	where the last step is obtained by differentiating both sides w.r.t $\tilde{p}$. Thus the constraint $\eqref{weightchoicea} = \eqref{weightchoiceb} \implies a(t) \leq -\alpha$ can be given as
	\begin{eqnarray*}
		&&\frac{1-\tilde{p}}{\tilde{p}}b(\tilde{p}) + \frac{1}{\tilde{p}^2} \left( 2 - \int_{1/2}^{\tilde{p}}b(t)dt \right) \enspace \geq \enspace \alpha, \quad \tilde{p} \in [1/2,1) \\
		&\iff& b(\tilde{p}) \enspace \geq \enspace \left[ \frac{\alpha \tilde{p}}{(1-\tilde{p})} - \frac{2}{\tilde{p}(1-\tilde{p})} \right] + \frac{1}{\tilde{p}(1-\tilde{p})} \int_{1/2}^{\tilde{p}}b(t)dt, \quad \tilde{p} \in [1/2,1) \\
		&\iff& b(\tilde{p}) \enspace \geq \enspace f(\tilde{p}) + g(\tilde{p}) \int_{1/2}^{\tilde{p}}b(t)k(t)dt, \quad \tilde{p} \in [1/2,1),
	\end{eqnarray*}
	where $f(\tilde{p}) = \left[ \frac{\alpha \tilde{p}}{(1-\tilde{p})} - \frac{2}{\tilde{p}(1-\tilde{p})} \right]$, $g(\tilde{p}) = \frac{1}{\tilde{p}(1-\tilde{p})} = \frac{1}{\tilde{p}} + \frac{1}{(1-\tilde{p})}$ and $k(t)=1$. Again by applying Theorem~\ref{beesacktheo} we have
	\begin{equation*}
		b(\tilde{p}) \enspace \geq \enspace f(\tilde{p}) + g(\tilde{p}) \int_{1/2}^{\tilde{p}}f(t)k(t)\exp{\left( \int_{t}^{\tilde{p}}g(r)k(r)dr \right)}dt, \quad \tilde{p} \in [1/2,1).
	\end{equation*}
	Since
	\begin{equation*}
		\int_{t}^{\tilde{p}}g(r)k(r)dr = \int_{t}^{\tilde{p}}\frac{1}{r} + \frac{1}{(1-r)}dr = [\ln r - \ln (1-r)]_{t}^{\tilde{p}} = \ln \left(\frac{\tilde{p}}{(1-\tilde{p})} \frac{(1-t)}{t}\right),
	\end{equation*}
	\begin{eqnarray*}
		\int_{1/2}^{\tilde{p}}f(t)k(t)\exp{\left( \int_{t}^{\tilde{p}}g(r)k(r)dr \right)}dt &=& \int_{1/2}^{\tilde{p}}f(t)\frac{\tilde{p}}{(1-\tilde{p})} \frac{(1-t)}{t}dt \\
		&=& \frac{\tilde{p}}{(1-\tilde{p})} \int_{1/2}^{\tilde{p}}\left[ \frac{\alpha t}{(1-t)} - \frac{2}{t(1-t)} \right]\frac{(1-t)}{t} dt \\
		&=& \frac{\tilde{p}}{(1-\tilde{p})} \int_{1/2}^{\tilde{p}} \alpha - \frac{2}{t^2} dt \\
		&=& \frac{\tilde{p}}{(1-\tilde{p})} \left[ \alpha t + \frac{2}{t} \right]_{1/2}^{\tilde{p}} \\
		&=& \frac{\tilde{p}}{(1-\tilde{p})} \left[ \alpha \tilde{p} + \frac{2}{\tilde{p}} - \frac{\alpha}{2} - 4 \right],
	\end{eqnarray*}
	we get
	\begin{eqnarray*}
		b(\tilde{p}) &\geq& \left[ \frac{\alpha \tilde{p}}{(1-\tilde{p})} - \frac{2}{\tilde{p}(1-\tilde{p})} \right] + \frac{1}{\tilde{p}(1-\tilde{p})} \frac{\tilde{p}}{(1-\tilde{p})} \left[ \alpha \tilde{p} + \frac{2}{\tilde{p}} - \frac{\alpha}{2} - 4 \right], \quad \tilde{p} \in [1/2,1) \\
		&=& \frac{\alpha \tilde{p}}{(1-\tilde{p})} + \frac{\alpha \tilde{p}}{(1-\tilde{p})^2} - \frac{\alpha}{2(1-\tilde{p})^2} - \frac{2}{(1-\tilde{p})^2}, \quad \tilde{p} \in [1/2,1).
	\end{eqnarray*}
\end{proof}


\begin{proof} \textbf{(Corollary \ref{specialcoro})}
We have to show that $\tilde{\psi}_{\ell}^*$ will satisfy \eqref{maineq} with $\alpha = \beta$, for all $\beta$-mixable proper loss functions. Since
\begin{equation}
\label{explinkdesign}
\frac{(\tilde{\psi}_{\ell}^*)''(\tilde{p})}{(\tilde{\psi}_{\ell}^*)'(\tilde{p})} = \frac{\frac{w_{\ell}'(\tilde{p}) w_{\ell^{\mathrm{log}}}(\tilde{p}) - w_{\ell}(\tilde{p}) w_{\ell^{\mathrm{log}}}'(\tilde{p})}{{w_{\ell^{\mathrm{log}}}(\tilde{p})}^2}}{\frac{w_{\ell}(\tilde{p})}{w_{\ell^{\mathrm{log}}}(\tilde{p})}} = \frac{w_{\ell}'(\tilde{p})}{w_{\ell}(\tilde{p})} - \frac{w_{\ell^{\mathrm{log}}}'(\tilde{p})}{w_{\ell^{\mathrm{log}}}(\tilde{p})} = \frac{w_{\ell}'(\tilde{p})}{w_{\ell}(\tilde{p})} - (\log{w_{\ell^{\mathrm{log}}}(\tilde{p})})' ,
\end{equation}
by substituting ${\tilde{\psi}} = \tilde{\psi}_{\ell}^*$ and $\alpha = \beta$ in \eqref{maineq} we have,
\begin{eqnarray*}
&&-\frac{1}{\tilde{p}} + \beta w_{\ell}(\tilde{p}) \tilde{p} \enspace \leq \enspace \frac{w_{\ell}'(\tilde{p})}{w_{\ell}(\tilde{p})} - \frac{(\tilde{\psi}_{\ell}^*)''(\tilde{p})}{(\tilde{\psi}_{\ell}^*)'(\tilde{p})} \enspace \leq \enspace \frac{1}{1-\tilde{p}} - \beta w_{\ell}(\tilde{p}) (1-\tilde{p}), \quad \forall \tilde{p} \in (0,1) \\
&\iff&-\frac{1}{\tilde{p}} + \beta w_{\ell}(\tilde{p}) \tilde{p} \enspace \leq \enspace (\log{w_{\ell^{\mathrm{log}}}(\tilde{p})})' \enspace \leq \enspace \frac{1}{1-\tilde{p}} - \beta w_{\ell}(\tilde{p}) (1-\tilde{p}), \quad \forall \tilde{p} \in (0,1) \\
&\iff&-\frac{1}{\tilde{p}} + \beta w_{\ell}(\tilde{p}) \tilde{p} \enspace \leq \enspace -\frac{1}{\tilde{p}} + \frac{1}{1-\tilde{p}} \enspace \leq \enspace \frac{1}{1-\tilde{p}} - \beta w_{\ell}(\tilde{p}) (1-\tilde{p}), \quad \forall \tilde{p} \in (0,1) \\
&\iff&\beta \enspace \leq \enspace \frac{1}{\tilde{p}(1-\tilde{p})w_{\ell}(\tilde{p})} \enspace = \enspace \frac{w_{\ell^{\mathrm{log}}}(\tilde{p})}{w_{\ell}(\tilde{p})}, \quad \forall \tilde{p} \in (0,1),
\end{eqnarray*}
which is true for all $\beta$-mixable binary proper loss functions. From \eqref{explinkdesign}
\begin{align*}
\left( \log{(\tilde{\psi}_{\ell}^*)'(\tilde{p})} \right)' &= (\log{w_{\ell}(\tilde{p})})' - (\log{w_{\ell^{\mathrm{log}}}(\tilde{p})})' = \left( \log{\frac{w_{\ell}(\tilde{p})}{w_{\ell^{\mathrm{log}}}(\tilde{p})}} \right)', \\
\Rightarrow \left[ \log{(\tilde{\psi}_{\ell}^*)'(\tilde{p})} \right]_{1/2}^{\tilde{p}} &= \left[ \log{\frac{w_{\ell}(\tilde{p})}{w_{\ell^{\mathrm{log}}}(\tilde{p})}} \right]_{1/2}^{\tilde{p}}, \\
\Rightarrow \log{(\tilde{\psi}_{\ell}^*)'(\tilde{p})} - \log{(\tilde{\psi}_{\ell}^*)'\left(\frac{1}{2}\right)} &= \log{\frac{w_{\ell}(\tilde{p})}{w_{\ell^{\mathrm{log}}}(\tilde{p})} \cdot \frac{w_{\ell^{\mathrm{log}}}(\frac{1}{2})}{w_{\ell}(\frac{1}{2})}},
\end{align*}
it can be seen that a design choice of $(\tilde{\psi}_{\ell}^*)'\left(\frac{1}{2}\right) = 1$ is made in the construction of this link function.
\end{proof}

%

\section{Squared Loss}
\label{sec:square_loss}

\begin{figure}[htbp!]
\centering
  \includegraphics[width=7cm]{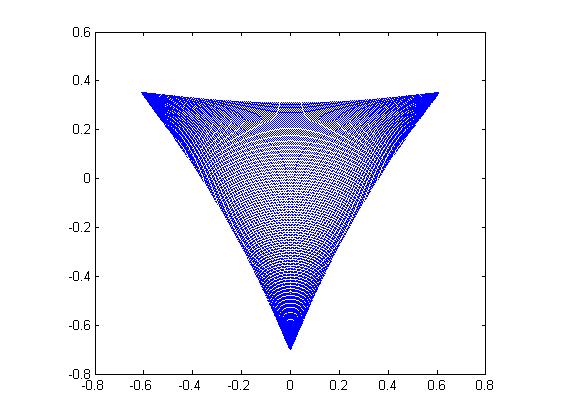}\\
  \caption{Projection of the exp-prediction set of square loss ($\beta=1$) along the $\vone_3$ direction.
   By the apparent lack of convexity of the projection, the condition $\partial_{\vone_n} \mathcal{B}_{\beta} \subseteq E_\beta (\ell(\mathcal{V}))$ in Proposition \ref{geoprop} does not hold in this case.}\label{fig:proj_square}
\end{figure}

In this section we will consider the multi-class squared loss with partial losses given by $\ell_{i}^{\mathrm{sq}}(p):=\sum_{j \in [n]}(\llbracket i=j \rrbracket - p_j)^2$. The Bayes risk of this loss is $\Lubartil_{\ell^{\mathrm{sq}}}(\tilde{p}) = 1 - \sum_{i=1}^{n-1}{p_i^2} - (1 - \sum_{i=1}^{n-1}{p_i})^2$. Thus the Hessian of the Bayes risk is given by
\begin{equation*}
\textnormal{\textsf{H}}\Lubartil_{\ell^{\mathrm{sq}}}(\tilde{p}) = 2
\begin{pmatrix}
  -2 & -1 & \cdots & -1 \\
  -1 & -2 & \cdots & -1 \\
  \vdots  & \vdots  & \ddots & \vdots  \\
  -1 & -1 & \cdots & -2
 \end{pmatrix}.
\end{equation*}
For the identity link, from \eqref{kpeq} we get $k(\tilde{p})=-\textnormal{\textsf{H}}\Lubartil_{\ell^{\mathrm{sq}}}(\tilde{p})$ and $\textnormal{\textsf{D}}_v[k(\tilde{p})]=0$ since $\textnormal{\textsf{D}}\tilde{\psi}(\tilde{p})=I_{n-1}$. Thus from \eqref{multiexpcondition}, the multi-class squared loss is $\alpha$-exp-concave (with $\alpha > 0$) if and only if for all $\tilde{p} \in \mathring{\tilde{\Delta}}^n$ and for all $i \in [n]$
\begin{align}
  0 &\preccurlyeq k(\tilde{p}) - \alpha k(\tilde{p}) \cdot (e_i^{n-1}-\tilde{p}) \cdot (e_i^{n-1}-\tilde{p})' \cdot k(\tilde{p}) \nonumber \\
  \iff k(\tilde{p})^{-1} &\succcurlyeq \alpha (e_i^{n-1}-\tilde{p}) \cdot (e_i^{n-1}-\tilde{p})'. \label{idcancond1}
\end{align}
Similarly for the canonical link, from \eqref{canlinkcond1} and \eqref{canlinkcond2}, the composite loss is $\alpha$-exp-concave (with $\alpha > 0$) if and only if for all $\tilde{p} \in \mathring{\tilde{\Delta}}^n$ and for all $i \in [n]$
\begin{equation}
k(\tilde{p})^{-1} = -[\textnormal{\textsf{H}}\Lubartil_{\ell^{\mathrm{sq}}}(\tilde{p})]^{-1} \succcurlyeq \alpha (e_i^{n-1}-\tilde{p}) \cdot (e_i^{n-1}-\tilde{p})'. \label{idcancond2}
\end{equation}
From \eqref{idcancond1} and \eqref{idcancond2}, it can be seen that for the multi-class squared loss the level of exp-concavification by identity link and canonical link are same. When $n=2$, since $k(\tilde{p})=4$, the condition \eqref{idcancond1} is equivalent to
\begin{eqnarray*}
&&\frac{1}{4} \geq \alpha (e_i^{n-1}-\tilde{p}) \cdot (e_i^{n-1}-\tilde{p})', \quad i \in [2], \forall{\tilde{p} \in (0,1)} \\
&\iff&\alpha \leq \frac{1}{4\tilde{p}^2} \quad \text{and} \quad \alpha \leq \frac{1}{4(1-\tilde{p})^2}, \quad \forall{\tilde{p} \in (0,1)} \\
&\iff&\alpha \leq \frac{1}{4}.
\end{eqnarray*}
When $n=3$, using the fact that a $2 \times 2$ matrix is positive semi-definite if its trace and determinant are both non-negative, it can be easily verified that the condition \eqref{idcancond1} is equivalent to $\alpha \leq \frac{1}{12}$.

For binary squared loss, the link functions constructed by geometric (Proposition \ref{geoprop}) and calculus (Corollary \ref{specialcoro}) approach are:
\begin{equation*}
\tilde{\psi}(\tilde{p}) = e^{-2(1-\tilde{p})^2} - e^{-2\tilde{p}^2} \quad \text{and} \quad \tilde{\psi}_{\ell}^*(\tilde{p}) = \frac{4}{4} \int_{0}^{\tilde{p}}{\frac{w_{\ell^{\mathrm{sq}}}(v)}{w_{\ell^{\mathrm{log}}}(v)}dv} = 4 \left(\frac{\tilde{p}^2}{2} - \frac{\tilde{p}^3}{3}\right),
\end{equation*}
respectively. By applying these link functions we can get 1-exp-concave composite squared loss.

\section{Boosting Loss}
\label{sec:boosting_loss}

\begin{figure}
  \centering
    \includegraphics[width=0.50\textwidth]{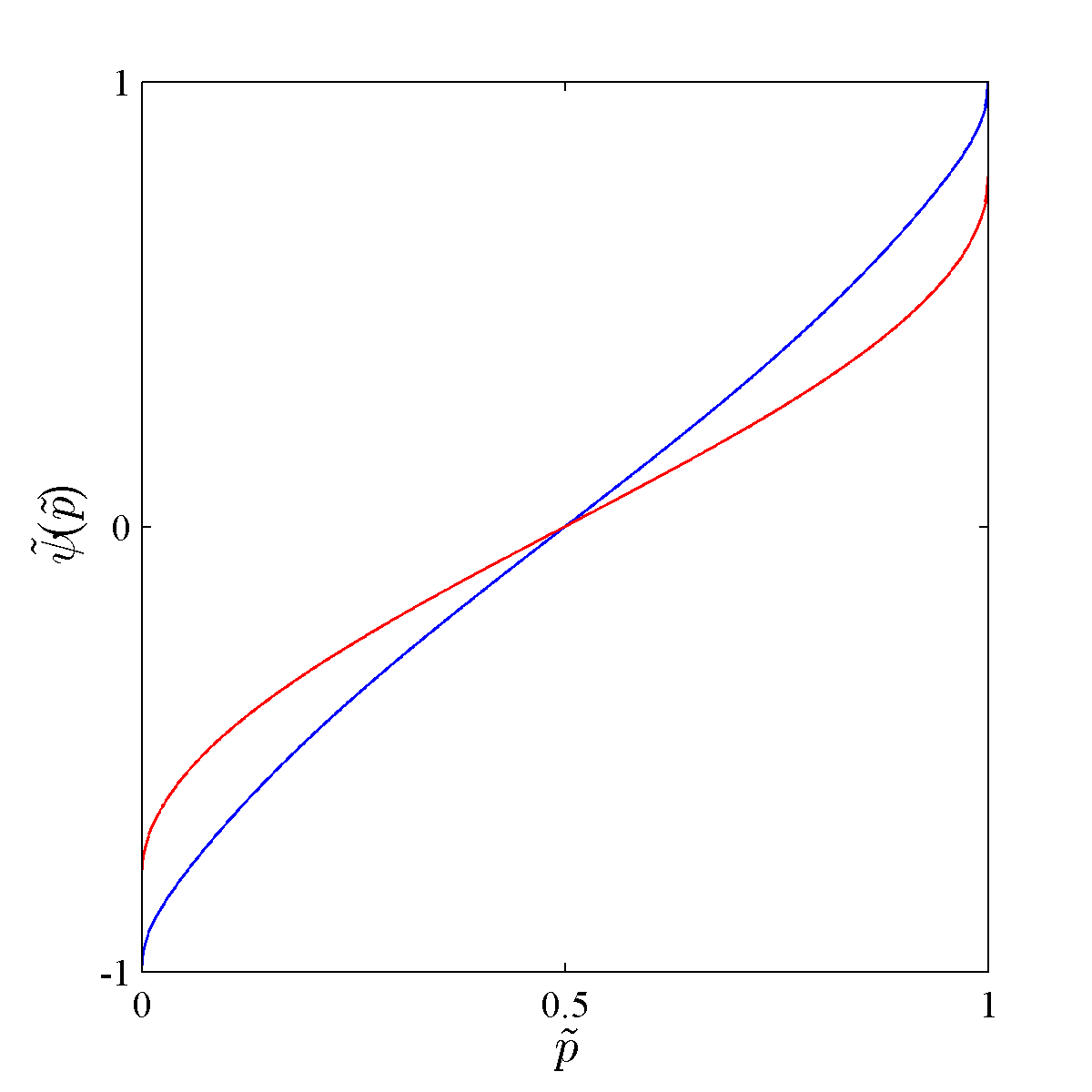}
  \caption{Exp-concavifying link functions for binary boosting loss constructed by Proposition \ref{geoprop} (\textcolor{blue}{---}) and Corollary \ref{specialcoro} (\textcolor{red}{---}).}
  \label{fig:expboost}
\end{figure}

Consider the binary ``boosting loss" (\cite{buja2005loss}) with partial losses given by
\begin{equation*}
\ell^{\mathrm{boost}}_1(\tilde{p}) = \frac{1}{2} \sqrt{\frac{1-\tilde{p}}{\tilde{p}}} \quad \text{and} \quad \ell^{\mathrm{boost}}_2(\tilde{p}) = \frac{1}{2} \sqrt{\frac{\tilde{p}}{1-\tilde{p}}}, \quad \forall \tilde{p} \in (0,1).
\end{equation*}
This loss has weight function
\begin{equation*}
w_{\ell^{\mathrm{boost}}}(\tilde{p}) = \frac{1}{4(\tilde{p} (1-\tilde{p}))^{3/2}}, \quad \forall \tilde{p} \in (0,1).
\end{equation*}
By applying the results of \cite{van2012mixability}, we can show that this loss is mixable with mixability constant 2 (since $\beta_{\ell}=\inf_{\tilde{p} \in (0,1)}{\frac{w_{\ell^{\mathrm{log}}}(\tilde{p})}{w_{\ell}(\tilde{p})}}$).

Now we can check the level of exp-concavification of this loss for different choices of link functions. By considering the identity link $\tilde{\psi}(\tilde{p}) = \tilde{p}$, from \eqref{identitynscond}
\begin{align*}
-\frac{1}{\tilde{p}} + \alpha w_{\ell^{\mathrm{boost}}}(\tilde{p}) \tilde{p} \enspace &\leq \enspace \frac{w_{\ell^{\mathrm{boost}}}'(\tilde{p})}{w_{\ell^{\mathrm{boost}}}(\tilde{p})}, \quad \forall \tilde{p} \in (0,1) \\
\Rightarrow -\frac{1}{\tilde{p}} + \alpha w_{\ell^{\mathrm{boost}}}(\tilde{p}) \tilde{p} \enspace &\leq \enspace 6 w_{\ell^{\mathrm{boost}}}(\tilde{p}) \sqrt{\tilde{p} (1-\tilde{p})} (2\tilde{p}-1), \quad \forall \tilde{p} \in (0,1) \\
\Rightarrow \alpha \tilde{p} - 6 \sqrt{\tilde{p} (1-\tilde{p})} (2\tilde{p}-1) \enspace &\leq \enspace \frac{1}{w_{\ell^{\mathrm{boost}}}(\tilde{p}) \tilde{p}}, \quad \forall \tilde{p} \in (0,1) \\
\Rightarrow \alpha \enspace &\leq \enspace 8 \sqrt{\frac{1-\tilde{p}}{\tilde{p}}} (\tilde{p}-1/4), \quad \forall \tilde{p} \in (0,1) \\
\Rightarrow \alpha \enspace &\leq \enspace 0,
\end{align*}
we see that the boosting loss is non-exp-concave. Similarly from \eqref{binarycanonical}
\begin{equation}
\label{boosteq}
\alpha \enspace \leq \enspace \frac{1}{w_{\ell^{\mathrm{boost}}}(\tilde{p}) \tilde{p}^2} \enspace = \enspace 4 \sqrt{\frac{1-\tilde{p}}{\tilde{p}}} (1-\tilde{p}), \quad \forall \tilde{p} \in (0,1)
\end{equation}
it can be seen that the RHS of \eqref{boosteq} approaches $0$ as $p \rightarrow 1$, thus it is not possible to exp-concavify (for some $\alpha > 0$) this loss using the canonical link. For binary boosting loss, the link functions constructed by geometric (Proposition \ref{geoprop}) and calculus (Corollary \ref{specialcoro}) approach are:
\begin{equation*}
\tilde{\psi}(\tilde{p}) = e^{-\sqrt{\frac{1-\tilde{p}}{\tilde{p}}}} - e^{-\sqrt{\frac{\tilde{p}}{1-\tilde{p}}}} \quad \text{and} \quad \tilde{\psi}_{\ell}^*(\tilde{p}) = \frac{4}{2} \int_{0}^{\tilde{p}}{\frac{w_{\ell^{\mathrm{boost}}}(v)}{w_{\ell^{\mathrm{log}}}(v)}dv} = \frac{1}{2} \arcsin(-1+2 \tilde{p}),
\end{equation*}
respectively (as shown in Figure~\ref{fig:expboost}). By applying these link functions we can get 2-exp-concave composite boosting loss.


\section{Log Loss}
\label{sec:log_loss}

By using the results from this paper and \cite{van2012mixability} one can easily verify that the multi-class log loss is both 1-mixable and 1-exp-concave.  For binary log loss, the link functions constructed by geometric (Proposition \ref{geoprop}) and calculus (Corollary \ref{specialcoro}) approach are:
\begin{equation*}
\tilde{\psi}(\tilde{p}) = e^{\log{\tilde{p}}} - e^{\log{1-\tilde{p}}} = 2 \tilde{p} - 1 \quad \text{and} \quad \tilde{\psi}_{\ell}^*(\tilde{p}) = \frac{4}{4} \int_{0}^{\tilde{p}}{\frac{w_{\ell^{\mathrm{log}}}(v)}{w_{\ell^{\mathrm{log}}}(v)}dv} = \tilde{p},
\end{equation*}
respectively.

\end{document}

%% file: multiclass.tex
Unfortunately, the condition that
$\partial_{\vone_n} \mathcal{B}_{\beta} \subseteq E_\beta (\ell(\mathcal{V}))$
is generally not satisfied;
an example based on squared loss ($\beta= 1$ and $n=3$ classes) is shown in Figure \ref{fig:hole_square},
where for $A$ and $B$ in $E_\beta (\ell(\mathcal{V}))$, the mid-point $C$ can travel along the ray of direction $\vone_3$ without hitting any point in the exp-prediction set $E_\beta(\ell(\mathcal{V}))$.
Therefore we resort to \emph{approximating} a given $\beta$-mixable loss by a sequence of $\beta$-exp-concave losses parameterised by positive constant $\epsilon$,
while the approximation approaches the original loss in some appropriate sense as $\epsilon$ tends to 0.
Without loss of generality, we assume $\Vcal = \Delta^n$.

Inspired by Proposition~\ref{geoprop}, a natural idea to construct the approximation is by adding ``faces" to the exp-prediction set such that all rays in the $\vone_n$ direction will be blocked.
Technically, it turns out more convenient to add faces that block rays in (almost) all directions of positive orthant.
See Figure \ref{fig:extend_3d} for an illustration.
In particular, we extend the ``rim" of the exp-prediction set by hyperplanes that are $\epsilon$ close to  axis-parallel.
The key challenge underlying this idea is to design an appropriate \emph{parameterisation} of the surrogate loss $\elltil_\epsilon$,
which not only produces such an extended exp-prediction set,
but also ensures that $\elltil_\epsilon(p) = \ell(p)$ for almost all $p \in \Delta^n$ as $\epsilon \downarrow 0$.

\begin{figure*}[t]
\begin{minipage}[b]{0.31\textwidth}
\includegraphics[width=1\textwidth,height=0.9\textwidth]{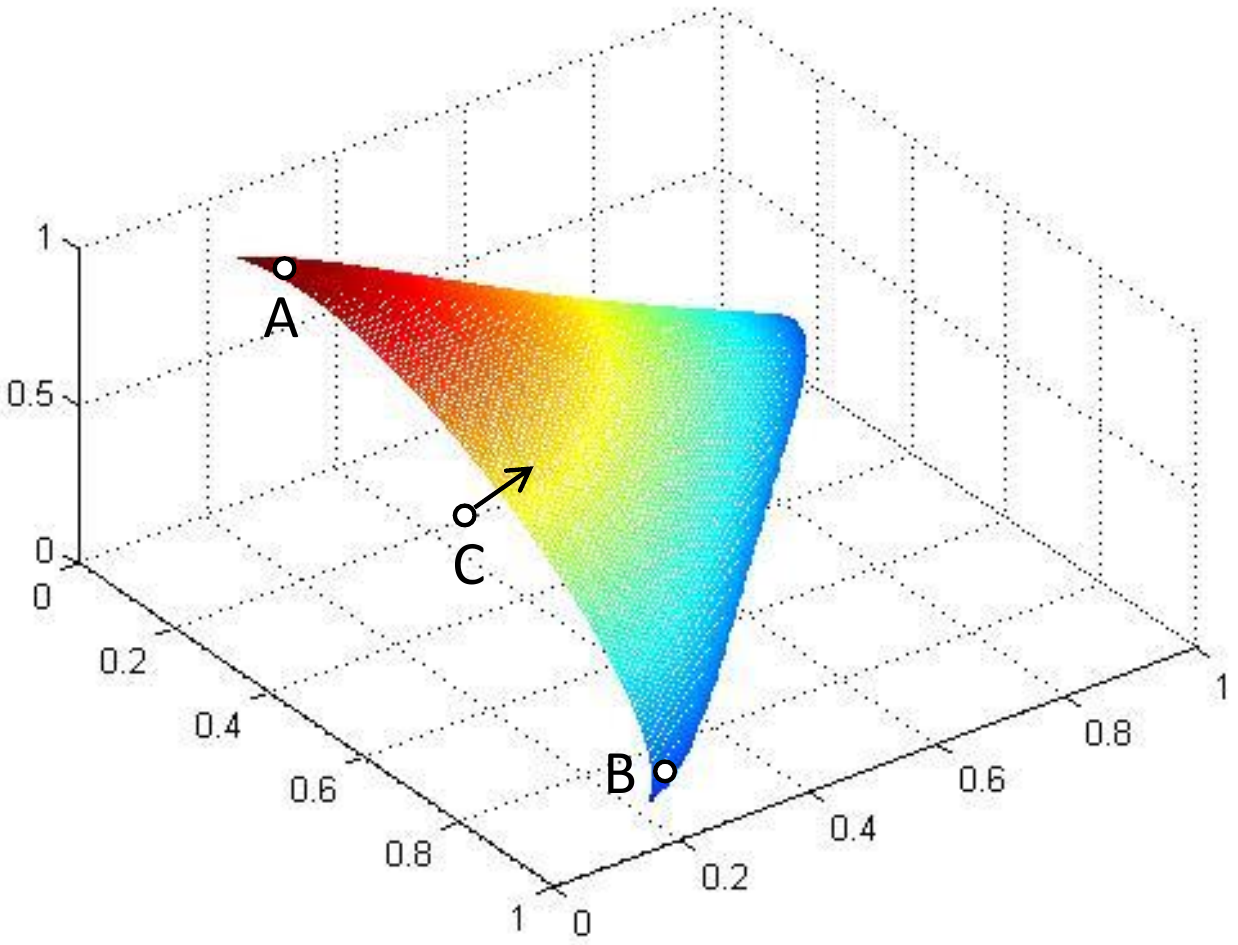}
  \caption{Ray ``escaping" in $\vone_n$ direction. More evidence in Figure \ref{fig:proj_square} in Appendix \ref{sec:square_loss}.}
  \label{fig:hole_square}
\label{fig:facade}
\end{minipage}
~
\begin{minipage}[b]{0.31\textwidth}
\includegraphics[width=1\textwidth,height=0.9\textwidth]{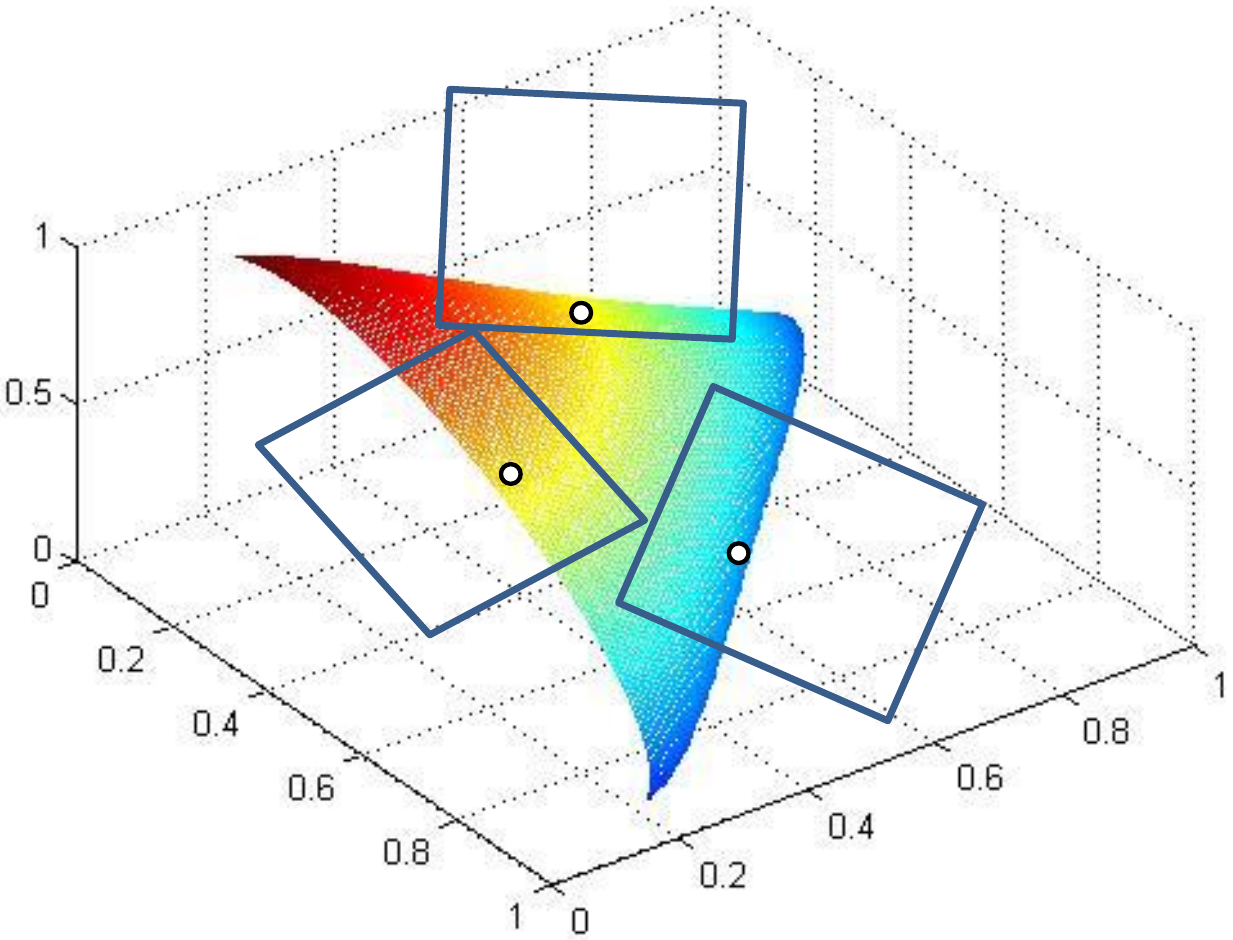}
  \caption{Adding ``faces" to block rays in (almost) all positive directions.}
  \label{fig:extend_3d}
\label{fig:facade}
\end{minipage}
~
\begin{minipage}[b]{0.34\textwidth}
\centering
\includegraphics[width=0.94\textwidth,height=0.84\textwidth]{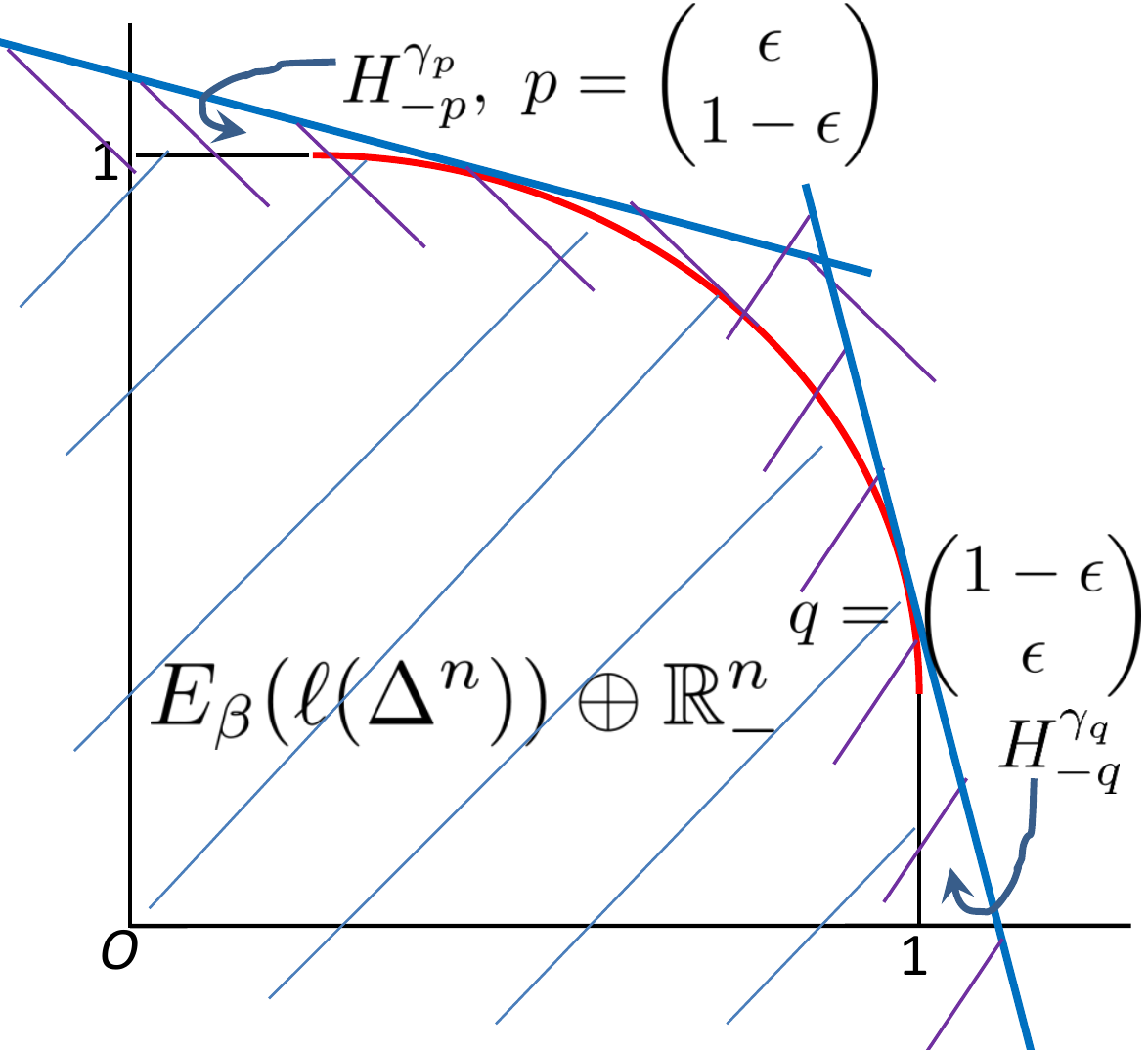}
  \caption{Sub-exp-prediction set extended by removing near axis-parallel supporting hyperplanes.}
  \label{fig:extend_2d}
\end{minipage}
\end{figure*}

Given a $\beta$-mixable loss $\ell$, its sub-exp-prediction set defined as follows must be convex:
\begin{align}
  T_\ell := E_\beta(\ell(\Delta^n))\varoplus\RR_-^n = \{ E_\beta(\ell(p)) - x : p \in \Delta^n, x \in \RR_+^n\}.
\end{align}
Note $T_\ell$ extends infinitely in any direction $p \notin \RR_+^n$.
Therefore it can be written in terms of supporting hyperplanes as
\begin{align}
\label{eq:supp_T_ell}
  T_\ell = \mathop{\bigcap}\limits_{p \in \Delta^n} H_{-p}^{\gamma_p}, \where \gamma_p = \min_{x \in T_\ell} x' \cdot (-p), \text{ and } H_{-p}^{\gamma_p} := \{x: x' \cdot (-p) \ge \gamma_p\}.
\end{align}
To extend the sub-exp-prediction set with ``faces", we remove some hyperplanes involved in \eqref{eq:supp_T_ell} that correspond to the $\epsilon$ ``rim" of the simplex (see Figure \ref{fig:extend_2d} for an illustration in 2-D)
\begin{align}
\label{eq:ext_T_ell}
  T_\ell^\epsilon := \mathop{\bigcap}\limits_{p \in \Delta^n_\epsilon} H_{-p}^{\gamma_p}, \where \Delta^n_\epsilon := \{ p \in \Delta^n: \min_i p_i > \epsilon\}.
\end{align}
Since $\epsilon > 0$, for any $p \in \Delta^n_\epsilon$,
$E_\beta^{-1} (H_{-p}^{\gamma_p} \cap \RR_+^n)$ is exactly the super-prediction set of a log-loss with appropriate scaling and shifting (see proof in Appendix \ref{sec:proof}).
So it must be convex.
Therefore $E_\beta^{-1}(T_\ell^\epsilon \cap \RR_+^n) = \mathop{\bigcap}_{p \in \Delta^n_\epsilon} E_\beta^{-1} (H_{-p}^{\gamma_p} \cap \RR_+^n)$ must be convex, and its recession cone is clearly $\RR_+^n$.
This guarantees that the following loss is proper over $p \in \Delta^n$ \citep[][Proposition 2]{williamson2014geometry}:
\begin{align}
\label{def:elltil}
  \elltil_\epsilon (p) = \argmin_{z \in E_\beta^{-1}(T_\ell^\epsilon \cap \RR_+^n)} p' \cdot z,
\end{align}
where the argmin must be attained uniquely (Appendix \ref{sec:proof}). Our next proposition states that $\elltil_\epsilon$ meets all the requirements of approximation suggested above.

\begin{proposition}
\label{exp_concave_approx}
For any $\epsilon > 0$, $\tilde{\ell}_\epsilon$ satisfies the condition $\partial_{\vone_n} \mathcal{B}_{\beta} \subseteq E_\beta (\ell(\mathcal{V}))$. In addition, $\elltil_\epsilon = \ell$ over a subset $S_\epsilon \subseteq \Delta^n$,
where for any $p$ in the relative interior of $\Delta^n$, 
$p \in S_\epsilon$ for sufficiently small $\epsilon$, i.e., $\lim_{\epsilon\downarrow 0}\mathrm{vol}(\Delta^n\setminus S_\epsilon)=0$.
\end{proposition}

Note $\|\elltil_\epsilon(p)-\ell(p)\|$ is not bounded for $p\notin S_\epsilon$. While the result does not show that \emph{all} $\beta$-mixable losses can be made $\beta$-exp-concave, it is suggestive that such a result may be obtainable by a different argument.


%% file: multiclass_proof.tex
\begin{proof} \textbf{(Proposition \ref{exp_concave_approx})}
We first show that for a half space $H_{-p}^{\gamma_p}$ defined in \eqref{eq:supp_T_ell} with $p \in \Delta^n_\epsilon$, $E_\beta^{-1}(H_{-p}^{\gamma_p} \cap \RR^n_+)$ must be the super-prediction set of a scaled and shifted log loss.
In fact, as $p_i > 0$, clearly $\gamma_p = \min_{x \in T_\ell} x' \cdot (-p) < 0$.
Define a new loss $\elltil^{\text{log}}_i(q) = -\frac{1}{\beta} \log (-\frac{\gamma_p}{p_i} q_i)$ over $q \in \Delta^n$.
Then $S_{\elltil^{\log}} \subseteq E_\beta^{-1}(H_{-p}^{\gamma_p} \cap \RR^n_+)$ can be seen from
\begin{align}
  \sum_i (-p_i) \exp(-\beta \elltil^{\text{log}}_i(q)) = \sum_i (-p_i) (-\frac{\gamma_p}{p_i} q_i) = \gamma_p.
\end{align}
Conversely, for any $u$ such that $u_i > 0$ and $(-p)'\cdot u = \gamma_p$,
simply choose $q_i = -\frac{u_i p_i}{\gamma_p}$.
Then $q \in \Delta^n$ and $E_\beta(\elltil^{\text{log}}(q)) = u$.
In summary, $E_\beta^{-1}(H_{-p}^{\gamma_p} \cap \RR^n_+)$ is the super-prediction set of $\elltil^{\text{log}}$.

To prove Proposition \ref{exp_concave_approx}, we first show that for any point $a \in T_\ell^{\epsilon}$ and any direction $d$ from the relative interior of the positive orthant (which includes the $\vone_n$ direction),
the ray $\{a + r d : r \ge 0\}$ will be blocked by a boundary point of $T_\ell^{\epsilon}$.
This is because by the definition of $T_\ell^\epsilon$ in \eqref{eq:ext_T_ell},
the largest value of $r$ to guarantee $a + rd \in T_\ell^\epsilon$ can be computed by
\begin{align}
  r^* := \sup \{r\ge 0: a+rd \in T_\ell^{\epsilon}\} = \sup \{r\ge 0: (a+rd)' \cdot (-p) \ge \gamma_p, \forall p \in \Delta^n_\epsilon \}
\end{align}
must be finite and attained.
Denote $x = a + r^*d$, which must be on the boundary of $T_\ell^{\epsilon}$ because
\begin{align}
\label{eq:bd_1}
  -x' \cdot p &\ge \gamma_p, \text{ for all } p \in \Delta^n_\epsilon, \\
\label{eq:bd_2}
  \text{ and } -x' \cdot p^* &= \gamma_{p^*} \text{ for some } p^* \in \Delta^n_\epsilon \text{ (not necessarily unique)}.
\end{align}
In order to prove the first statement of Proposition \ref{exp_concave_approx}, it suffices to show that for any point $x$ on the boundary of $T_\ell^{\epsilon}$,
there exists a $q \in \Delta^n$ such that $E_\beta(\elltil_\epsilon(q)) = x$.
Suppose $x$ satisfies \eqref{eq:bd_1} and \eqref{eq:bd_2}.
Then consider the (shifted/scaled) log loss $\elltil^{\log}$ that corresponds to $H_{p^*}^{\gamma_{p^*}}$.
Because log loss is strictly proper, there must be a unique $q \in \Delta^n$ such that the hyperplane $H_0:=\{z: q' \cdot z = q' \cdot E_\beta^{-1}(x)\}$ supports the super-prediction set of $\elltil^{\log}$ (\ie\ $E_\beta^{-1}(H_{-p^*}^{\gamma_{p^*}} \cap \RR_+^n)$) at $E_\beta^{-1}(x)$.
Since $E_\beta^{-1}(T_\ell^\epsilon \cap \RR_+^n)$ is a convex subset of $E_\beta^{-1}(H_{-p^*}^{\gamma_{p^*}} \cap \RR_+^n)$,
this hyperplane also supports $E_\beta^{-1}(T_\ell^\epsilon \cap \RR_+^n)$ at $E_\beta^{-1}(x)$.
Therefore $E_\beta^{-1}(x)$ is an optimal solution to the problem in the definition of $\elltil_\epsilon(q)$ in \eqref{def:elltil}.
Finally observe that it must be the unique optimal solution,
because if there were another solution which also lies on $H_0$, then by the convexity of the super-prediction set of $\elltil_\epsilon$,
the line segment between them must also lie on the prediction set of $\elltil_\epsilon$.
This violates the mixability condition of $\elltil_\epsilon$,
because by construction its sub-exp-prediction set is convex.

In order to check where $\ell(p) = \elltil_\epsilon(p)$, a sufficient condition is that the normal direction $d$ on the exp-prediction set evaluated at $E_\beta(\ell(p))$ satisfies $d_i / \sum_j d_j > \epsilon$.
Simple calculus shows that $d_i \propto p_i \exp(\beta \ell_i(p))$.
Therefore as long as $p$ is in the relative interior of $\Delta^n$,
$d_i / \sum_j d_j > \epsilon$ can always be satisfied by choosing a sufficiently small $\epsilon$.
And for each fixed $\epsilon$, the set $S_\epsilon$ mentioned in the theorem consists exactly of all such $p$ that satisfies this condition.
\end{proof} 

%% file: Kamalaruban15.bbl
\begin{thebibliography}{21}
\providecommand{\natexlab}[1]{#1}
\providecommand{\url}[1]{\texttt{#1}}
\expandafter\ifx\csname urlstyle\endcsname\relax
  \providecommand{\doi}[1]{doi: #1}\else
  \providecommand{\doi}{doi: \begingroup \urlstyle{rm}\Url}\fi

\bibitem[Banerjee et~al.(2005)Banerjee, Guo, and Wang]{banerjee2005optimality}
Arindam Banerjee, Xin Guo, and Hui Wang.
\newblock On the optimality of conditional expectation as a {B}regman
  predictor.
\newblock \emph{IEEE Transactions on Information Theory}, 51\penalty0
  (7):\penalty0 2664--2669, 2005.

\bibitem[Buja et~al.(2005)Buja, Stuetzle, and Shen]{buja2005loss}
Andreas Buja, Werner Stuetzle, and Yi~Shen.
\newblock Loss functions for binary class probability estimation and
  classification: Structure and applications,” manuscript, available at
  www-stat.wharton.upenn.edu/$\sim$buja, 2005.

\bibitem[Dragomir(2000)]{dragomir2000some}
Sever~Silvestru Dragomir.
\newblock Some {G}ronwall type inequalities and applications.
\newblock \emph{RGMIA Monographs, Victoria University, Australia}, 19, 2000.

\bibitem[Gneiting and Raftery(2007)]{gneiting2007strictly}
Tilmann Gneiting and Adrian~E. Raftery.
\newblock Strictly proper scoring rules, prediction, and estimation.
\newblock \emph{Journal of the American Statistical Association}, 102\penalty0
  (477):\penalty0 359--378, 2007.

\bibitem[Hand(1994)]{hand1994deconstructing}
David~J. Hand.
\newblock Deconstructing statistical questions.
\newblock \emph{Journal of the Royal Statistical Society, Series A (Statistics
  in Society)}, 157\penalty0 (3):\penalty0 317--356, 1994.

\bibitem[Hand and Vinciotti(2003)]{hand2003local}
David~J. Hand and Veronica Vinciotti.
\newblock Local versus global models for classification problems: Fitting
  models where it matters.
\newblock \emph{The American Statistician}, 57\penalty0 (2):\penalty0 124--131,
  2003.

\bibitem[Haussler et~al.(1998)Haussler, Kivinen, and
  Warmuth]{haussler1998sequential}
David Haussler, Jyrki Kivinen, and Manfred~K. Warmuth.
\newblock Sequential prediction of individual sequences under general loss
  functions.
\newblock \emph{IEEE Transactions on Information Theory}, 44\penalty0
  (5):\penalty0 1906--1925, 1998.

\bibitem[Hazan et~al.(2007)Hazan, Agarwal, and Kale]{hazan2007logarithmic}
Elad Hazan, Amit Agarwal, and Satyen Kale.
\newblock Logarithmic regret algorithms for online convex optimization.
\newblock \emph{Machine Learning}, 69\penalty0 (2-3):\penalty0 169--192, 2007.

\bibitem[Kalnishkan and Vyugin(2005)]{kalnishkan2005weak}
Yuri Kalnishkan and Michael~V. Vyugin.
\newblock The weak aggregating algorithm and weak mixability.
\newblock In \emph{Proc.\ Annual Conf.\ Computational Learning Theory}, pages
  188--203. Springer, 2005.

\bibitem[Kivinen and Warmuth(1999)]{kivinen1999averaging}
Jyrki Kivinen and Manfred~K. Warmuth.
\newblock Averaging expert predictions.
\newblock In \emph{Proc.\ Annual Conf.\ Computational Learning Theory}, pages
  153--167. Springer, 1999.

\bibitem[Reid and Williamson(2010)]{reid2010composite}
Mark~D. Reid and Robert~C. Williamson.
\newblock Composite binary losses.
\newblock \emph{Journal of Machine Learning Research}, 11:\penalty0 2387--2422,
  2010.

\bibitem[Reid and Williamson(2011)]{reid2011information}
Mark~D. Reid and Robert~C. Williamson.
\newblock Information, divergence and risk for binary experiments.
\newblock \emph{Journal of Machine Learning Research}, 12:\penalty0 731--817,
  2011.

\bibitem[Rockafellar(1970)]{rockafellar1970convex}
R.~Tyrrell Rockafellar.
\newblock \emph{Convex analysis}.
\newblock Princeton university press, 1970.

\bibitem[van Erven(2012)]{van2012exp}
Tim van Erven.
\newblock From exp-concavity to mixability, 2012.
\newblock
  http://www.timvanerven.nl/blog/2012/12/from-exp-concavity-to-mixability/.

\bibitem[Van~Erven et~al.(2012)Van~Erven, Reid, and
  Williamson]{van2012mixability}
Tim Van~Erven, Mark~D. Reid, and Robert~C. Williamson.
\newblock Mixability is {B}ayes risk curvature relative to log loss.
\newblock \emph{Journal of Machine Learning Research}, 13\penalty0
  (1):\penalty0 1639--1663, 2012.

\bibitem[Vernet et~al.(2014)Vernet, Reid, and Williamson]{vernet2011composite}
Elodie Vernet, Mark~D. Reid, and Robert~C. Williamson.
\newblock Composite multiclass losses.
\newblock \emph{Journal of Machine Learning Research}, 2014.
\newblock Under review.
  http://users.cecs.anu.edu.au/$\sim$williams/papers/P189.pdf.

\bibitem[Vovk(1995)]{vovk1995game}
Volodya Vovk.
\newblock A game of prediction with expert advice.
\newblock In \emph{Proc.\ Annual Conf.\ Computational Learning Theory}, pages
  51--60. ACM, 1995.

\bibitem[Vovk(2001)]{vovk2001competitive}
Volodya Vovk.
\newblock Competitive on-line statistics.
\newblock \emph{International Statistical Review}, 69\penalty0 (2):\penalty0
  213--248, 2001.

\bibitem[Vovk and Zhdanov(2009)]{vovk2009prediction}
Volodya Vovk and Fedor Zhdanov.
\newblock Prediction with expert advice for the {B}rier game.
\newblock \emph{Journal of Machine Learning Research}, 10:\penalty0 2445--2471,
  2009.

\bibitem[Williamson(2014)]{williamson2014geometry}
Robert~C. Williamson.
\newblock The geometry of losses.
\newblock In \emph{Proc.\ Annual Conf.\ Computational Learning Theory}, pages
  1078--1108, 2014.

\bibitem[Zinkevich(2003)]{zinkevich2003online}
Martin Zinkevich.
\newblock Online convex programming and generalized infinitesimal gradient
  ascent.
\newblock In \emph{Proceedings of the International Conference on Machine
  Learning}, pages 928--936, 2003.

\end{thebibliography}
